\documentclass[a4paper,11pt]{article}

\input{macros.tex}

\title{
\toptitlebar
{{\center\baselineskip 18pt
                      {\Large\bf Byzantine-Robust and Differentially Private Federated Optimization under Weaker Assumptions}}
} 
\bottomtitlebar}
\date{}
\author{
Rustem Islamov\textsuperscript{1}, 
Grigory Malinovsky\textsuperscript{2},
Alexander Gaponov\textsuperscript{2},
Aurelien Lucchi\textsuperscript{1},
\newline
Peter Richtárik\textsuperscript{2},
Eduard Gorbunov\textsuperscript{3}
}
\affil{
\textsuperscript{1}University of Basel, Switzerland, \quad
\textsuperscript{2}KAUST, Saudi Arabia, \quad
\textsuperscript{3}MBZUAI, UAE
}

\begin{document}

\maketitle

\def\thefootnote{\arabic{footnote}}

\begin{abstract}
Federated Learning (FL) enables heterogeneous clients to collaboratively train a shared model without centralizing their raw data, offering an inherent level of privacy. However, gradients and model updates can still leak sensitive information, while malicious servers may mount adversarial attacks such as Byzantine manipulation. These vulnerabilities highlight the need to address differential privacy (DP) and Byzantine robustness within a unified framework. Existing approaches, however, often rely on unrealistic assumptions such as bounded gradients, require auxiliary server-side datasets, or fail to provide convergence guarantees. We address these limitations by proposing \algname{Byz-Clip21-SGD2M}, a new algorithm that integrates robust aggregation with double momentum and carefully designed clipping. We prove high-probability convergence guarantees under standard $L$-smoothness and $\sigma$-sub-Gaussian gradient noise assumptions, thereby relaxing conditions that dominate prior work. Our analysis recovers state-of-the-art convergence rates in the absence of adversaries and improves utility guarantees under Byzantine and DP settings. Empirical evaluations on CNN and MLP models trained on MNIST further validate the effectiveness of our approach. 
\end{abstract}

\section{Introduction}

The rapid deployment of large-scale machine learning models has positioned Federated Learning (FL) \citep{konevcny2016federated, mcmahan2017communication} as a central paradigm, enabling a collection of potentially heterogeneous clients, ranging from smartphones and sensors to data centers, to collaboratively train a shared model without transmitting their raw data to a central server \citep{li2020federated, yang2019federated}. While this setting offers an inherent degree of privacy, it is far from complete: gradients and model parameters can still leak sensitive information \citep{zhu2019deep, geiping2020inverting}, and adversaries may exploit the distributed nature of FL to launch attacks such as membership inference, model stealing, or Byzantine manipulation \citep{shokri2017membership,tramer2016stealing,blanchard2017machine}. This dual vulnerability highlights the need to study differential privacy, to mitigate privacy leakage, and Byzantine robustness, to ensure resilience against arbitrary or malicious participants, within a unified framework. Both challenges share the common difficulty of optimizing with corrupted or unreliable updates, whether the corruption arises from deliberate noise injection for privacy or from adversarial manipulations. Moreover, addressing these issues is not only of algorithmic importance: modern AI systems must also comply with evolving regulatory requirements, such as GDPR and the EU AI Act for privacy \citep{GDPR2016,EUAIAct2024}, or the EU NIS2 Directive, Cyber Resilience Act, and NIST AI Risk Management Framework for security and robustness \citep{NIS22022,CyberResilienceAct2024,NISTAI100-1}.

In this context, we propose to study Byzantine robustness and differential privacy jointly, with the goal of establishing a principled foundation for robust, privacy-preserving federated learning and providing a pathway toward trustworthy AI in practice. A central open challenge is to move beyond the strong assumptions that dominate the current literature.
Most theoretical analyses in differential privacy are carried out under assumptions such as bounded gradients -- an assumption that fails even for simple quadratic objectives \citep{li2022soteriafl, wang2023efficient, lowy2023private} -- or by focusing on the full-batch regime \citep{shulgin2025first}.
Only recently have approaches emerged that achieve strong privacy guarantees while retaining favorable optimization performance under milder conditions \citep{islamov2025double}. Byzantine robustness has evolved largely as a separate line of work \citep{lyu2022privacy}, yet its analyses similarly depend on gradient boundedness to establish convergence guarantees \citep{blanchard2017machine, mhamdi2018hidden}. Attempts to combine DP with Byzantine resilience remain even more limited: existing methods either assume access to an auxiliary server-side dataset to obtain a better gradient approximation \citep{xiang2023practical}, rely on unrealistic aggregation rules \citep{guerraoui2021differential, zhang2023byzantine}, impose restrictive bounded-gradient assumptions \citep{allouah2023privacy}.

Motivated by the shortcomings of prior theoretical results, we arrive at the following guiding question:
\begin{quote}
\textit{Is it possible to develop an algorithm that provably converges in the presence of malicious clients while simultaneously providing strong differential privacy guarantees under standard assumptions?}
\end{quote}

\paragraph{Main Contributions.}
We provide an affirmative answer to the above question. In doing so, we make the following contributions:
\begin{itemize}
    \item We introduce \algname{Byz-Clip21-SGD2M}, a new algorithm that combines robust aggregation with double momentum and carefully designed clipping, achieving efficient performance under both Byzantine and differential privacy adversaries.

    \item We provide a high-probability convergence analysis of \algname{Byz-Clip21-SGD2M} under $L$-smoothness and $\sigma$-sub-Gaussian gradient noise, thereby relaxing assumptions that have been used in prior work (see Table~\ref{tab:summary_experiments} for details). Our results improve utility guarantees by avoiding unrealistic assumptions, while recovering state-of-the-art convergence rates in the absence of Byzantine or DP adversaries.

    \item We complement our theoretical analysis with empirical validation, demonstrating the effectiveness of \algname{Byz-Clip21-SGD2M} on CNN and MLP models trained on the MNIST dataset.

\end{itemize}

\section{Related Works}

\paragraph{Error Feedback.} Biased compression in federated learning reduces communication by transmitting compressed local updates that are not unbiased estimates of the original signal \citep{ajalloeian2020convergence, beznosikov2023biased, demidovich2023guide}. Common examples include Top-$K$ sparsification \citep{alistarh2017qsgd, wangni2018gradient} and sign-based compression \citep{bernstein2018signsgd,stich2018sparsified}. While these methods are efficient in practice, the introduced bias can accumulate and affect convergence, requiring Error Feedback (\algname{EF}) \citep{seide20141} techniques to mitigate its impact. Earlier works primarily focused on the single-node setting or relied on restrictive assumptions, such as bounded gradients, bounded compression error, or gradient dissimilarity, to establish convergence \citep{stich2018sparsified,stich2019error,koloskova2019decentralized}. Moreover, the convergence rates of \algname{EF} degrade in the presence of client heterogeneity, and this dependence is intrinsic rather than a proof artifact \citep{gorbunov2020linearly}.
To overcome these limitations, \citep{richtarik2021ef21} proposed \algname{EF21}, a variant whose guarantees do not depend on heterogeneity bounds. Nevertheless, \algname{EF21}-SGD still requires increasingly large batch sizes to achieve a target accuracy \citep{fatkhullin2024momentum}. Importantly, this limitation is not fundamental: recent work demonstrates that incorporating Heavy-Ball momentum removes the need for large batches \citep{fatkhullin2024momentum}. Later, \algname{EF21} was extended to several practical setups \citep{fatkhullin2025ef21,makarenko2022adaptive}. In a parallel line of research, \citet{gao2023econtrol,gao2025accelerated} introduced \algname{EControl}-type Error Feedback methods with convergence guarantees extending to the convex setting. Error Feedback has also been studied in decentralized \citep{yau2022docom,huang2023stochastic,islamov2024towards} setting, for non-smooth \citep{islamov2025safe} and composite \citep{gao2025composite} optimization problems, as well as in the context of second-order methods \citep{safaryan2021fednl,qian2021basis,islamov2023distributed}.

\begin{table*}[t]
    \centering
    
    \caption{ { \small Convergence guarantees of algorithms in the presence of Byzantine and/or DP adversaries.~$\wtilde{\cO}$~hides~logarithmic and decaying-with-$T$ terms. Notation: (s)CVX=(strongly) convex functions, nCVX=non-convex functions,~DP=supports differential privacy, BR=supports Byzantine robustness, E=In-Expectation analysis, P=High-probability analysis.}}
    \label{tab:summary_experiments}
    \resizebox{\textwidth}{!}{
        \begin{tabular}{@{\hskip 0pt}c@{\hskip 0pt}c@{\hskip 0pt}c@{\hskip 0pt}c@{\hskip 0pt}c@{\hskip 0pt}c@{\hskip 0pt}c}
            \toprule

            {\bf Method} &
            {\bf Utility} &
            \makecellnew{{\bf DP} } & 
            \makecellnew{{\bf BR} } &
            {\bf Assumptions} &
            \makecellnew{{\bf Additional} \\ {\bf Comments}} &
            \makecellnew{{\bf Type}}

            \\ \toprule

            \makecellnew{\algname{SoteriaFL} \\ \citep{liu2022communication} } &
            $\wtilde{\cO}\left(\frac{\sqrt{d}}{\sqrt{G}\varepsilon}\right)$ &
            \cmark &
            \xmark &
            \makecellnew{$M$-bounded gradients} &
            nCVX, local DP &
            E
            \\    

            \makecellnew{\algname{$\alpha$-Norm} \\ \citep{shulgin2025smoothed} } &
            $\wtilde{\cO}\left(\frac{\sqrt{d}}{\sqrt{G}\varepsilon} + R\right)^{(a)}$ &
            \cmark &
            \xmark &
            \makecellnew{$L$-smoothness \\ Full-batch gradients} &
            nCVX, local DP &
            E
            \\             
            \makecellnew{\algname{Clip21-SGD2M} \\ \citep{islamov2025double} } &
            $\wtilde{\cO}\left(\frac{\sqrt{d}}{\sqrt{G}\varepsilon}\right)^{(b)} $ &
            \cmark &
            \xmark &
            \makecellnew{$L$-smoothness \\ $\sigma$-sub-Gaussian noise}  &
            nCVX, local DP &
            P
            \\ \midrule[1pt]\midrule[1pt]

            \makecellnew{\algname{Byz-SGDM} \\ \citep{karimireddy2020byzantine} } &
            $\wtilde{\cO}\left(\deltabyz\zeta^2\right)^{(c)}$ &
            \xmark &
            \cmark &
            \makecellnew{$L$-smoothness \\ Bounded variance} &
            nCVX &
            E

            \\

            \makecellnew{\algname{Robust D-SHB} \\ \citep{allouah2023fixing} } &
            $\wtilde{\cO}\left(\deltabyz M^2\right)$ &
            \xmark &
            \cmark &
            \makecellnew{$L$-smoothness \\ Full-batch gradients} &
            \makecellnew{nCVX \\ CVX} &
            E

            \\ 

            \makecellnew{\algname{Byz-VR-MARINA} \\ \citep{gorbunov2022variance} } &
            $\wtilde{\cO}\left(\deltabyz\zeta^2\right)$ &
            \xmark &
            \cmark &
            \makecellnew{$L$-smoothness \\ Local functions are finite-sums}  &
            nCVX &
            E

            \\  \midrule[1pt]\midrule[1pt]

            \makecellnew{\algname{Safe-DSHB} \\ \citep{allouah2023privacy}}&
            $\wtilde{\cO}\left(\frac{d}{G\varepsilon^2} + \frac{\deltabyz}{\varepsilon^2} + \deltabyz M^2\right)$ &
            \cmark &
            \cmark &
            $M$-bounded gradients &
            sCVX, local-DP&
            E 

            \\   


            
            \makecellnew{\algname{DP-BREM} \\ \citep{gu2025dp} } &
            $\wtilde{\cO}\left(\deltabyz\zeta^2 + \deltabyz\frac{\sqrt{d}}{\sqrt{G}\varepsilon}\right)^{(d)}$&
            \cmark &
            \cmark &
            \makecellnew{ $L$-smoothness\\Momentum buffers are i.i.d}&
            nCVX, central DP &
            E
           \\

           \addlinespace[3pt]

            \makecellnew{\algname{P\&S Learning} \\ \citep{xiang2023practical} } &
            N/A${}^{(e)}$&
            \cmark &
            \cmark &
            \makecellnew{$L$-smoothness \\
            $\nabla f_i(x)-\nabla f_i(x^\star)$ are sub-Gaussian\\ 
            Access to external non-private dataset} &
            sCVX, local DP &
            P 
            

            
            \\ \midrule[1pt]\midrule[1pt]

          \makecellnew{  \algname{Byz-Clip21-SGD2M} \\ \textbf{[Ours]}}&
            \makecellnew{$\wtilde{\cO}\left(R + R^{2/3}\right)^{(b,f)}$ \\
            $R = \frac{\sqrt{d}}{\sqrt{G}\varepsilon} + \frac{\sqrt{\deltabyz d}}{\sqrt{G}\varepsilon} + \sqrt{\deltabyz}\cN$}& 
            \cmark &
            \cmark &
            \makecellnew{$L$-smoothness \\ $\sigma$-sub-Gaussian noise} & 
            nCVX, local-DP & 
            P

            \\ \bottomrule
        \end{tabular}
        }
        \begin{tablenotes}
      {\scriptsize 
      \item $(a)$ Here $R\eqdef \max_{i\in\cG}\|\nabla f_i(x^0)-g_i^0\|.$
      \item $(b)$ Derived under the assumption $\nicefrac{\sqrt{d}}{\sqrt{G}\varepsilon} \ge 1$.
      \item $(c)$ $\zeta^2$ denotes the constant in the gradient dissimilarity assumption.
      \item $(d)$ The provided rate contradicts the lower bound due to the unrealistic assumption that momentum buffers are i.i.d. distributed.
      \item $(e)$ The authors do not provide explicit bounds.
      \item $(f)$ In the general case, $\cN = LF^0 + \zeta^2 + \sigma(\sqrt{LF^0} + \zeta)$, where $F^0 = f(x^0)- f^\star$ (see \Cref{th:main_theorem}). In the presence of Byzantine adversaries only, $\cN = \sqrt{\deltabyz}\zeta^2$ (see \Cref{th:nodp_withbyz}).
        }
\end{tablenotes}

\end{table*}


\paragraph{Byzantine Robust Optimization.}

Naively averaging updates from clients in a distributed setting lacks robustness since even a single Byzantine client can destabilize training \citep{chen2017distributed}. To address this, numerous alternatives to averaging have been proposed to enhance robustness against adversarial updates \citep{pillutla2022robust, yin2018byzantine, Damaskinos2019Aggregathor, karimireddy2021learning, Allouah2024Adaptive}.

Nevertheless, permutation-invariant algorithms, i.e., the ones whose output is unchanged under shuffling of the stochastic gradients computed on different workers \citep[Definition B]{karimireddy2021learning}, alone cannot ensure convergence in the Byzantine regime \citep[Theorem II]{karimireddy2021learning}. This limitation can be alleviated by incorporating update shuffling and averaging in random groups \citep{karimireddy2020byzantine}, variance-reduction techniques \citep{wu2020federated, gorbunov2022variance}, or preprocessing based on the averaging with nearest neighbors \citep{allouah2023fixing}. Beyond these, other approaches include concentration-based gradient filtering \citep{AlistarhAllenZhuLi2018, AllenZhu2020ByzantineResilientNonConvex}, server-free protocols with random checks \citep{gorbunov2022secure}, redundant computations \citep{chen2018draco, rajput2019detox}, and reputation-based mechanisms \citep{rodriguez2020dynamic, regatti2020bygars, xu2020reputation}. More recent works further extend these ideas to partial participation \citep{malinovsky2024byzantine}, decentralized training \citep{he2022byzantine}, and communication compression \citep{Rammal2024}. In most of these cases, the theoretical analysis is carried out under the $L$-smoothness and $\zeta^2$-gradient dissimilarity\footnote{More precisely, the mentioned works rely either on $\frac{1}{G}\sum_{i\in \cG}\|\nabla f_i(x) - \nabla f(x)\|^2 \leq \zeta^2$ or $\frac{1}{G}\sum_{i\in \cG}\|\nabla f_i(x) - \nabla f(x)\|^2 \leq A_\zeta \|\nabla f(x)\|^2 + \zeta^2$ for all $x\in \R^d$.} assumptions, which have become standard in Byzantine robustness literature.


\paragraph{Differentially Private Optimization.}


Differential privacy (DP) is typically achieved by clipping each client’s update and adding Gaussian noise, limiting the influence of any individual client \citet{mcmahan2017learning}. In central DP, a trusted server adds noise before updating the global model, while in local DP, clients perturb updates before sending them, protecting privacy even from the server \citep{kasiviswanathan2011can, allouah2024privacy}. Local DP offers stronger privacy but reduces model utility, though this can be mitigated using secure shufflers or aggregators \citep{feldman2020amplification, bonawitz2017practical}. DP can also enable update compression without additional cost \citep{chaudhuri2022privacy, hegazy2023compression}. Private optimization methods, such as \algname{DP-SGD} \citep{abadi2016deep}, enforce DP by clipping gradients and adding noise scaled to the clipping sensitivity. However, most convergence analyses neglect the bias introduced by clipping. For smooth functions, guarantees typically assume either \emph{bounded gradient norms} \citep{li2020federated, zhang2020private, murata2023diff2, wang2023efficient, lowy2023private, wang2024efficient} or that \emph{clipping is effectively inactive} \citep{noble2022differentially, zhang2024private}. Recently, several works have provided improved analyses of \algname{DP-Clip-SGD} and its variants. In the single-node setting, generalized smoothness has also been considered without restrictive requirements for \algname{DP-Clip-SGD} \citep{koloskova2023revisiting}. Next, \citet{khah2025differentially} also focus on the single-node setting and provide high-probability results under the heavy-tailed noise assumption and arbitrary clipping level for \algname{DP-Clip-SGD}. \citet{zhao2025differential} consider a projected \algname{SGD}-type method utilizing DP mean estimation and provide its in-expectation convergence bounds. \citet{compagnoni2026adaptive} study SDE approximation of \algname{DP-Clip-SGD} in the single-node setup. In the FL setting, analysis under realistic conditions in the presence of differential privacy has been studied in \citet{islamov2025double, shulgin2025smoothed, shulgin2025first}.

\paragraph{Differentially Private and Byzantine Robust Methods.}

Although substantial progress has been made separately in differential privacy and Byzantine robustness, their interaction in distributed learning is still not well understood. Recent efforts to address both aspects at once remain limited. A key limitation of existing theoretical work is the reliance on restrictive assumptions, often unrealistic and mainly introduced to simplify the analysis. For example, some studies focus on a single aggregation rule \citep{guerraoui2021differential}, while others assume a particular noise structure near the minimizer or even access to a non-private external server-side dataset \citep{xiang2023practical}. Other analyses restrict attention to the strongly convex case \citep{gao2023bvdfed, allouah2023privacy} or impose functional similarity assumptions that do not hold in heterogeneous federated learning \citep{lan2025one}. Further limitations include assuming i.i.d. momentum buffers \citep{gu2025dp} or bounded-gradient conditions \citep{guerraoui2021combining, gao2023bvdfed, allouah2023privacy}. Taken together, these limitations highlight the significant challenges in jointly addressing differential privacy and Byzantine robustness.

We summarize the existing works and the assumptions underlying their convergence analyses for establishing utility bounds in \Cref{tab:summary_experiments}. Compared to prior work, our theoretical analysis relies on the weakest set of assumptions.

\section{Preliminaries}
\label{sec:prem}
\textbf{Problem Formulation.} In distributed training with potentially Byzantine clients, the goal is to minimize the empirical loss
\begin{equation}\label{eq:problem}
\min_{x\in\R^d} \left[f(x) \eqdef \frac{1}{G}\sum_{i\in\cG} f_i(x)\right],
\end{equation}
where $x \in \R^d$ denotes the model parameters. The set of regular clients is $\cG$, with cardinality $G \eqdef |\cG|$, while the total number of clients is $n$. The remaining clients, $\cB \eqdef [n]\setminus\cG$, are Byzantine: instead of following the prescribed learning algorithm, they may transmit arbitrary (possibly adversarial) vectors, and we assume they also observe the updates shared by the other clients\footnote{This assumption is admittedly strong and often unrealistic in practice. Nevertheless, it is widely adopted in the Byzantine-robust learning literature, since methods that guarantee robustness against such powerful adversaries also remain robust in more practical settings, where attackers are typically weaker.}. The number of Byzantine clients is $|\cB| = \deltabyz n$, with $\deltabyz < \frac{1}{2}$ to ensure that Byzantine clients do not form a majority; otherwise, problem \eqref{eq:problem} becomes unsolvable \citep{pai2021can}. Each client $i \in [n]$ is associated with a local loss function $f_i$, defined over its own data. 
    
\paragraph{Differential Privacy.}

Next, we will use the following classical definition of $(\varepsilon,\delta)$-Differential Privacy ($(\varepsilon, \delta)$-DP), which introduces
plausible deniability into the output of a learning algorithm.
\begin{definition}[\citep{dwork2014algorithmic}]\label{def:privacy} A randomized method $\cM\colon \cD \to \mathbb{R}$ satisfies $(\varepsilon,\delta)$-Differential Privacy ($(\varepsilon,\delta)$-DP for shortness) if for any two datasets $D,D^\prime\in\cD$ that differ in 1 sample and for any $S \subseteq \mathbb{R}$
\begin{equation} 
    \Prob(\cM(D)\in S) \le e^{\varepsilon}\Prob(\cM(D^\prime)\in\cS)+\delta.
\end{equation}
\end{definition}
Decreasing $\varepsilon$ and $\delta$ strengthens privacy by making it harder to identify the particular data point that differs between neighboring datasets.

\paragraph{Robust Aggregation.}

We follow the definition of \citet{allouah2023fixing} that covers many practical aggregation rules introduced in prior work \citep{farhadkhani2022byzantine, karimireddy2020byzantine}, including coordinate-wise median \citep{yin2018byzantine} and geometric median \citep{pillutla2022robust}.

\begin{definition}[\citep{allouah2023fixing}]\label{def:arrag} Let the ratio of Byzantine clients be $\deltabyz \eqdef \nicefrac{|\cB|}{n} < \nicefrac{1}{2}$ and $c \ge 0$ be a constant. The aggregation rule {\rm RAgg} is said to be $(\deltabyz, c)$-robust if for any vectors $\{x_1,\dots,x_n\}\subseteq \R^d$ and any subset $S \subset [n]$ of size $n-|\cB|$ the output $\hat{x} = {\rm RAgg}(x_1,\dots,x_n)$ satisfies
    \begin{equation}\label{eq:arrag}
        \|\hat{x}-\overline{x}\|^2\le \frac{c\deltabyz}{n-|\cB|} \sum_{i\in S}\|x_i-\overline{x}\|^2
    \end{equation}
    where $\overline{x} = \frac{1}{n-|\cB|}\sum_{i\in S}x_i.$
\end{definition}
Under this definition, robust aggregators deliver a worst-case guarantee: for $S = \cG$, the output $\hat{x}$ stays close to the mean $\overline{x}$ of regular clients in squared distance. The rule can be further strengthened by input mixing steps such as \algname{NNM} \citep{allouah2023fixing}. Importantly, we require the worst-case property, i.e., \eqref{eq:arrag} holds for every vector collection $\{x_i\}_{i=1}^n$ to enable our high-probability analysis. In contrast, prior work often studies stochastic aggregators under in-expectation guarantees \citep{karimireddy2021learning, karimireddy2020byzantine, gorbunov2022variance, malinovsky2024byzantine}.

\paragraph{Assumptions.} 

In the convergence analysis of \algname{Byz-Clip21-SGD2M}, we make use of two standard assumptions. The first describes the class of smooth functions, a standard assumption made in the non-convex optimization literature \citep{ghadimi2013stochastic, carmon2020lower}.

\begin{assumption}\label{asmp:smoothness}
    For all $i\in\cG$, $f_i$ is $L$-smooth, i.e., for all $x,y\in\R^d$ we have 
    \begin{equation} 
        \|\nabla f_i(x) - \nabla f_i(y)\| \le L\|x-y\|.
    \end{equation}
    Additionally, we have $f^\star \eqdef \inf_{x\in\R^d} f(x) > -\infty.$
\end{assumption}
To simplify the presentation, we consider the worst smoothness constant across clients. Next, we introduce the assumption that formalizes the stochasticity in local gradients.

\begin{assumption}\label{asmp:stoch_grad} 
    Each client $i\in\cG$ has access to a $\sigma$-sub-Gaussian unbiased estimator $\nabla f_i(x,\xi)$  of a local gradient $\nabla f_i(x)$, i.e., there exists a constant $\sigma \ge 0$ such that for all $x\in\R^d$ we have
    \begin{equation} 
        \E{\nabla f_i(x,\xi)} = \nabla f_i(x), \; \; \E{\exp(\nicefrac{\|\theta_i\|^2}{\sigma^2})} \le e,
    \end{equation}
    where $\xi$ denotes the source of stochasticity and $\theta_i \eqdef \nabla f_i(x,\xi) - \nabla f_i(x).$
\end{assumption}
Sub-Gaussian noise assumption is standard when analyzing high-probability complexity of \algname{SGD}-type algorithms \citep{ghadimi2012optimal, nemirovski2009robust, liu2023high}.

Bounds on data heterogeneity are essential in the analysis of Byzantine-robust algorithms; without such assumptions, robustness is unachievable in fully arbitrary heterogeneous environments. A common approach is to assume a bounded average deviation of local gradients from the global gradient (gradient dissimilarity) \citep{wu2020federated, gorbunov2022variance, allouah2023privacy}, as formalized in \Cref{asmp:bounded_heterogeneity_avg}. For our high-probability analysis, we instead require the maximum in \Cref{asmp:bounded_heterogeneity}. However, we impose this condition only at initialization $x^0$, whereas prior work typically assumes it holds uniformly over the entire space.

\begin{assumption}\label{asmp:bounded_heterogeneity_avg}
The local gradients $\{\nabla f_i(x)\}_{i\in\cG}$ satisfy bounded $\zeta^2$-gradient dissimilarity for some $\zeta\ge 0$, i.e., for all $x\in\R^d$ we have 
\begin{equation}\label{eq:data_heter_bound_avg} 
    \frac{1}{G} \sum_{i\in \mathcal{G}} \|\nabla f_i(x) - \nabla f(x)\|^2 \le \zeta^2.
\end{equation}
\end{assumption}

\addtocounter{assumption}{-1}
\refstepcounter{assumption}
\begingroup
  \renewcommand{\theassumption}{3.3a}
  \begin{assumption}\label{asmp:bounded_heterogeneity}
The local gradients $\{\nabla f_i(x^0)\}_{i\in\cG}$ satisfy bounded $\zeta^2$-gradient dissimilarity for some $\zeta\ge 0$, i.e., we have 
\begin{equation} \label{eq:data_heter_bound}
    \max_{i\in\cG}\|\nabla f_i(x^0) - \nabla f(x^0)\|^2 \le \zeta^2.
\end{equation}
\end{assumption}
\endgroup

\begin{remark}
\label{rem:1}
    Our results also extend to a more general condition in which the right-hand side of \eqref{eq:data_heter_bound} scales linearly with $\|\nabla f(x^0)\|^2$, namely, $    \max_{i\in\cG}\|\nabla f_i(x^0) - \nabla f(x^0)\|^2 \le A_\zeta \|\nabla f(x^0)\|^2+ \zeta^2.$  For clarity of presentation, the general result is provided in the supplementary materials, while we assume $A_\zeta = 0$ in the main text. Note that this assumption always holds, i.e., for any problem of the form \eqref{eq:problem}, there exists $\zeta \geq 0$ such that inequality \eqref{eq:data_heter_bound} is satisfied. However, as we explicitly show in our theoretical results, the value of $\zeta$ directly affects the convergence bounds since it is required to establish a bound $\Delta$ on the initial value of the Lyapunov function $\Phi^0$ introduced in \eqref{eq:lyapunov_function}.
\end{remark}

\section{Algorithm Design}
\label{sec:alg_des}

We introduce the \algname{Byz-Clip21-SGD2M} algorithm, designed to operate under both Byzantine failures and differential privacy noise. The main steps are outlined in \Cref{alg:byz_clip21_sgd2m}, followed by a detailed discussion highlighting the key design innovations and their rationale.

\paragraph{Momentum Mechanism.}

\algname{Byz-Clip21-SGD2M} employs two momentum buffers, $\{v_i^t\}_{i\in\cG}$ on the client side and $\{m_i^t\}_{i\in\cG}$ on the server side, both essential for mitigating noise during training. On the client side, momentum with parameter $\beta$ smooths mini-batch gradient noise; without it, the method may fail to converge, as seen in algorithms such as \algname{EF21-SGD} and \algname{Clip21-SGD}, which provably fail to converge to arbitrary accuracy $\varepsilon$ under mini-batch noise \citep{fatkhullin2024momentum, islamov2025double}. On the server side, momentum with parameter $\hbeta$ counteracts the accumulation of DP-noise in the buffers $\{m_i^t\}_{i\in\cG}$, preventing it from destroying convergence due to fast accumulation of the DP noise in the updates.

\paragraph{Clipping.}
It is crucial to bound model updates before injecting DP-noise. Without clipping, updates may have arbitrarily large norms, which would dominate the injected noise and invalidate the privacy guarantee. Clipping restricts each update to a controlled range, ensuring that the privacy guarantees remain meaningful and that the added noise achieves its intended effect. In \algname{Byz-Clip21-SGD2M}, we employ the standard norm-clipping mechanism, which rescales updates 
relative to a threshold~$\tau$: 
\begin{equation}\label{eq:clipping} 
\clip_{\tau}(x) = \begin{cases}
    \frac{\tau}{\|x\|}x, &\text{ if } \|x\| > \tau,\\
    x, &\text{ if } \|x\| \le \tau.
\end{cases}
\end{equation}
This procedure ensures a fixed sensitivity at each update, after which Gaussian noise is added to the clipped values to provide local $(\varepsilon, \delta)$-DP guarantees. Equivalently, the same effect can be interpreted as adding noise directly to the averaged client update. Consequently, \algname{Byz-Clip21-SGD2M} is compatible with all of the techniques discussed above and can also be deployed in the central DP setting, where $\sigma_\omega^2$ 
scales down with the number of clients.

\paragraph{Error Feedback.} 
To mitigate client drift, which arises from the potentially different local data distributions across clients, we augment our method with an \algname{EF21}-style error feedback mechanism on the client side. Error feedback controls the discrepancy between the true global update and local updates sent to the server. Incorporating this mechanism into \algname{Byz-Clip21-SGD2M} allows the algorithm to compensate for the bias introduced by heterogeneous client data, thereby ensuring convergence under larger levels of data heterogeneity. This design choice makes the method robust to one of the central difficulties in federated optimization and allows eliminating the need for the bounded gradient assumption imposed in prior work.

\paragraph{Robust Aggregation.}

At the heart of robust \algname{Byz-Clip21-SGD2M} lies the aggregation rule ${\rm RAgg}$. To mitigate adversarial clients that attempt to disrupt training, we employ a robust aggregation rule satisfying \Cref{def:arrag}. Robust aggregation limits the damage that Byzantine clients can cause during training.

\paragraph{Challenges of Designing Differential Private and Byzantine-Robust Algorithms.}

Designing algorithms that are both differentially private and Byzantine-robust is inherently challenging as the two objectives often conflict with each other. Differential privacy requires adding extra noise to the raw updates, which reduces their consistency and makes it easier for adversarial updates to remain undetected. Moreover, using robust aggregation to combine the DP-noisy updates can amplify the detrimental effects of noise, further hindering convergence. Consequently, theoretical analyses must be carried out with particular care and often under restrictive assumptions. Overcoming these limitations requires carefully designing algorithms that balance privacy, robustness, and practical learnability.


        

\section{Theoretical Analysis}
\label{sec:theory}

\begin{algorithm}[t]
\caption{\algname{Byz-Clip21-SGD2M}}
\label{alg:byz_clip21_sgd2m}
\begin{algorithmic}[1]
\REQUIRE  $x^0 \in \R^d,$ momentum parameters $\beta,\hbeta\in(0,1],$ step-size $\gamma > 0$, $g_i^0=m_i^0\in\R^d,$ clipping parameter $\tau > 0$, DP-noise variance $\sigma_{\omega}^2 \ge 0$
    \FOR{$t=0, \ldots, T-1$}
        \STATE $x^{t+1} = x^t - \gamma g^t$
        \FOR{$i\in\cG$}
            \STATE $v_i^{t+1} = (1-\beta)v_i^t + \beta\nabla f_i(x^{t+1}, \xi^{t+1}_i)$
            \STATE $\omega_i^{t+1} \sim \cN(0, \sigma_{\omega}^2\mI)$
            \STATE $c_i^{t+1} = \clip_{\tau}(v_i^{t+1} - g_i^t) + \omega_i^{t+1}$ 
            \STATE $g_i^{t+1} = g_i^t + \hbeta\clip_{\tau}(v_i^{t+1} - g_i^t)$
        \ENDFOR
        \FOR{$i\in\cB$}
            \STATE $c_i^{t+1} = (*)$ \hfill{sends arbitrary vector}
        \ENDFOR
        \STATE $m_i^{t+1} = m_i^t + \hbeta c_i^{t+1}~$ ${}^{(*)}$
         \STATE $g^{t+1} =\aragg(m_1^{t+1},\dots,m_n^{t+1})$
        
    \ENDFOR
\end{algorithmic}	
{\small ${}^{(*)}$ This step can be moved to the client side without affecting the DP guarantee and convergence analysis, thanks to the post-processing property \citep[Proposition 2.1]{dwork2014algorithmic}.}
\end{algorithm}

Our convergence analysis is based on a carefully designed Lyapunov function
\begin{equation}
    \label{eq:lyapunov_function}
    \Phi^t \eqdef 
      f(x^t) - f^\star + \frac{8\gamma\beta}{\hbeta^2\eta^2}\frac{1}{G}\sum_{i\in\cG}\|v_i^t-\nabla f_i(x^t)\|^2 + \frac{2\gamma}{\hbeta\eta}\frac{1}{G}\sum_{i\in\cG}\|g_i^t-v_i^t\|^2 + \frac{2\gamma}{\beta}\|\overline{v}^t-\nabla f(x^t)\|^2, 
\end{equation}
where $\overline{v}^t \eqdef \frac{1}{G}\sum_{i\in\cG}v_i^t$, and $\eta \sim \tau$, whose value will be specified in the main convergence theorem. 

The decrease of this Lyapunov function guarantees convergence of the algorithm, as it implies that $f(x^t) -f^\star$ decreases during the training. Moreover, as the iterations proceed, the auxiliary variables $v_i^t$ and $g_i^t$ become increasingly accurate approximations of $\nabla f_i(x^t)$ and $v_i^t$, respectively. In other words, the momentum buffers provide a smoothed estimate of the true gradients, demonstrating the benefit of momentum in reducing variance and effectively mitigating the noise introduced during training. The coefficients in the Lyapunov function are carefully chosen to balance its different components, ensuring that each term contributes at the same order and none dominates the others. The detailed proofs are deferred to \Cref{apx:main_theorem,apx:special_cases}.

\begin{restatable}[Simplified]{theorem}{maintheorem}
\label{th:main_theorem}
Let Assumptions~\ref{asmp:smoothness}, \ref{asmp:stoch_grad}, and \ref{asmp:bounded_heterogeneity} hold, and $\alpha\in(0,1)$ be a failure probability.
    Let $\wtilde{B}_{\rm init} \eqdef \max_{i\in\cG}\{\|\nabla f_i(x^0)\|\} > 3\tau$, $\eta \sim \nicefrac{\tau}{\wtilde{B}_{\rm init}}$, and $\Delta \ge \Phi^0$ for $\Phi^0$ defined in \eqref{eq:lyapunov_function}. Then there exists a choice of hyperparameters $\gamma, \beta$, and $\hbeta$ (see a full statement in \Cref{th:main_theorem_full}) such that the iterates of \algname{Byz-Clip21-SGD2M} (Alg.~\ref{alg:byz_clip21_sgd2m}) run with DP-noise variance $\sigma_{\omega}^2 > 0$ and $(c,\deltabyz)$-robust aggregator $(\deltabyz > 0)$ satisfy with probability at least $1-\alpha$ that 
    $$ \frac{1}{T}\sum_{t=0}^{T-1}\|\nabla f(x^t)\|^2 \le \wtilde{\cO}\left(\Delta_0 (R + R^{2/3})\right),$$
    where 
    \begin{equation*} 
        R = \frac{\sqrt{d}\sigma_{\omega}}{\sqrt{GT}\tau} + \frac{\sqrt{c\deltabyz d}\sigma_{\omega}}{\sqrt{T}\tau} + \sqrt{c\deltabyz}, \quad F^0 \eqdef f(x^0)-f^\star,
    \end{equation*}
    \[ \Delta_0 \eqdef \sqrt{L\Delta}(\sqrt{L\Delta} + \wtilde{B} + \sigma) = \wtilde{\cO}(LF^0 + \zeta^2 + \sigma(\sqrt{LF^0} +\zeta)),
    \] 
    and $\wtilde{\cO}$ hides constant and logarithmic factors and higher order terms decreasing in $T$.\footnote{The failure probability $\alpha$ appears in the logarithmic terms similar to prior work \citep{fang2019sharp,bassily2021differentially,sadiev2023high}.}
\end{restatable}

\begin{remark}
\label{rem:2}
We stress that \Cref{asmp:bounded_heterogeneity} is invoked only to bound the constant $\Delta_0$; the convergence guarantee itself still holds without this condition due to the use of clipping. This general result highlights that bounded heterogeneity is primarily used to control the norms of the updates. In our setting, the updates $c_i^{t+1}$ are used to update the momenta $m_i^{t+1}$, which are then aggregated through $\aragg$. The stability of $m_i^{t+1}$ is maintained by the clipping operator, since without such control the norm could otherwise grow unbounded. Moreover, due to our design choice $\hat\beta \sim 1/T$, the growth rate of the norm is limited to $\mathcal{O}(t\tau\hat\beta) = \mathcal{O}(1)$, ensuring that the momenta remain well-behaved throughout the training process. 
\end{remark}

\begin{figure*}[t]
    \centering
    \begin{tabular}{c}

        \hspace{-2mm}\includegraphics[width=0.9\linewidth]{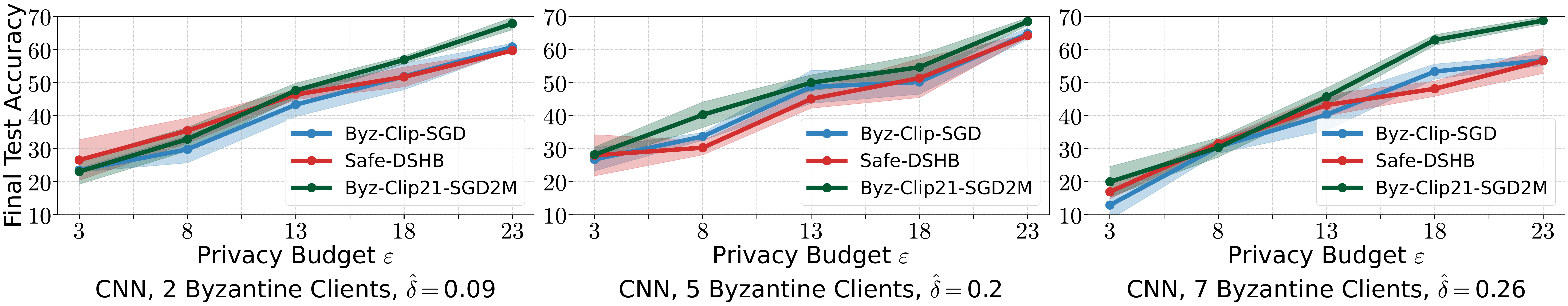} \\

        \hspace{-2mm}\includegraphics[width=0.9\linewidth]{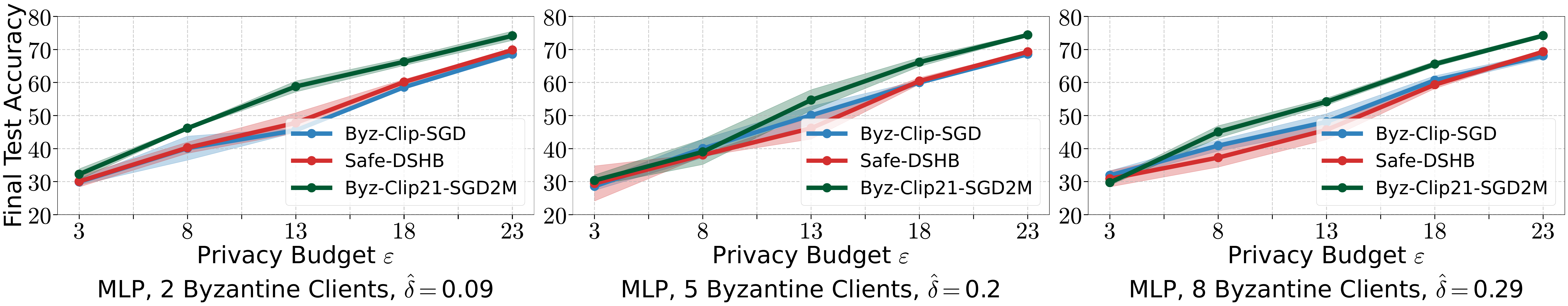}



    \end{tabular}

    \caption{Performance of \algname{Byz-Clip21-SGD2M}, \algname{Byz-Clip-SGD} (\Cref{alg:byz_clip_sgd}), and \algname{Safe-DSHB} (\Cref{alg:safe_dshb}) when training CNN (top line) and MLP (bottom line) models on the MNIST dataset for different numbers of Byzantine clients and privacy budgets, when Byzantine clients use \algname{IPM} attack.}
    \label{fig:cnn_test_acc}
\end{figure*}

Next, we consider the implication of this theorem when $\sigma_{\omega}$ is set in a special way such that each step of \algname{Byz-Clip21-SGD2M} satisfies local-$(\varepsilon,\delta)$-DP.

\begin{restatable}[Simplified]{corollary}{corollarywithdpwithbyz}\label{cor:withdpwithbyz} Under the setup of \Cref{th:main_theorem}, let $\sigma_{\omega} = \Theta\left(\frac{\tau}{\varepsilon}\sqrt{T\log(\frac{T}{\delta})\log(\frac{1}{\delta})}\right)$ for some $\varepsilon,\delta\in(0,1)$. Then there exists a choice of hyperparameters $\gamma,\beta$, and $\hbeta$ (see a full statement in \Cref{cor:cor_main_theorem_full}) such that all $T$ iterations of \algname{Byz-Clip21-SGD2M} satisfy local $(\varepsilon, \delta)$-DP with robustness-privacy-utility trade-off bounded as
\begin{equation*} 
  \frac{1}{T}\sum_{t=0}^{T-1}\|\nabla f(x^t)\|^2 \le  \wtilde{\cO}\left(\Delta_0(R + R^{2/3})\right),
\end{equation*}
with probability at least $1-\alpha$, where 
\begin{equation*} 
    R = \textcolor{orange}{\frac{\sqrt{d}}{\sqrt{G}\varepsilon}} + \textcolor{blue}{\sqrt{c\deltabyz}\frac{\sqrt{d}}{\varepsilon}} +  \textcolor{purple}{\sqrt{c\deltabyz}}.
\end{equation*}
\end{restatable}
In \Cref{cor:withdpwithbyz}, the \textcolor{orange}{first term} corresponds to the error arising from privacy, while the \textcolor{purple}{last term}  reflects the error introduced by the presence of Byzantine clients. The \textcolor{blue}{middle term} represents the additional penalty incurred when enforcing both requirements simultaneously. 
Our utility bound takes a similar form as the lower bound in the strongly convex case \citep{allouah2023privacy}, which demonstrates that our analysis is essentially tight. In particular, \Cref{th:main_theorem} shows that the impacts of privacy and Byzantine adversaries multiply. The privacy analysis of \algname{Byz-Clip21-SGD2M} then follows from the well-known result in (\citet[Theorem 3.22]{dwork2014algorithmic} in combination with the advanced composition theorem (\citep[Theorem 3.20]{dwork2014algorithmic}), we obtain that all $T$ iterations of \algname{Byz-Clip21-SGD2M} are local $(\varepsilon, \delta)$-differentially private.

\paragraph{Comparison to prior work.}

We analyze \algname{Byz-Clip21-SGD2M} under the standard $L$-smoothness and $\sigma$-sub-Gaussian assumptions, thereby improving over earlier work that required the far more restrictive bounded-gradient condition \citep{zhu2022bridging, allouah2023privacy}. To our knowledge, the only high-probability study of optimization algorithms in the simultaneous presence of DP and Byzantine adversaries is due to \citet{xiang2023practical}; however, their analysis is limited to a single attack model and relies on strong assumptions, thereby limiting its applicability. Other works rely on expectation-based analysis, yet still impose restrictive assumptions; see \Cref{tab:summary_experiments} for a concrete comparison. By contrast, our convergence guarantee for \algname{Byz-Clip21-SGD2M} depends only logarithmically on the failure probability  $\alpha$, which is consistent with standard high-probability analyses, and holds for any aggregation rule satisfying \Cref{def:arrag}. 


\paragraph{Main Challenges in the Analysis.}

Server-side aggregation induces a bias between the server aggregate $g^t$ and the honest-average update $\frac{1}{G}\sum_{i\in\cG}m_i^t$. This bias scales multiplicatively with the DP-noise variance $\sigma_{\omega}$ and $c\deltabyz$; controlling it without assuming that the gradients are bounded requires a careful analysis with well-calibrated hyperparameters. To address this, we treat clipping as a contractive compressor, with the complication that its contraction factor is input-dependent, which motivates a high-probability analysis. Using a refined inductive argument, we show that the inputs to the clipping operator remain bounded with high probability. Finally, obtaining only poly-logarithmic dependence on the failure probability $\alpha$ hinges on a careful decomposition of the noise terms, each controlled via concentration inequalities.

\subsection{Convergence in Special Cases}

We now examine several special cases of our bound, which showcase both the tightness of our analysis and the broad applicability of \algname{Byz-Clip21-SGD2M}. Most follow directly from \Cref{th:main_theorem}.


\paragraph{Absence of DP and Byzantine Adversaries.} Setting $\sigma_{\omega}=0$ and $\deltabyz=0$ significantly simplifies the choice of hyperparameters. In particular, we can set $\gamma\sim \nicefrac{\beta}{L}$ and $\hbeta=1$. The following corollary provides a formal convergence in this setting, obtained directly from \Cref{th:main_theorem}.

\begin{figure*}[t]
    \centering
    \begin{tabular}{c}

        \vspace{-1mm}
        \hspace{-2mm}\includegraphics[width=\linewidth]{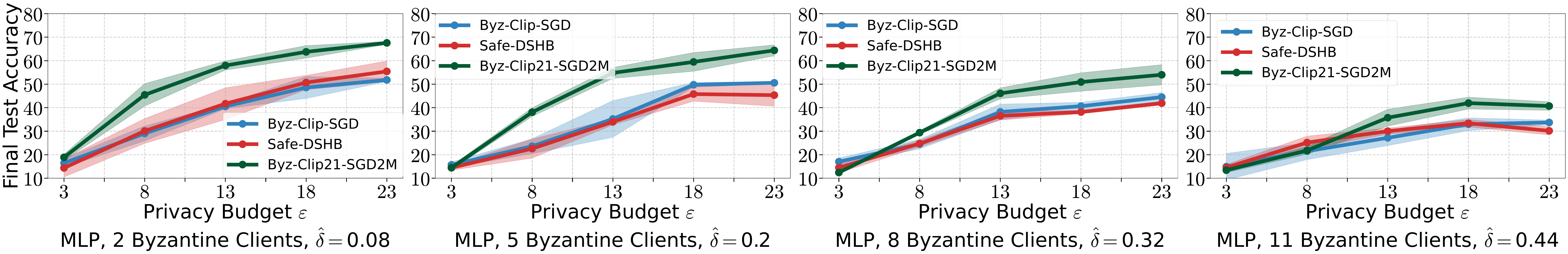} \\
    
        \hspace{-2mm}\includegraphics[width=\linewidth]{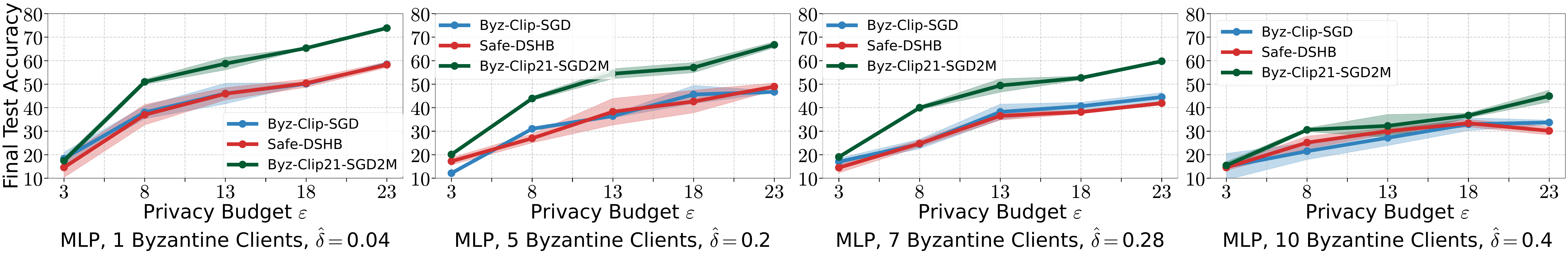}



    \end{tabular}
    \caption{Performance of \algname{Byz-Clip21-SGD2M}, \algname{Byz-Clip-SGD} (\Cref{alg:byz_clip_sgd}), and \algname{Safe-DSHB} (\Cref{alg:safe_dshb}) when training CNN (top line) and MLP (bottom line) models on the MNIST dataset for different numbers of Byzantine clients and privacy budgets, when Byzantine clients use label flipping attack.}
    \label{fig:label_flip}
    \vspace{-4mm}
\end{figure*}

\begin{restatable}[Simplified]{corollary}{nodpnobyz}\label{eq:corollary_no_dp_no_byz}
    Under assumptions of \Cref{th:main_theorem}, let $\sigma_{\omega}=0$ and $\deltabyz=0$. Then, there exists a set of hyperparameters $\gamma \sim \nicefrac{\beta}{L}$ and $\hbeta=1$ such that the iterates of \algname{Byz-Clip21-SGD2M} satisfy with probability at least $1-\alpha$ that $\frac{1}{T}\sum_{t=0}^{T-1}\|\nabla f(x^t)\|^2$ is bounded by
    \begin{equation*} 
    \wtilde{\cO}\left(\frac{L\Delta(1+\wtilde{B}/\tau)}{T} + \frac{\sigma(\sqrt{L\Delta} + \wtilde{B} + \sigma)}{\sqrt{GT}}\right)
    \end{equation*}
\end{restatable}
This result yields the standard $\wtilde{\cO}(\nicefrac{1}{\sqrt{GT}})$ convergence rate, which matches known lower bounds in this regime \citep{arjevani2023lower}. Importantly, it is established under only the standard assumptions of $L$-smoothness and $\sigma$-sub-Gaussian gradient noise, thereby improving over prior results \citep{liu2022communication, noble2022differentially, allouah2024privacy}.


\paragraph{DP in the absence of Byzantine Clients.} As a next special case, we consider the setting where data is protected solely through DP noise injection, with no Byzantine clients present (i.e., $\deltabyz = 0$). In this regime, the privacy-utility trade-off can be obtained by simply setting $\deltabyz=0$ in \Cref{cor:withdpwithbyz}.

\begin{restatable}[Simplified]{corollary}{corollarywithdpnobyz}\label{cor:withdpnobyz} Under the setup of \Cref{th:main_theorem}, let $\sigma_{\omega} = \Theta\left(\frac{\tau}{\varepsilon}\sqrt{T\log(\frac{T}{\delta})\log(\frac{1}{\delta})}\right)$ for some $\varepsilon,\delta\in(0,1)$ and $\deltabyz=0$. Then there exists a choice of hyperparameters $\gamma,\beta$, and $\hbeta$ such that the iterates of \algname{Byz-Clip21-SGD2M} satisfy with probability at least $1-\alpha$ that $\frac{1}{T}\sum_{t=0}^{T-1}\|\nabla f(x^t)\|^2$ is bounded by
\begin{equation*} 
  \sqrt{L\Delta}(\sqrt{L\Delta} + \wtilde{B} + \sigma)\cdot \wtilde{\cO}\left(\frac{\sqrt{d}}{\sqrt{G}\varepsilon} + \left(\frac{\sqrt{d}}{\sqrt{G}\varepsilon}\right)^{2/3}\right).
\end{equation*}
\end{restatable}

In this setting, our general theorem subsumes the results of \citet{islamov2025double}. Moreover, in modern applications where the model size $d$ is much larger than the number of regular clients $G$ \citep{charles2024fine, chua2024mind}, the dominant privacy–utility term is $\wtilde{\cO}(\nicefrac{\sqrt{d}}{\sqrt{G}\varepsilon})$, matching known lower bounds in this regime \citep{duchi2018minimax}. Moreover, unlike prior work, this result does not rely on unrealistic problem assumptions.


\paragraph{Byzantine robustness in the absence of DP-noise.} In this regime, clipping, originally required for privacy guarantees, is no longer needed. Accordingly, we set $\tau=+\infty$ and $\hbeta=1$ in \algname{Byz-Clip21-SGD2M}, which reduces to the algorithm analyzed by \citet{karimireddy2020byzantine}; see \Cref{alg:byz_clip21_sgd2m_nodp}. This simplification also removes the need to store the server-side momentum buffers $\{m_i^t\}_{i=1}^n$. Since applying \Cref{th:main_theorem} directly in this case would yield weaker guarantees, we instead develop a dedicated high-probability analysis, leading to much simpler step-size conditions. Furthermore, the Lyapunov function itself becomes considerably simpler, as the buffers $\{g_i^t\}_{i\in\cG}$ are no longer required (see \Cref{eq:lyapunov_function_nodp}).

\begin{restatable}[Simplified]{theorem}{corollarynodpwithbyz}\label{th:nodp_withbyz} Let \Cref{asmp:smoothness}, \Cref{asmp:stoch_grad} and \Cref{asmp:bounded_heterogeneity_avg} hold. Then there exists a choice of hyperparameters $\gamma,\beta$, and $\hbeta=1$ such that the iterates of \algname{Byz-Clip21-SGD2M} with $\sigma_{\omega} = 0$ satisfy with probability at least $1-\alpha$ that $\frac{1}{T}\sum_{t=0}^{T-1}\|\nabla f(x^t)\|^2$ is bounded by
\begin{equation*} 
    \wtilde{\cO}\left(\frac{LF^0}{T}
        + \frac{\sigma(\sqrt{LF^0} + \sigma/\sqrt{G} + \sqrt{c\deltabyz}(\zeta +\sigma)}{\sqrt{GT}}
        + c\deltabyz \zeta^2\right). 
\end{equation*}
\end{restatable}

We establish that \algname{Byz-Clip21-SGD2M} converges at rate $\wtilde{\cO}(\nicefrac{1}{\sqrt{GT}})$ to a neighborhood of size $\wtilde{\cO}(c\deltabyz\zeta^2)$, matching the in-expectation lower bound of \citet{karimireddy2020byzantine}.

\section{Experiments}\label{sec:main_exp}

To demonstrate the efficacy of the proposed algorithm, we test the performance of \algname{Byz-Clip21-SGD2M} against \algname{Byz-Clip-SGD} (\Cref{alg:byz_clip_sgd}) and \algname{Safe-DSHB} (\Cref{alg:safe_dshb}) \citep{allouah2023privacy} when training CNN and MLP models on the MNIST dataset \citep{lecun2010mnist}. We fix the number of regular clients to $20$ and distribute the dataset among them equally. Then, we add Byzantine clients that perform \algname{IPM} attack \citep{xie2020fall} on the vectors transmitted from the clients to the server. In particular, each Byzantine client computes the average of the transmitted vectors of regular clients and multiplies the average by $-10$ as described in \citet{xie2019zeno} (\Cref{fig:cnn_test_acc}). We test performance when varying the number of Byzantine clients and the privacy budget $\varepsilon\in\{3, 8, 13, 18, 23\}$. For each algorithm we perform an extensive tuning of the learning rate parameter in $\{10, 1, 10^{-1}, 10^{-2}, 10^{-3}\}$ and clipping threshold $\tau \in \{1, 10^{-1}, 10^{-2}, 10^{-4}, 10^{-5}, 10^{-6}\}$. For \algname{Byz-Clip21-SGD2M} and \algname{Safe-DSHB}, we fix the local momentum parameter $\beta=0.1$, while we use $\hbeta=0.01$ for \algname{Byz-Clip21-SGD2M}. In such a setting, every algorithm allocates an equal amount of privacy budget for tuning, since the set of tested hyperparameters remains the same across algorithms. We test the performance of algorithms with a batch size $64$ for MLP and $32$ for CNN training. Consistent with our theory, where we assume access to $\sigma$-sub-Gaussian stochastic gradients, we disable privacy amplification by sub-sampling and inject DP noise with standard deviation $\sigma_{\omega} = \frac{\tau}{\varepsilon}\sqrt{T\log\left(\nicefrac{1}{\delta}\right)},$ where $T$ is the total number of iterations. 

The results in \Cref{fig:cnn_test_acc} present the mean and one standard deviation across three random seeds. In all configurations, \algname{Byz-Clip21-SGD2M} is competitive, matching or surpassing the baselines in test accuracy. These findings align with our theory and provide empirical evidence that \algname{Byz-Clip21-SGD2M} is an effective method for training under DP noise injection and Byzantine attacks. Additional experimental results and training details are summarized in \Cref{sec:additional_exp}.

\section{Conclusion and Limitations}

In this work, we present \algname{Byz-Clip21-SGD2M}, a new algorithm that admits provable convergence under standard $L$-smoothness and $\sigma$-sub-Gaussian gradient noise in the simultaneous presence of DP and Byzantine adversaries. Several directions merit further study: $(i)$ establishing lower bounds without the bounded-gradient assumption of \citet{allouah2023privacy}; $(ii)$ tightening our convergence guarantees to match known lower bounds exactly when DP and Byzantine adversaries are considered separately; $(iii)$ extending the analysis of \algname{Byz-Clip21-SGD2M} to heavy-tailed gradient noise and relaxed smoothness assumptions \citep{zhang2020improved,alimisis2025we}; $(iv)$ strengthening the privacy guarantees via amplification by data and client subsampling.

\section{Acknowledgment}
The research reported in this publication was supported by funding from King Abdullah University of Science and Technology (KAUST): i) KAUST Baseline Research Scheme, ii) CRG Grant ORFS-CRG12-2024-6460, and iii) Center of Excellence for Generative AI, under award number 5940.

\bibliography{references}
\bibliographystyle{plainnat}

\newpage
\appendix
\counterwithin{figure}{section}
\counterwithin{table}{section}

\vbox{
  {\hrule height 2pt \vskip 0.15in \vskip -\parskip}
  \centering
  {\LARGE\bf Appendix\par}
  {\vskip 0.2in \vskip -\parskip \hrule height 0.5pt \vskip 0.09in}
}

\newcommand\invisiblepart[1]{%
  \refstepcounter{part}%
  \addcontentsline{toc}{part}{\protect\numberline{\thepart}#1}%
}

\invisiblepart{Appendix}
\setcounter{tocdepth}{2}
\localtableofcontents

\appendixtrue

\section{Notation}

\begin{table*}[!h]
    \centering
    
    \caption{ { \small Summary of notation used in the paper, including convergence proofs.}}
    \label{tab:notation}
    \resizebox{\textwidth}{!}{
        \begin{tabular}{cccc}
            \toprule

            {\bf Symbol} &
            {\bf Definition} &
            {\bf Meaning} & 
            {\bf Reference} 
            
            \\ \toprule

            $n$ &
            -- &
            Number of all clients &
            \Cref{sec:prem} \\
            
            $x$ &
            -- &
            Model parameters &
            \Cref{sec:prem} \\

            $d$ &
            -- &
            Number of model parameters &
            \Cref{sec:prem} \\

            $\cG$ &
            -- &
            Set of regular clients &
            \Cref{sec:prem} \\

            $G$ &
            -- &
            Number of regular clients &
            \Cref{sec:prem} \\

            $\cB$ &
            -- &
            Set of Byzantine clients &
            \Cref{sec:prem} \\

            $\deltabyz$ &
            -- &
            Fraction of Byzantine clients &
            \Cref{sec:prem} \\

            $\varepsilon, \delta$ &
            -- &
            Privacy budget constants &
            \Cref{def:privacy} \\

            $\tau$ &
            -- & 
            Clipping threshold &
            \Cref{eq:clipping} \\

            $c$ &
            -- &
            Scaling constant in aggregation rule &
            \Cref{def:arrag} \\

            $L$ &
            -- &
            Smoothness constant &
            \Cref{asmp:smoothness} \\

            $A_{\zeta}, \zeta^2$ & 
            \makecellnew{$\frac{1}{G}\sum_{i\in\cG}\|\nabla f_i(x) - \nabla f(x)\|^2 \le A_{\zeta}\|\nabla f(x)\|^2 + \zeta^2$ or \\
            $\max_{i\in\cG}\|\nabla f_i(x)-\nabla f(x)\|^2 \le A_{\zeta}\|\nabla f(x)\|^2 + \zeta^2$} & 
            Heterogeneity bound &
            \makecellnew{\Cref{asmp:bounded_heterogeneity_avg} or \\
            \Cref{asmp:bounded_heterogeneity}
            } \\

            $f^\star$ &
            $f^\star \eqdef \inf_{x\in\R^d} f(x) > -\infty$ &
            Lower bound on the function value &
            \Cref{asmp:smoothness} \\ 

            $T$ &
            -- &
            Number of iterations &
            -- \\

            \midrule

            $\nabla f_i(x^t, \xi_i^t)$ &
            -- &
            Stochastic gradient at iteration $t$ of client $i\in\cG$ &
            \Cref{asmp:stoch_grad} \\

            $\nabla f(x^t, \xi^t)$ &
            $\nabla f(x^t,\xi^t) \eqdef \frac{1}{G}\sum_{i\in\cG} \nabla f_i(x^t, \xi^t_i)$ &
            \makecellnew{Averaged stochastic gradient at iteration $t$\\ of all regular clients } &
            \Cref{eq:def_averages} \\

            $\overline{v}^t$ &
            $ \overline{v}^t \eqdef \frac{1}{G}\sum_{i\in\cG} v_i^t$ &
            Averaged client momentum buffer &
            \Cref{eq:def_averages} \\

            $\overline{g}^t$ &
            $\overline{g}^t \eqdef  \frac{1}{G}\sum_{i\in\cG} g_i^t$ &
            Averaged \algname{EF21} learnable shifted vector &
            \Cref{eq:def_averages} \\

            $\overline{m}^t$ &
            $\overline{m}^t \eqdef  \frac{1}{G}\sum_{i\in\cG} m_i^t$ &
            Averaged server momentum buffer &
            \Cref{eq:def_averages} \\

            $\omega_i^t$ &
            -- &
            \makecellnew{DP noise added by client $i\in\cG$ \\ at iteration $t$} & 
            \Cref{alg:byz_clip21_sgd2m} \\

            $\Omega_i^t$ &
            $\Omega_i^t \eqdef \sum_{l=1}^t\omega_i^l, \quad 
    \overline{\Omega}^t \eqdef  \frac{1}{G}\sum_{i\in\cG}\Omega_i^t$ &
            Accumulated DP noise of client $i\in\cG$ &
            \Cref{eq:def_averages} \\

            $\overline{\Omega}^t$ &
            $\overline{\Omega}^t \eqdef  \frac{1}{G}\sum_{i\in\cG}\Omega_i^t$ &
            \makecellnew{Averaged accumulated DP noise\\ of all regular clients} &
            \Cref{eq:def_averages} \\

            $\theta_i^t$ &
            $\theta_i^t \eqdef \nabla f_i(x^t,\xi^t_i) - \nabla f_i(x^t)$ &
            \makecellnew{Mini-batch noise at iteration $t$ \\ of regular client $i\in\cG$}&
            \Cref{asmp:stoch_grad} \\

            $\theta^t$ &
            $\theta^t \eqdef \frac{1}{G}\sum_{i\in\cG}\theta^t_i$ &

            \makecellnew{Averaged mini-batch noise\\ of all regular clients} &
            \Cref{eq:def_averages} \\

            $\Phi^t$ &
            -- &
            Lyapunov function &
            \Cref{eq:lyapunov_function} \\

            $\wtilde{\Phi}^t$ &
            -- &
            \makecellnew{Lyapunov function in the absence \\ of DP adversaries} &
            \Cref{eq:lyapunov_function_nodp} \\

            $\alpha$ &
            -- &
            Failure probability &
            \Cref{th:main_theorem} \\

            \midrule 

                        $\Delta$ &
            $\Delta \ge \Phi^0$ &
            \makecellnew{Upper bound on the Lyapunov function \\ at initialization} &
            \Cref{lem:bound_vi_t_gi_t} \\
            
            $B_{\rm init}$ &
            $B_{\rm init} \eqdef \max\{3\tau, \max_i\{\|\nabla f_i(x^0)\|\} + b\}$ &
            Sub-optimality of initialization &
            \Cref{th:main_theorem_full} \\

            $a$ &
            $a \eqdef \left(\sqrt{2} + 2\sqrt{3\log\frac{8(T+1)}{\alpha}}\right)\sqrt{d}\sigma_{\omega}\sqrt{\frac{T}{G}}$ &
            High-probability bound on $\Omega^t$ &
            \Cref{eq:constants} \\

            $\ha$ &
            $\ha \eqdef \left(\sqrt{2} + 2\sqrt{3\log\frac{8(T+1)}{\alpha}}\right)\sqrt{d}\sigma_{\omega}\sqrt{T}$ &
            High-probability bound on $\Omega_i^t$ &
            \Cref{eq:constants} \\
            
            $b^2$ & 
            $b^2 \eqdef 2\sigma^2\log\frac{16(T+1)G}{\alpha}$ &
            High-probability bound on $\|\theta_i^t\|$ &
            \Cref{eq:constants} \\

            $\hc^2$ &
            $\hc^2 \eqdef \left(\sqrt{2} + 2\sqrt{3\log\frac{8(T+1)}{\alpha}}\right)^2\sigma^2$ &
            High-probability bound on $\|\theta^t\|$ & 
            \Cref{eq:constants} \\

            $\tilde{b}^2$ &
            $\tilde{b}^2 \eqdef 2\sigma^2\log\left(\frac{6(T+1)G}{\alpha}\right)$ & 
            \makecellnew{High-probability bound on $\|\theta^t_i\|$ \\
            (in the absence of DP adversaries)
            } &
            \Cref{eq:constants_nodp} \\
            
            $\tilde{c}^2$ &
            $\tilde{c}^2 \eqdef \left(\sqrt{2} + 2\sqrt{3\log\frac{8(T+1)}{\alpha}}\right)\sigma^2$ &
            \makecellnew{High-probability bound on $\|\theta^t\|$ \\
            (in the absence of DP adversaries)
            } &
            \Cref{eq:constants_nodp} \\

            $\tilde{z}^2$ &
            $\tilde{z}^2 \eqdef \left(\sqrt{2} + 2\sqrt{3\log\frac{8G(T+1)}{\alpha}}\right)\sigma^2$ &
            \makecellnew{High-probability bound on $\|\sum_{k=0}^{t}(1-\beta)^{t-k}(\theta^t-\theta_i^t)\|$ \\
            (in the absence of DP adversaries)
            } &
            \Cref{eq:constants_nodp}


            \\ \bottomrule
        \end{tabular}
        }
\end{table*}

In our analysis, we make use of the following quantities
\begin{align}\label{eq:def_averages}
    \overline{v}^t &\eqdef \frac{1}{G}\sum_{i\in\cG} v_i^t,\quad 
    \overline{g}^t \eqdef  \frac{1}{G}\sum_{i\in\cG} g_i^t,\quad 
    \overline{m}^t \eqdef  \frac{1}{G}\sum_{i\in\cG} m_i^t, \quad
    \Omega_i^t \eqdef \sum_{l=1}^t\omega_i^l, \quad 
    \overline{\Omega}^t \eqdef  \frac{1}{G}\sum_{i\in\cG}\Omega_i^t,\notag\\
    \theta^t &\eqdef \frac{1}{G}\sum_{i\in\cG} \theta_i^t
    \nabla f(x^t,\xi^t) \eqdef \frac{1}{G}\sum_{i\in\cG}\nabla f_i(x^t,\xi^t_i).
\end{align}
Next, we derive the relation between $\overline{m}^t$ and $\overline{g}^t$ of the form
\begin{align}\label{eq:gt_bar_mt_bar_relation_averaged}
    \overline{m}^t = \overline{g}^t + \hat{\beta}\Omega^t
\end{align}
Indeed, we have $\overline{m}^0 = \overline{g}^0$ by initialization of \Cref{alg:byz_clip21_sgd2m}. Assume that the relation holds at iteration $t$, let us show that it also holds at iteration $t+1:$
\begin{align}
    \overline{m}^{t+1} =& \overline{m}^t + \frac{\hbeta}{G}\sum_{i\in\cG}[\clip_{\tau}(v_i^{t+1} - g_i^t) + \omega_i^{t+1}]\notag\\
    \overset{\text{Ind. Asmp.}}{=}\;& \overline{g}^t + \hbeta\overline{\Omega}^t + \frac{\hbeta}{G}\sum_{i\in\cG} \clip_{\tau}(v_i^{t+1} - g_i^t) + \frac{\hbeta}{G}\sum_{i\in\cG}\omega_i^{t+1}\notag\\
    \overset{\eqref{eq:def_averages}}{=}\;& \frac{1}{G}\sum_{i\in\cG} [g_i^t +\hbeta\clip_{\tau}(v_i^{t+1} - g_i^t)]
    + \frac{\hbeta}{G}\sum_{i\in\cG}[\Omega_i^t + \omega_i^{t+1}]\notag\\
    =\;& \frac{1}{G}\sum_{i\in\cG}g_i^{t+1} 
    + \frac{\hbeta}{G}\sum_{i\in\cG} \Omega_i^{t+1}\notag\\
    =\;& \overline{g}^{t+1}
    + \hbeta \overline{\Omega}^{t+1}.
\end{align}
Similarly, we obtain
\begin{align}\label{eq:gt_bar_mt_bar_relation}
    m_i^t = g_i^t + \hbeta\Omega_i^t.
\end{align}

\section{Useful Lemmas}

First, we start with a key property of the clipping operator.

\begin{lemma}[Lemma 4.1 in \citep{khirirat2023clip21}]\label{lem:clipping_property} The clipping operator satisfies for any $x\in\R^d$
\begin{align}
    \|\clip_{\tau}(x) - x\| \le \max\{\|x\|-\tau, 0\}.
\end{align}
\end{lemma}

Next, in our high probability analysis, we make use of the following concentration inequality.
\begin{lemma}[Lemma C.3 in \citep{gorbunov2019optimal}]\label{lem:concentration_lemma} Let $\{\xi_k\}_{k=1}^N$ be the sequence of random vectors with values in $\R^n$ such that 
\[
\E{\xi_k \mid \xi_{k-1},\dots, \xi_1} = 0 \text{ almost surely, } \forall k\in\{1,\dots,N\},
\]
and set $S_N \eqdef \sum_{k=1}^N \xi_k$. Assume that the sequence $\{\xi_k\}_{k=1}^N$ are sub-Gaussian, i.e.
\[
\E{\exp\left(\nicefrac{\|\xi_k\|^2}{\sigma_k^2} \mid \xi_{k-1},\dots, \xi_1\right)} \le \exp(1) \text{ almost surely, } \forall k\in\{1,\dots,N\},
\]
where $\sigma_2,\dots,\sigma_N$ are some positive numbers. Then for all $\gamma \ge 0$
\begin{equation}
\Prob\left(\|S_N\| \ge (\sqrt{2}+2\gamma)\sqrt{\sum_{k=1}^N\sigma_k^2}\right) \le \exp(-\nicefrac{\gamma^2}{3}).
\end{equation}
\end{lemma}

\section{Descent Lemmas}

\begin{lemma}\label{lem:descent_lemma_in_f}
Let $f$ be $L$-smooth, $F^t \eqdef f(x^t) - f^\star,$ $\{x^t\}$ be generated by \Cref{alg:byz_clip21_sgd2m} with $\gamma \le \frac{1}{2L}$. Then we have 
\begin{align}\label{eq:descent_lemma_in_f}
\begin{aligned}
        f(x^{t+1}) 
        &\le f(x^t) 
        - \frac{\gamma}{2}\|\nabla f(x^t)\|^2
        - \frac{1}{4\gamma}\|x^{t+1} - x^t\|^2
        + 2\gamma\|\nabla f(x^t) - \overline{v}^t\|^2
        + \frac{2\gamma}{G}\sum_{i\in\cG}\|v_i^t - g_i^t\|^2\\
        &\qquad +\; 
        2\gamma\hbeta^2\|\Omega^t\|^2
        + 16\gamma\hbeta^2c\deltabyz \tau^2T^2
        + 16\gamma\hbeta^2\frac{c\deltabyz}{G}\sum_{i\in\cG}\|\Omega_i^t\|^2,
    \end{aligned}
    \end{align}
    
\end{lemma}
\begin{proof}
    Using the derivations from \citep{islamov2025double}, Lemma 2 we first get
    \begin{align*}
        f(x^{t+1}) &\le f(x^t) 
        - \frac{\gamma}{2}\|\nabla f(x^t)\|^2
        - \frac{\gamma}{4}\|g^t\|^2
        + \frac{\gamma}{2}\|\nabla f(x^t) - g^t\|^2.
    \end{align*}
    We continue the derivations as follows 
    \begin{align}\label{eq:jefnjenqoldwnq}
        f(x^{t+1}) &\overset{(i)}{\le} f(x^t) 
        - \frac{\gamma}{2}\|\nabla f(x^t)\|^2
        - \frac{\gamma}{4}\|g^t\|^2
        + 2\gamma\|\nabla f(x^t) - \overline{v}^t\|^2
        + 2\gamma\|\overline{v}^t - \overline{g}^t\|^2
        + 2\gamma\|\overline{g}^t - \overline{m}^t\|^2\notag\\
        &\qquad +\; 
        2\gamma\|\overline{m}^t - g^t\|^2\notag\\
        &\overset{(ii)}{\le} f(x^t) 
        - \frac{\gamma}{2}\|\nabla f(x^t)\|^2
        - \frac{1}{4\gamma}\|x^{t+1} - x^t\|^2
        + 2\gamma\|\nabla f(x^t) - \overline{v}^t\|^2
        + \frac{2\gamma}{G}\sum_{i\in\cG}\|v_i^t - g_i^t\|^2
        \notag\\
        &\qquad +\; 
        2\gamma\hbeta^2\|\Omega^t\|^2
        + 2\gamma\|\overline{m}^t - g^t\|^2,
    \end{align}
    where $(i)$ follows from Jensen's inequality applied to $\|\cdot\|^2$, 
    $(ii)$ is obtained using again Jensen's inequality applied to $\|\cdot\|^2$, the update rule of $x^t,$ and \eqref{eq:gt_bar_mt_bar_relation_averaged}. Now we bound the term $\|\overline{m}^t - g^t\|^2$ using the properties of the aggregator. We have 
    \begin{align}
        \|\overline{m}^t - g^t\|^2 &= \left\|\overline{m}^t - \aragg(m_1^t, \dots, m_n^t)\right\|^2\notag\\
        &\overset{(iii)}{\le} \frac{c\deltabyz}{G}\sum_{i\in\cG}\left\|\overline{m}^t - m_i^t\right\|^2\notag\\
        &\overset{(iv)}{\le} \frac{c\deltabyz}{G^2}\sum_{j,l\in\cG}\left\|m_j^t - m_l^t\right\|^2 
        = \frac{c\deltabyz}{G^2}\sum_{j,l\in\cG,j\neq l}\left\|m_j^t - m_l^t\right\|^2\notag\\
        &\le \frac{c\deltabyz}{G(G-1)}\sum_{j,l\in\cG,j\neq l}\|m_j^t - m_l^t\|^2\notag\\
        &\overset{(v)}{=} \frac{c\deltabyz}{G(G-1)}\sum_{j,l\in\cG, j\neq l}\|g_j^t + \hbeta\Omega_j^t - g_l^t - \hbeta\Omega_l^t\|^2\notag\\
        &\overset{(vi)}{\le}\frac{8c\deltabyz}{G}\sum_{i \in \cG}\left\|\hbeta\sum_{s=0}^t\clip_{\tau}(v_i^{s}-g_i^{s-1}) \right\|^2
        + \frac{8c\deltabyz\hbeta^2}{G}\sum_{i\in\cG}\|\Omega_i^t\|^2 \notag\\
        &\overset{(vii)}{\le} 8c\deltabyz\hbeta^2(t+1)^2\tau^2
        + \frac{8c\deltabyz\hbeta^2}{G}\sum_{i\in\cG}\|\Omega_i^t\|^2,
        \label{eq:agg_inequality}
    \end{align}
    where $(iii)$ follows from the definition of $\aragg$; $(iv)$ -- from Young's inequality, $(v)$ -- from \eqref{eq:gt_bar_mt_bar_relation}, $(vi)$ -- from Young's inequality and update rule of $g_i^t$ together with the initialization $g_i^0=0$, $(vii)$ -- from Young's inequality. 
    Plugging \eqref{eq:agg_inequality} in \eqref{eq:jefnjenqoldwnq} and noting that $(t+1) \le T$, we get 
    \begin{align}\label{eq:aqmnowjdnqw}
        f(x^{t+1}) 
        &\le f(x^t) 
        - \frac{\gamma}{2}\|\nabla f(x^t)\|^2
        - \frac{1}{4\gamma}\|x^{t+1} - x^t\|^2
        + 2\gamma\|\nabla f(x^t) - \overline{v}^t\|^2
        + \frac{2\gamma}{G}\sum_{i\in\cG}\|v_i^t - g_i^t\|^2
        \notag\\
        &\qquad +\; 
        2\gamma\hbeta^2\|\Omega^t\|^2
        + 16\gamma\hbeta^2c\deltabyz \tau^2T^2
        + 16\gamma\hbeta^2\frac{c\deltabyz}{G}\sum_{i\in\cG}\|\Omega_i^t\|^2,
    \end{align}
    that finalizes the proof.
\end{proof}

\begin{lemma}\label{lem:bound_vi_t_gi_t} Let each $f_i$ be $L$-smooth. Then, for the iterates of \Cref{alg:byz_clip21_sgd2m}  we have the following inequality with probability $1$
\begin{align}
\begin{aligned}
        \|v_i^{t+1} - g_i^t\| &\le  (1-\hbeta)\|v_i^t - g_i^{t-1}\|
        + \hbeta\max\left\{0, \|v_i^t - g_i^{t-1}\| - \tau\right\}
        + \beta L\gamma\|\overline{m}^t\|
        + \beta\|\nabla f_i(x^t) - v_i^t\|
        \notag\\
        &\qquad +\; 
        2\beta\hbeta\tau T L\gamma\sqrt{2c\deltabyz}
        + 2\beta\hbeta L \gamma\sqrt{\frac{2c\deltabyz}{G}\sum_{i\in\cG}\|\Omega_i^t\|^2}
        + \beta\|\theta^{t+1}_i\|,
\end{aligned}
\end{align}
where $\theta^t_i \eqdef \nabla f_i(x^t,\xi^t_i) - \nabla f_i(x^t).$
\end{lemma}
\begin{proof}
    First, we follow the steps of the proof of \citet[Lemma 12]{islamov2025double} to derive
    \begin{align}
        \|v_i^{t+1} - g_i^t\| &\le (1-\hbeta)\|v_i^t - g_i^{t-1}\|
        + \hbeta\max\left\{0, \|v_i^t - g_i^{t-1}\| - \tau\right\}
        + \beta L\gamma\|g^t\|
        + \beta\|\nabla f_i(x^t) - v_i^t\|\notag\\
        &\qquad +\; \beta\|\theta^{t+1}_i\|.
    \end{align}
    Next, we continue as follows 
    \begin{align}
        \|v_i^{t+1} - g_i^t\| &\le (1-\hbeta)\|v_i^t - g_i^{t-1}\|
        + \hbeta\max\left\{0, \|v_i^t - g_i^{t-1}\| - \tau\right\}
        + \beta L\gamma\|g^t-\overline{m}^t\|
        + \beta L\gamma\|\overline{m}^t\|
        \notag\\
        &\qquad +\; \beta\|\nabla f_i(x^t) - v_i^t\|
        + \beta\|\theta^{t+1}_i\|.
    \end{align}    
    Now we perform similar derivations as in \Cref{lem:descent_lemma_in_f} to bound $\|g^t-\overline{m}^t\|$ (see \eqref{eq:agg_inequality}), which finally gives the bound
    \begin{align}
        \|v_i^{t+1} - g_i^t\| &\le (1-\hbeta)\|v_i^t - g_i^{t-1}\|
        + \hbeta\max\left\{0, \|v_i^t - g_i^{t-1}\| - \tau\right\}
        + \beta L\gamma\|\overline{m}^t\|
        + \beta\|\nabla f_i(x^t) - v_i^t\|
        \notag\\
        &\qquad +\; \beta L\gamma\sqrt{8c\deltabyz\hbeta^2T^2\tau^2 + 8\hbeta^2\frac{c\deltabyz}{G}\sum_{i\in\cG}\|\Omega_i^t\|^2}
        + \beta\|\theta^{t+1}_i\|\notag\\
        &\overset{(i)}{\le} (1-\hbeta)\|v_i^t - g_i^{t-1}\|
        + \hbeta\max\left\{0, \|v_i^t - g_i^{t-1}\| - \tau\right\}
        + \beta L\gamma\|\overline{m}^t\|
        + \beta\|\nabla f_i(x^t) - v_i^t\|
        \notag\\
        &\qquad +\; 2\beta\hbeta\tau T L\gamma\sqrt{2c\deltabyz}
        + 2\beta\hbeta L \gamma\sqrt{\frac{2c\deltabyz}{G}\sum_{i\in\cG}\|\Omega_i^t\|^2}
        + \beta\|\theta^{t+1}_i\|,
    \end{align}    
    where $(i)$ follows from the triangle inequality: $\sqrt{s+q} \le \sqrt{s} + \sqrt{q}$ for any $s,q\ge 0.$
\end{proof}

\begin{lemma}\label{lem:bound_mt_bar} Let each $f_i$ be $L$-smooth, and $\Delta \ge \Phi^0.$ Assume that the following inequalities hold for the iterates generated by \algname{Byz-Clip21-SGD2M}
\begin{enumerate}
    \item $\gamma \le \frac{1}{24L}$,
    \item $\beta\in[0,1]$, $\hbeta \in[0, \nicefrac{1}{24}]$,
    \item $\|\overline{m}^{t-1}\| \le 
    \sqrt{64L\Delta} + 3(B_{\rm init}-\tau) 
    + 3b 
    + 3\hbeta a 
    + \hbeta\sqrt{2c\deltabyz}\ha
    + 3\sqrt{2c\deltabyz}\hbeta\tau T$,
    
    \item $\|\overline{g}^{t-1}\| \le  
    \sqrt{64L\Delta} + 3(B_{\rm init}-\tau) + 3b + 3\sqrt{2c\deltabyz}\hbeta\tau T;$
    
    \item $\|\nabla f_i(x^{t-1}) - v_i^{t-1}\| \le  
    \sqrt{4L\Delta}
    + \frac{3}{2}(B_{\rm init}-\tau) 
    + \frac{3}{2}b 
    + \hbeta a 
    + 4\hbeta\sqrt{2c\deltabyz}\ha
    + \frac{3}{2}\sqrt{2c\deltabyz}\hbeta\tau T$ for all $i\in\cG;$

    \item $\|v_i^t - g_i^{t-1}\| \le B_{\rm init}$ for all $i\in\cG;$
    \item $\|\Omega^t\| \le a;$
    \item $\|\Omega_i^{t-1}\| \le \ha$ for all $i\in\cG;$
    \item $\|\theta_i^t\| \le b$ for all $i\in\cG;$
    \item $\Phi^{t-1} \le 2\Delta.$
\end{enumerate}
    Then we have 
    \begin{align}
        \|\overline{m}^t\| \le  
        \sqrt{64L\Delta}
        + 3(B_{\rm init}-\tau)
        + 3b
        + 3\hbeta a 
        + \hbeta\sqrt{2c\deltabyz}\ha
        + 3\sqrt{2c\deltabyz}\hbeta\tau T.
    \end{align}
\end{lemma}
\begin{proof}
    We start as follows 
    \begin{align}
        \|\overline{m}^t\|
        &\overset{(i)}{=}\left\|
        \overline{m}^{t-1} + \frac{\hbeta}{G}\sum_{i\in\cG} [\clip_{\tau}(v_i^t - g_i^{t-1}) + \omega_i^t]\right\|\notag\\
        &= \left\|\overline{m}^{t-1}
        + \frac{\hbeta}{G}\sum_{i\in\cG}\left[\nabla f_i(x^t) + (v_i^t - \nabla f_i(x^t)) 
        + \clip_{\tau}(v_i^t - g_i^{t-1}) - (v_i^t - g_i^{t-1})\right]\right.\notag\\
        &\qquad \; \left.- \frac{\hbeta}{G}\sum_{i\in\cG} g_i^{t-1}  
        + \frac{\hbeta}{G}\sum_{i\in\cG}\omega_i^t
        \right\|\notag\\
        &= \left\|\overline{m}^{t-1}
        + \frac{\hbeta}{G}\sum_{i\in\cG}\left[\nabla f_i(x^t) + (v_i^t - \nabla f_i(x^t)) 
        + \clip_{\tau}(v_i^t - g_i^{t-1}) - (v_i^t - g_i^{t-1})\right]\right.\notag\\
        &\qquad \; \left.- \overline{g}^{t-1}  + (1-\hbeta)\overline{g}^{t-1} 
        + \frac{\hbeta}{G}\sum_{i\in\cG}\omega_i^t
        \right\|\notag.
    \end{align}
    where $(i)$ follows from the update rule of $\overline{m}^t$. We continue the derivations as follows 
    \begin{align}
        \|\overline{m}^t\| &\overset{(ii)}{\le} \left\|\overline{m}^{t-1} - \overline{g}^{t-1} + \frac{\hbeta}{G}\sum_{i\in\cG}\omega_i^t\right\|
        + \hbeta\|\nabla f(x^t)\|
        + \frac{\hbeta}{G}\sum_{i\in\cG}\|\clip_{\tau}(v_i^t - g_i^{t-1}) - (v_i^t - g_i^{t-1})\|\notag\\
        &\qquad +\; (1-\hbeta)\|\overline{g}^{t-1}\|
        + \frac{\hbeta}{G}\sum_{i\in\cG}\|v_i^t-\nabla f_i(x^t)\|\notag\\
        &\overset{(iii)}{\le} \left\|\overline{g}^{t-1} +\hbeta\Omega^{t-1} - \overline{g}^{t-1} + \frac{\hbeta}{G}\sum_{i\in\cG}\omega_i^t\right\|
        + \hbeta\|\nabla f(x^{t-1})\|
        + \hbeta\|\nabla f(x^{t-1}) - \nabla f(x^t)\|\notag\\
        &\qquad +\;
        \frac{\hbeta}{G}\sum_{i\in\cG}\|\clip_{\tau}(v_i^t - g_i^{t-1}) - (v_i^t - g_i^{t-1})\|
        + (1-\hbeta)\|\overline{g}^{t-1}\|\notag\\
        &\qquad +\; 
        \frac{\hbeta}{G}\sum_{i\in\cG}\|(1-\beta)v_i^{t-1} + \beta\nabla f_i(x^t,\xi^t_i)-\nabla f_i(x^t)\|,
    \end{align}
    where $(ii)$ follows from the triangle inequality, $(iii)$ -- from \eqref{eq:gt_bar_mt_bar_relation_averaged},  triangle inequality, and update rule of $v_i^t$. Using the definition of $\Omega^t$ we continue as follows
    \begin{align}
        \|\overline{m}^t\|
        &\overset{(iv)}{\le} \hbeta\|\Omega^{t}\|
        + \hbeta\|\nabla f(x^{t-1})\|
        + \hbeta L\gamma\|g^{t-1}\|
        + \frac{\hbeta}{G}\sum_{i\in\cG}\max\left\{0, \|v_i^t - g_i^{t-1}\| - \tau\right\}
        + (1-\hbeta)\|\overline{g}^{t-1}\|\notag\\
        &\qquad +\; 
        \frac{\hbeta}{G}\sum_{i\in\cG}\left((1-\beta)\|v_i^{t-1} -\nabla f_i(x^t)\| + \beta\|\nabla f_i(x^t,\xi^t_i)-\nabla f_i(x^t)\|\right)\notag\\
        &\overset{(v)}{\le} 
        \hbeta\|\Omega^t\|
        + \hbeta\sqrt{2L(f(x^{t-1})-f^\star)}
        + (2-\beta)\hbeta L\gamma\|g^{t-1}\|
        + \hbeta(B_{\rm init}-\tau)
        + (1-\hbeta)\|\overline{g}^{t-1}\|\notag\\
        &\qquad +\; \frac{\hbeta\beta}{G}\sum_{i\in\cG}\|\theta^t_i\|
        + \frac{\hbeta}{G}(1-\beta)\sum_{i\in\cG}\|v_i^{t-1} - \nabla f_i(x^{t-1})\|\notag\\
        &\overset{(vi)}{\le} 
        \hbeta\|\Omega^t\|
        + \hbeta\sqrt{2L\Phi^{t-1}}
        + {\color{black}2}(2-\beta)\hbeta L\gamma\|\overline{m}^{t-1} - g^{t-1}\|
        + {\color{black}2}(2-\beta)\hbeta L\gamma\|\overline{m}^{t-1}\|
        + \hbeta(B_{\rm init}-\tau)\notag\\
        &\qquad +\; 
        (1-\hbeta)\|\overline{g}^{t-1}\|
        + \frac{\hbeta\beta}{G}\sum_{i\in\cG}\|\theta^t_i\|
        + \frac{\hbeta}{G}(1-\beta)\sum_{i\in\cG}\|v_i^{t-1} - \nabla f_i(x^{t-1})\|,\notag
    \end{align}

    where $(iv)$ follows from $L$-smoothness, update rule of $x^t,$ the clipping property from \Cref{lem:clipping_property}, and triangle inequality, $(v)$ -- from $L$-smoothness, triangle inequality, the update rule of $x^t$, the definition of $\theta^t_i$, and the assumption $5$ of the lemma, $(vi)$ -- from the definition of $\Phi^t$ and triangle inequality. Now we use \eqref{eq:agg_inequality} and the assumptions of the lemma to derive
    \begin{align}
        \|\overline{m}^t\| &\overset{(vii)}{\le} \hbeta a 
        + \hbeta\sqrt{4L\Delta} 
        + 4\hbeta L\gamma\sqrt{8c\deltabyz\hbeta^2T^2\tau^2 + \hbeta^2\frac{8c\deltabyz}{G}\sum_{i\in\cG}\|\Omega_i^{t-1}\|^2}\notag\\
        &\qquad +\; 4\hbeta L\gamma\left(  
        \sqrt{64L\Delta} + 3(B_{\rm init}-\tau) + 3b + 3\hbeta a + \hbeta\sqrt{2c\deltabyz}\ha
        + 3\sqrt{2c\deltabyz}\hbeta\tau T\right)
        + \hbeta (B_{\rm init}-\tau)\notag\\
        &\qquad +\; (1-\hbeta)\left(  
        \sqrt{64L\Delta} 
        + 3(B_{\rm init}-\tau) 
        + 3b
        + 3\sqrt{2c\deltabyz}\hbeta\tau T\right)
        + \hbeta\beta b\notag\\
        &\qquad +\; \hbeta(1-\beta)\left(  
        \sqrt{4L\Delta}
        + \frac{3}{2}(B_{\rm init}-\tau)
        + \frac{3}{2}b
        + \hbeta a
        + 4\hbeta\sqrt{2c\deltabyz} \ha
        + \frac{3}{2}\sqrt{2c\deltabyz}\hbeta\tau T\right)\notag\\
        &\overset{(viii)}{\le} \hbeta a 
        + \hbeta\sqrt{4L\Delta} 
        + 8\hbeta^2T\tau L\gamma\sqrt{2c\deltabyz}
        + 8\hbeta^2 L\gamma\sqrt{2c\deltabyz}\ha
        + \hbeta (B_{\rm init}-\tau)
        + \hbeta\beta b
        \notag\\
        &\qquad +\; 4\hbeta L\gamma\left( 
        \sqrt{64L\Delta} + 3(B_{\rm init}-\tau) + 3b + 3\hbeta a + 3\hbeta\sqrt{2c\deltabyz}\ha
        + 3\sqrt{2c\deltabyz}\hbeta\tau T\right)
        \notag\\
        &\qquad +\; (1-\hbeta)\left( 
        \sqrt{64L\Delta} 
        + 3(B_{\rm init}-\tau) 
        + 3b
        + 3\sqrt{2c\deltabyz}\hbeta\tau T\right)
        \notag\\
        &\qquad +\; \hbeta(1-\beta)\left( 
        \sqrt{4L\Delta}
        + \frac{3}{2}(B_{\rm init}-\tau)
        + \frac{3}{2}b
        + \hbeta a
        + 4\hbeta\sqrt{2c\deltabyz} \ha
        + \frac{3}{2}\sqrt{2c\deltabyz}\hbeta\tau T\right),
    \end{align}
    where $(vii)$ follows from assumptions $3$-$5$, $7$, $9$, $10$ $(viii)$ -- from inequality $\sqrt{s+q}\le \sqrt{s}+\sqrt{q}$ for any $a,b\ge 0$. Rearranging terms, we derive 
    \begin{align}
        \|\overline{m}^t\| &\le \sqrt{L\Delta}[2\hbeta + 32\hbeta L\gamma + 8(1-\hbeta) + 2\hbeta(1-\beta)]
        + (B_{\rm init}-\tau)[12\hbeta L\gamma + 3(1-\hbeta) + \hbeta + \nicefrac{3}{2}\hbeta(1-\beta)]\notag\\
        &\qquad +\; b [12\hbeta L\gamma + 3(1-\hbeta) + \hbeta\beta + \nicefrac{3}{2}\hbeta(1-\beta)]
        + a[\hbeta + 12\hbeta^2L\gamma + \hbeta^2(1-\beta)]\notag\\
        &\qquad +\; \sqrt{2c\deltabyz}\ha[8\hbeta^2L\gamma + 12\hbeta^2L\gamma + 4\hbeta^2(1-\beta)]\notag\\
        &\qquad +\; \tau T\sqrt{2c\deltabyz}[\nicefrac{3}{2}\hbeta^2(1-\beta) + 3\hbeta(1-\hbeta) + 12\hbeta^2L\gamma + 8\hbeta^2L\gamma].
    \end{align}
    For the first coefficient, we have 
    \begin{align}
        &\; 2\hbeta + 32\hbeta L\gamma + 8(1-\hbeta) + 2\hbeta(1-\beta) \le 8\notag\\
        \Leftarrow &\; 2\hbeta + 32\hbeta L\gamma + 2\hbeta(1-\beta) \le 8\hbeta\notag\\
        \Leftarrow &\; 8L\gamma \le 1,
    \end{align}
    where the last inequality is satisfied by the choice of the step-size $\gamma \le \frac{1}{24L}.$
    For the second coefficient, we have
    \begin{align}
        &\; 12\hbeta L\gamma + 3(1-\hbeta) + \hbeta + \nicefrac{3}{2}\hbeta(1-\beta) \le 3\notag\\
        \Leftarrow &\; 12\hbeta L\gamma + \hbeta + \nicefrac{3}{2}\hbeta(1-\beta) \le 3\hbeta\notag\\
        \Leftarrow &\; 24 L\gamma \le 1,
    \end{align}
    where the last inequality holds by the choice of the step-size $\gamma \le \frac{1}{24L}$. For the third coefficient, we have 
    \begin{align}
        &\; 12\hbeta L\gamma + 3(1-\hbeta) + \hbeta\beta + \nicefrac{3}{2}\hbeta(1-\beta) \le 3\notag\\
        \Leftarrow &\; 12\hbeta L\gamma + \hbeta\beta + \nicefrac{3}{2}\hbeta(1-\beta) \le 3\hbeta\notag\\
        \Leftarrow &\; 24L\gamma \le 1,
    \end{align}
    where the last inequality holds by the choice of the step-size $\gamma\le \frac{1}{24L}$. For the fourth coefficient, we have
    \begin{align}
        \hbeta + 12\hbeta^2L\gamma + \hbeta^2(1-\beta) \le 3\hbeta \Leftarrow 12L\gamma\hbeta^2 + \hbeta^2 \le 2\hbeta \Leftarrow 12L\gamma \le 1,
    \end{align}
    where the second last inequality holds by the choice of the step-size $\gamma\le \frac{1}{24L},$ and the last inequality holds by the choice of the momentum parameter $\hbeta \le 1.$
    For the fifth coefficient, we have 
    \begin{align}
        &\; 8\hbeta^2L\gamma + 12\hbeta^2L\gamma + 4\hbeta^2(1-\beta) \le \hbeta \Leftarrow 20\hbeta L\gamma + 4\hbeta \le 1,
    \end{align}
    where the last inequality holds by the choice of the step-size $\gamma \le \frac{1}{24L}$ and momentum parameter $\hbeta\le \frac{1}{24}.$
    For the sixth coefficient, we have 
    \begin{align}
        &\;\nicefrac{3}{2}\hbeta^2(1-\beta) + 3\hbeta(1-\hbeta) + 12\hbeta^2L\gamma + 8\hbeta^2L\gamma \le 3\notag\\
        \Leftarrow&\; \nicefrac{3}{2}\hbeta^2 + 20\hbeta^2L\gamma \le 3\hbeta \notag\\
        \Leftarrow&\; 20\hbeta L\gamma \le \frac{3}{2}\notag\\
        \Leftarrow&\; \frac{40}{3}\hbeta L\gamma \le 1,
    \end{align}
    where the last inequality holds by the choice of the step-size $\gamma \le \frac{1}{24L}$ and momentum parameter $\hbeta\le 1$.
\end{proof}

\begin{lemma}\label{lem:bound_nabla_fi_xt_vi_t} Let each $f_i$ be $L$-smooth, $\Delta \ge \Phi^0$, $B_{\rm init}>\tau$. Assume that the following inequalities hold for the iterates generated by \algname{Byz-Clip21-SGD2M} 
\begin{enumerate}
    \item $\gamma \le \frac{1}{24L}$;
    \item $12L\gamma \le \beta$;
    \item $\beta,\hbeta\in[0,1];$
    \item $\|\overline{m}^{t-1}\| \le  
    \sqrt{64L\Delta} + 3(B_{\rm init}-\tau) 
    + 3b 
    + 3\hbeta a 
    + \hbeta\sqrt{2c\deltabyz}\ha
    + 3\sqrt{2c\deltabyz}\hbeta\tau T$,
    
    \item $\|\overline{g}^{t-1}\| \le  
    \sqrt{64L\Delta} + 3(B_{\rm init}-\tau) + 3b
    + 3\sqrt{2c\deltabyz}\hbeta\tau T;$

    \item $\|\nabla f_i(x^{t-1}) - v_i^{t-1}\| \le  
    \sqrt{4L\Delta}
    + \frac{3}{2}(B_{\rm init}-\tau) 
    + \frac{3}{2}b 
    + \hbeta a 
    + 4\hbeta\sqrt{2c\deltabyz}\ha
    + \frac{3}{2}\sqrt{2c\deltabyz}\hbeta\tau T$ for all $i\in\cG;$
    \item $\|v_i^t - g_i^{t-1}\| \le B_{\rm init}$ for all $i\in\cG;$
    \item $\|\Omega^t\| \le a;$
    \item $\|\Omega_i^{t-1}\| \le \ha$ for all $i\in\cG;$
    \item $\|\theta_i^t\| \le b$ for all $i\in\cG;$
    \item $\Phi^{t-1} \le 2\Delta.$
\end{enumerate}
Then we have 
\begin{align}
    \|\nabla f_i(x^t) - v_i^t\| \le  
    \sqrt{4L\Delta} + \frac{3}{2}(B_{\rm init}-\tau) + \frac{3}{2}b + \hbeta a + 4\hbeta\sqrt{2c\deltabyz}\ha 
    + \frac{3}{2}\sqrt{2c\deltabyz}\hbeta\tau T.
\end{align}
\end{lemma}

\begin{proof}
    We have
    \begin{align}
        \|\nabla f_i(x^t) - v_i^t\| &\overset{(i)}{=} \|\nabla f_i(x^t) - (1-\beta)v_i^{t-1} - \beta\nabla f_i(x^t,\xi^t_i)\|\notag\\
        &\overset{(ii)}{\le} (1-\beta)\|\nabla f_i(x^t) - v_i^{t-1}\|
        + \beta\|\nabla f_i(x^{t}) - \nabla f_i(x^t,\xi^t_i)\|\notag\\
        &\overset{(iii)}{\le} (1-\beta)\|\nabla f_i(x^t) - \nabla f_i(x^{t-1})\|
        + (1-\beta)\|\nabla f_i(x^{t-1}) - v_i^{t-1}\|
        + \beta\|\theta^t_i\|\notag\\
        &\overset{(iv)}{\le} (1-\beta)L\gamma\|g^{t-1}\|
        + (1-\beta)\|\nabla f_i(x^{t-1}) - v_i^{t-1}\| 
        + \beta\|\theta^{t}_i\|\notag\\
        &\overset{(v)}{\le} (1-\beta)L\gamma\|\overline{m}^{t-1} - g^{t-1}\|
        + (1-\beta)L\gamma\|\overline{m}^{t-1}\|
        + (1-\beta)\|\nabla f_i(x^{t-1}) - v_i^{t-1}\| 
        + \beta\|\theta^{t}_i\|\notag\\
        &\overset{(vi)}{\le} 
        (1-\beta)L\gamma\left(2\hbeta\tau T\sqrt{2c\deltabyz} + 2\hbeta\sqrt{\frac{2c\deltabyz}{G}\sum_{i\in\cG}\|\Omega_i^{t-1}\|^2}\right)
        + (1-\beta)L\gamma\|\overline{m}^{t-1}\|
        \notag\\
        &\qquad +\; \beta\|\theta^{t}_i\|
        + (1-\beta)\|\nabla f_i(x^{t-1}) - v_i^{t-1}\| \notag\\
        &\overset{(vii)}{\le} 
        2\hbeta L\gamma\sqrt{2c\deltabyz}\tau T
        + 2\hbeta L\gamma(1-\beta)\sqrt{2c\deltabyz}\ha 
        + \beta b \notag\\
        &\qquad +\; (1-\beta)L\gamma\left( 
        \sqrt{64L\Delta} + 3(B_{\rm init}-\tau) + 3b + 3\hbeta a + \hbeta\sqrt{2c\deltabyz}\ha + 3\sqrt{2c\deltabyz}\hbeta\tau T\right)\notag\\
        &\qquad +\; (1-\beta)\left( 
        \sqrt{4L\Delta} + \frac{3}{2}(B_{\rm init}-\tau) + \frac{3}{2}b + \hbeta a + 4\hbeta\sqrt{2c\deltabyz}\ha + \nicefrac{3}{2}\sqrt{2c\deltabyz}\hbeta\tau T \right) 
        \notag\\
        &\le \sqrt{L\Delta}[8L\gamma + 2(1-\beta)]
        + (B_{\rm init}-\tau)[3L\gamma + \nicefrac{3(1-\beta)}{2}]
        + b[3L\gamma + \beta + \nicefrac{3(1-\beta)}{2}]
        \notag\\
        &\qquad +\; 
        \hbeta a[3 L\gamma + (1-\beta)]
        + \sqrt{2c\deltabyz}\hbeta\ha[2 L\gamma + L\gamma + 4(1-\beta)]
        \notag\\
        &\qquad +\; \sqrt{2c\deltabyz}\hbeta\tau T[2L\gamma + 3L\gamma + \frac{3}{2}(1-\beta)],\notag
    \end{align}

    where $(i)$ follows from the update rule of $v_i^t$, $(ii)$-$(iii)$ -- from  triangle inequality and definition of $\theta^t_i$, $(iv)$ -- from $L$-smoothness and the update rule of $x^t$, $(v)$ -- from triangle inequality 
    $(vi)$ -- from \eqref{eq:agg_inequality} and inequality $\sqrt{a+b} \le \sqrt{a} + \sqrt{b}$ for any $a,b\ge 0$, $(vii)$ -- from the assumptions $4$-$7$, $9$, $10$ of the lemma.

    For the first coefficient, we have
    \begin{align}
        8L\gamma + 2(1-\beta) \le 2 \Leftarrow 4L\gamma \le \beta 
    \end{align}
    where the last inequality holds by the choice of the step-size $12L\gamma \le \beta$. For the second coefficient, we have
    \begin{align}
        3L\gamma + \frac{3}{2}(1-\beta) \le \frac{3}{2} \Leftarrow 2L\gamma \le \beta,
    \end{align}
    where the last inequality holds by the choice of the step-size $12L\gamma \le \beta$. For the third coefficient, we have
    \begin{align}
        3L\gamma + \beta + \frac{3}{2}(1-\beta) \le \frac{3}{2} \Leftarrow 6L\gamma \le \beta,
    \end{align}
    where the last inequality holds by the choice of the step-size $12L\gamma \le \beta$. For the fourth coefficient, we have 
    \begin{align}
        3L\gamma + (1-\beta) \le 1 \Leftarrow 3L\gamma \le \beta,
    \end{align}
    the last inequality holds by the choice of the step-size $12L\gamma \le \beta.$ For the fifth coefficient, we have 
    \begin{align}
        &\; 2L\gamma + L\gamma  + 4(1-\beta) \le 4  \Leftarrow 3L\gamma  \le 4\beta,
    \end{align}
    where the last inequality holds by the choice of the step-size $12L\gamma\le \beta$. For the sixth coefficient, we have 
    \begin{align}
        2L\gamma + 3L\gamma + \frac{3}{2}(1-\beta) \le \frac{3}{2}  \Leftarrow \frac{10}{3}L\gamma \le \beta,
    \end{align}
    where the last inequality holds by the choice of the step-size $12L\gamma \le \beta$. This concludes the proof.
\end{proof}

\begin{lemma}\label{lem:bound_gt_hat} 
Let each $f_i$ be $L$-smooth, $\Delta \ge \Phi^0$, $B_{\rm init}>\tau$. Assume that the following inequalities hold for the iterates generated by \algname{Byz-Clip21-SGD2M} 
\begin{enumerate}
    \item $\gamma \le \frac{1}{24L}$;
    \item $\beta\in[0,1];$
    \item $\hbeta \leq \min\left\{\frac{\sqrt{L\Delta}}{a}, \frac{\sqrt{L\Delta}}{4\sqrt{2c\deltabyz}\ha}, 1\right\};$
    \item $\|\overline{m}^{t-1}\| \le  
    \sqrt{64L\Delta} + 3(B_{\rm init}-\tau) 
    + 3b 
    + 3\hbeta a 
    + \hbeta\sqrt{2c\deltabyz}\ha
    + 3\sqrt{2c\deltabyz}\hbeta\tau T$,
    
    \item $\|\overline{g}^{t-1}\| \le  
    \sqrt{64L\Delta} + 3(B_{\rm init}-\tau) + 3b
    + 3\sqrt{2c\deltabyz}\hbeta\tau T;$

    \item $\|\nabla f_i(x^{t-1}) - v_i^{t-1}\| \le  
    \sqrt{4L\Delta}
    + \frac{3}{2}(B_{\rm init}-\tau) 
    + \frac{3}{2}b 
    + \hbeta a 
    + 4\hbeta\sqrt{2c\deltabyz}\ha
    + \frac{3}{2}\sqrt{2c\deltabyz}\hbeta\tau T$ for all $i\in\cG;$
    \item $\|v_i^t - g_i^{t-1}\| \le B_{\rm init}$ for all $i\in\cG;$
    \item $\|\Omega^t\| \le a;$
    \item $\|\Omega_i^{t-1}\| \le \ha$ for all $i\in\cG;$
    \item $\|\theta_i^t\| \le b$ for all $i\in\cG;$
    \item $\Phi^{t-1} \le 2\Delta.$
\end{enumerate}
Then we have 
\begin{align}
    \|\overline{g}^{t}\| \le  
    \sqrt{64L\Delta} + 3(B_{\rm init}-\tau) + 3b + 3\sqrt{2c\deltabyz}\hbeta\tau T.
\end{align}

\end{lemma}
\begin{proof}
    We have 
    \begin{align}
        \|\overline{g}^t\| &\overset{(i)}{=} \left\|\overline{g}^{t-1} + \frac{1}{G}\sum_{i\in\cG} \hbeta\clip_{\tau}(v_i^t-g_i^{t-1})\right\|\notag\\
        &= \left\|\hbeta\nabla f(x^t) 
        + \hbeta(\overline{v}^t - \nabla f(x^t))
        + (1-\hbeta)\overline{g}^{t-1}
        + \frac{\hbeta}{G}\sum_{i\in\cG}[\clip_{\tau}(v_i^t - g_i^{t-1}) - (v_i^t - g_i^{t-1})]\right\|\notag\\
        &\overset{(ii)}{\le} 
         \hbeta\|\nabla f(x^t)\|
        + \frac{\hbeta}{G}\sum_{i\in\cG}\|v_i^t - \nabla f_i(x^t)\|
        + (1-\hbeta)\|\overline{g}^{t-1}\|\notag\\
        &\qquad 
        +\; \frac{\hbeta}{G}\sum_{i\in\cG}\|\clip_{\tau}(v_i^t - g_i^{t-1}) - (v_i^t - g_i^{t-1})\|\notag\\
        &\overset{(iii)}{\le} \hbeta\|\nabla f(x^{t-1})\|
        + \hbeta\|\nabla f(x^t) - \nabla f(x^{t-1})\|
        + \frac{\hbeta}{G}\sum_{i\in\cG}\|(1-\beta)v_i^{t-1} + \beta\nabla f_i(x^{t},\xi^t_i) - \nabla f_i(x^t)\|\notag\\
        &\qquad +\; (1-\hbeta)\|\overline{g}^{t-1}\|
        + \frac{\hbeta}{G}\sum_{i\in\cG}\max\left\{\|v_i^t - g_i^{t-1}\| - \tau, 0\right\},
    \end{align}
    where $(i)$ follows from the update rule of $g_i^t$, $(ii)$ -- from the triangle inequality, $(iii)$ -- from the update rule of $v_i^t,$ triangle inequality, and \Cref{lem:clipping_property}. We continue the derivation of the bound as follows
    \begin{align}
        \|\overline{g}^t\| &\overset{(iv)}{\le} \hbeta\sqrt{2L(f(x^{t-1}) - f^\star)}
        + \hbeta L\gamma\|g^{t-1}\|
        + (1-\hbeta)\|\overline{g}^{t-1}\|
        + \hbeta(B_{\rm init}-\tau)\notag\\
        &\qquad +\;\frac{\hbeta}{G}\sum_{i\in\cG}\left((1-\beta)\|v_i^{t-1} - \nabla f_i(x^{t-1})\|
        + (1-\beta)\|\nabla f_i(x^{t-1}) - \nabla f_i(x^t)\|
        + \beta\|\theta^t_i\|\right),\notag\\
        &\overset{(v)}{\le} 
        \hbeta\sqrt{4L\Delta}
        + \hbeta L\gamma(2-\beta)\|g^{t-1}\|
        + (1-\hbeta)\|\overline{g}^{t-1}\|
        + \hbeta(B_{\rm init}-\tau)
        + \hbeta\beta b
        \notag\\
        &\qquad +\; \frac{\hbeta}{G}(1-\beta)\sum_{i\in\cG}\|v_i^{t-1}-\nabla f_i(x^{t-1})\|\notag\\
        &\overset{(vi)}{\le} \hbeta\sqrt{4L\Delta}
        + \hbeta L\gamma(2-\beta)\|\overline{m}^{t-1}-g^{t-1}\|
        + \hbeta L\gamma(2-\beta)\|\overline{m}^{t-1}\|
        + (1-\hbeta)\|\overline{g}^{t-1}\|
        + \hbeta(B_{\rm init}-\tau)
        \notag\\
        &\qquad +\; 
        \hbeta\beta b
        + \frac{\hbeta}{G}(1-\beta)\sum_{i\in\cG}\|v_i^{t-1}-\nabla f_i(x^{t-1})\|,
    \end{align}
    where $(iv)$ follows from $L$-smoothness, the update rule of $x^t$, assumption 7 of the lemma, and triangle inequality, $(v)$ -- from assumptions 10 and 11 of the lemma, $L$-smoothness, the update rule of $x^t$, $(vi)$ -- from triangle inequality. Using \eqref{eq:agg_inequality} we continue
    \begin{align}
        \|\overline{g}^t\| 
        &\overset{(vii)}{\le} \hbeta\sqrt{4L\Delta}
        + 2L\gamma\hbeta\left(2\hbeta\tau T\sqrt{2c\deltabyz} + 2\hbeta\sqrt{\frac{2c\deltabyz}{G}\sum_{i\in\cG}\|\Omega_i^{t-1}\|^2}\right)
        + 2L\gamma\hbeta\|\overline{m}^{t-1}\|
        \notag\\
        &\qquad +\; (1-\hbeta)\|\overline{g}^{t-1}\|
        + \hbeta(B_{\rm init}-\tau)
        + \hbeta\beta b
        + \frac{\hbeta}{G}(1-\beta)\sum_{i\in\cG}\|v_i^{t-1}-\nabla f_i(x^{t-1})\|,
    \end{align}
    where $(vii)$ follows from \eqref{eq:agg_inequality} and inequality $\sqrt{a+b}\le \sqrt{a}+\sqrt{b}$ for any positive $a,b$. Then, we get
    \begin{align}
        \|\overline{g}^t\| 
         &\overset{(viii)}{\le} \hbeta\sqrt{4L\Delta}
        + 4L\gamma\hbeta^2\tau T\sqrt{2c\deltabyz}
        + 4L\gamma\hbeta^2\sqrt{2c\deltabyz}\ha
        + \hbeta(B_{\rm init}-\tau)
        + \hbeta\beta b\notag\\
        &\qquad +\; 4L\gamma\hbeta\left( 
        \sqrt{64L\Delta} + 3(B_{\rm init}-\tau) + 3b + 3\hbeta a + \hbeta\sqrt{2c\deltabyz}\ha + 3\sqrt{2c\deltabyz}\hbeta\tau T\right)
        \notag\\
        &\qquad +\; (1-\hbeta)\left( 
        \sqrt{64L\Delta} + 3(B_{\rm init}-\tau) + 3b + 3\sqrt{2c\deltabyz}\hbeta\tau T\right)\notag\\
        &\qquad +\; 
        \hbeta(1-\beta)\left( 
        \sqrt{4L\Delta} + \frac{3}{2}(B_{\rm init}-\tau) + \frac{3}{2}b + \hbeta a + 4\hbeta\sqrt{2c\deltabyz}\ha + \frac{3}{2}\sqrt{2c\deltabyz}\hbeta\tau T\right),
    \end{align}
    where $(viii)$ follows from the assumptions 4-6 and 9 of the lemma. Regrouping the terms, we derive
    \begin{align}
        \|\overline{g}^t\| 
        &\le \sqrt{L\Delta}[2\hbeta + 32L\gamma\hbeta + 8(1-\hbeta) + 2\hbeta]\notag\\
        &\qquad +\; (B_{\rm init}-\tau)[\hbeta + 12L\gamma\hbeta + 3(1-\hbeta) + \nicefrac{3\hbeta(1-\beta)}{2}]\notag\\
        &\qquad +\; b[\hbeta\beta + 12L\gamma\hbeta + 3(1-\hbeta) + \nicefrac{3\hbeta(1-\beta)}{2}]
        + \hbeta a[12L\gamma\hbeta + (1-\beta)\hbeta]\notag\\
        &\qquad +\; \hbeta\sqrt{2c\deltabyz}\ha[4L\gamma\hbeta + 4L\gamma\hbeta + 4\hbeta(1-\hbeta)]\notag\\
        &\qquad +\; \sqrt{2c\deltabyz}\hbeta\tau T[4L\gamma\hbeta + 12L\gamma\hbeta + 3(1-\hbeta) + \nicefrac{3\hbeta(1-\beta)}{2}],
    \end{align} 
    For the second term, we have 
    \begin{align}
        \hbeta + 12L\gamma\hbeta + 3(1-\hbeta) + \frac{3}{2}\hbeta(1-\beta) \le 3 \Leftarrow \hbeta + 12L\gamma\hbeta + \frac{3}{2}\hbeta \le 3\hbeta \Leftarrow 24L\gamma \le 1,
    \end{align}
    where the last inequality holds by the choice of the step-size $24L\gamma \le 1$.
    For the third term, we have
    \begin{align}
        \hbeta\beta + 12L\gamma\hbeta + 3(1-\hbeta) + \frac{3}{2}\hbeta(1-\beta) \le 3 \Leftarrow \beta\hbeta + 12L\gamma\hbeta + \frac{3}{2}\hbeta \le 3\hbeta \Leftarrow 24L\gamma \le 1,
    \end{align}
    where the last inequality holds by the choice of the step-size $24L\gamma \le 1$. For the fourth term, we have 
    \begin{align}\label{eq:knokenqoedn}
        \hbeta a[12L\gamma\hbeta + (1-\beta)\hbeta] \le \sqrt{L\Delta}[12L\gamma\hbeta + (1-\beta)\hbeta],
    \end{align}
    where we use $\hbeta\le \frac{\sqrt{L\Delta}}{a}$. Next, for the fifth term, we have 
    \begin{align}\label{eq:fnweijknqdqw}
    	&\; \hbeta\sqrt{2c\deltabyz}\ha[4L\gamma\hbeta + 4L\gamma\hbeta + 4\hbeta(1-\hbeta)]
        \le \sqrt{L\Delta}[2L\gamma\hbeta + \hbeta(1-\beta)],
    \end{align}
    where we use $\hbeta \le \frac{\sqrt{L\Delta}}{4\sqrt{2c\deltabyz}\ha}$. Therefore, combining the two inequalities  \eqref{eq:knokenqoedn} and \eqref{eq:fnweijknqdqw} with the first term, we obtain for the coefficient next to $\sqrt{L\Delta}$
    \begin{align}
        &\; 12L\gamma\hbeta + \hbeta(1-\beta) + 2L\gamma\hbeta + \hbeta(1-\beta) 
        + 2\hbeta + 32L\gamma\hbeta + 8(1-\hbeta) + 2\hbeta \le 8 \notag\\
        \Leftarrow&\; 46L\gamma\hbeta + 2\hbeta(1-\beta)
        + 4\hbeta \le 8\hbeta\notag\\
        \Leftarrow&\; 23L\gamma \le 1,
    \end{align}
    where the last inequality holds by the choice of the step-size $24L\gamma \le 1$. For the sixth coefficient, we have 
    \begin{align}
        &\; 4L\gamma\hbeta + 12L\gamma\hbeta + 3(1-\hbeta) + \frac{3\hbeta(1-\beta)}{2} \le 3 
        \Leftarrow 16L\gamma\hbeta + \frac{3\hbeta(1-\beta)}{2} \le 3\hbeta,\notag\\
        \Leftarrow &\; \frac{32}{3}L\gamma \le 1,
    \end{align}
    where the last inequality holds by the choice of the step-size $24L\gamma \le 1$. This concludes the proof.
    
\end{proof}

\begin{lemma}\label{lem:control_clip}
Let each $f_i$ be $L$-smooth, $\Delta \ge \Phi^0$, $B_{\rm init} > \tau$, and $i\in\cI_t \eqdef \{i\in\cG \mid \|v_i^t - g_i^{t-1}\| > \tau\}.$ Assume that the following inequalities hold for the iterates generated by \algname{Byz-Clip21-SGD2M}:
 \begin{enumerate}
    \item $\gamma \le \frac{1}{24L}$;
    \item $\beta\le \min\left\{\frac{\hbeta\tau}{54\sqrt{L\Delta}}, 1\right\};$
    \item $\beta\le \min\left\{\frac{2\hbeta\tau}{39(B_{\rm init}-\tau)}, 1\right\};$
    \item $\beta\le \min\left\{\frac{2\hbeta\tau}{69b}, 1\right\};$
    \item $\beta \le \min\left\{\frac{2}{41\sqrt{2c\deltabyz}T}, 1\right\}$;
    \item $\hbeta \in \min\left\{\frac{\sqrt{L\Delta}}{a}, \frac{\sqrt{L\Delta}}{4\sqrt{2c\deltabyz}\ha}, \frac{1}{24}\right\};$
    \item $\|\overline{m}^{t-1}\| \le \sqrt{64L\Delta} + 3(B_{\rm init}-\tau) 
    + 3b 
    + 3\hbeta a 
    + \hbeta\sqrt{2c\deltabyz}\ha
    + 3\sqrt{2c\deltabyz}\hbeta\tau T$,
    
    \item $\|\overline{g}^{t-1}\| \le \sqrt{64L\Delta} + 3(B_{\rm init}-\tau) + 3b
    + 3\sqrt{2c\deltabyz}\hbeta\tau T;$

    \item $\|\nabla f_i(x^{t-1}) - v_i^{t-1}\| \le  \sqrt{4L\Delta}
    + \frac{3}{2}(B_{\rm init}-\tau) 
    + \frac{3}{2}b 
    + \hbeta a 
    + 4\hbeta\sqrt{2c\deltabyz}\ha
    + \frac{3}{2}\sqrt{2c\deltabyz}\hbeta\tau T$ for all $i\in\cG;$
    \item $\|\Omega_i^{t}\| \le \ha$ for all $i\in\cG;$
    \item $\|\theta_i^{t}\| \le b$ and $\|\theta_i^{t+1}\| \le b$ for all $i\in\cG.$

\end{enumerate}
Then we have 
\begin{align}
    \|v_i^{t+1} - g_i^t\| \le \|v_i^t - g_i^{t-1}\| - \frac{\tau\hbeta}{2}.
\end{align}

\end{lemma}

\begin{proof}
    Since $i\in\cI_t,$ then $\|v_i^t-g_i^{t-1}\| > \tau$ and from \Cref{lem:bound_vi_t_gi_t} we have
    \begin{align}\label{eq:njlnwdqowjdnwq}
        \|v_i^{t+1} - g_i^t\| &\le  (1-\hbeta)\|v_i^t - g_i^{t-1}\|
        + \hbeta \|v_i^t - g_i^{t-1}\| - \hbeta\tau
        + \beta L\gamma\|\overline{m}^t\|
        + \beta\|\nabla f_i(x^t) - v_i^t\|
        \notag\\
        &\qquad +\; 2\hbeta\beta L\gamma\tau T\sqrt{2c\deltabyz}
        + 2\beta\hbeta L \gamma\sqrt{\frac{2c\deltabyz}{G}\sum_{i\in\cG}\|\Omega_i^t\|^2}
        + \beta\|\theta^{t+1}_i\|.
    \end{align}
    Using the assumptions of the lemma and Lemmas~\ref{lem:bound_mt_bar} and \ref{lem:bound_nabla_fi_xt_vi_t}, we obtain from \eqref{eq:njlnwdqowjdnwq}
   \begin{align}
        \|v_i^{t+1} - g_i^t\| &\le \|v_i^t-g_i^{t-1}\|
        - \hbeta\tau
        + \beta L\gamma\left(\sqrt{64L\Delta} + 3(B_{\rm init}-\tau) + 3b + 3\hbeta a + \hbeta\sqrt{2c\deltabyz}\ha + 3\sqrt{2c\deltabyz}\hbeta\tau T\right)\notag\\
        &\qquad +\; \beta\left(\sqrt{4L\Delta} + \frac{3}{2}(B_{\rm init}-\tau) + \frac{3}{2}b + \hbeta a + 4\hbeta\sqrt{2c\deltabyz}\ha + \frac{3}{2}\sqrt{2c\deltabyz}\hbeta\tau T\right)
        \notag\\
        &\qquad +\; 2\hbeta\beta L\gamma\tau T\sqrt{2c\deltabyz}
        + 2\beta\hbeta L\gamma\sqrt{2c\deltabyz}\ha
        + \beta b.
    \end{align}
    Regrouping the terms, we derive
    \begin{align}
        \|v_i^{t+1} - g_i^t\| &\le \|v_i^t - g_i^{t-1}\| - \hbeta\tau 
        + \sqrt{L\Delta}[8L\gamma\beta + 2\beta]
        + (B_{\rm init}-\tau)[3L\gamma\beta + \nicefrac{3\beta}{2}]\notag\\
        &\qquad +\; 
        b[3L\gamma\beta + \nicefrac{3\beta}{2} + \beta] 
        + \hbeta a[3L\gamma\beta + \beta]
        +  \hbeta\sqrt{2c\deltabyz }\ha[3\beta L\gamma + 4\beta]
        \notag\\
        &\qquad +\;
        \hbeta\tau T\sqrt{2c\deltabyz}[3\beta L\gamma + {\color{black}\nicefrac{3\beta}{2}} + 2\beta L\gamma]\notag\\
        &= \|v_i^t - g_i^{t-1}\| - \hbeta\tau 
        + \sqrt{L\Delta}[8L\gamma\beta + 2\beta]
        + (B_{\rm init}-\tau)[3L\gamma\beta + \nicefrac{3\beta}{2}]\notag\\
        &\qquad +\; 
        b[3L\gamma\beta + \nicefrac{5\beta}{2}] 
        + \hbeta a[3L\gamma\beta + \beta]
        + \hbeta\sqrt{2c\deltabyz }\ha[3\beta L\gamma + 4\beta]
        \notag\\
        &\qquad +\; 
        \hbeta\tau T\sqrt{2c\deltabyz}[5\beta L\gamma + {\color{black}\nicefrac{3\beta}{2}}]\label{eq:pdnqowndqowjndq}
    \end{align}
    First, note that using restrictions $\hbeta a \le \sqrt{L\Delta}$, we obtain
    \begin{align}\label{eq:kfonwejne}
        \hbeta a[3L\gamma\beta + \beta] \le \sqrt{L\Delta}[3L\gamma\beta + \beta].
    \end{align}
    Next, we also have 
    \begin{align}\label{eq:noqwnqodnowqd}
        \hbeta\sqrt{2c\deltabyz}\ha[3\beta L\gamma + 4\beta] \le \frac{\sqrt{L\Delta}}{4}[3\beta L\gamma + 4\beta] = \sqrt{L\Delta}[\nicefrac{3}{4}\beta L\gamma + \beta].
    \end{align}
    
    Combining \eqref{eq:kfonwejne} and \eqref{eq:noqwnqodnowqd} with other terms involving $\sqrt{L\Delta}$ and using $24L\gamma\le 1$, we have 
    \begin{align}
    	\sqrt{L\Delta}[3L\gamma\beta + \beta +\nicefrac{3}{4}L\gamma\beta + \beta 
        + 8L\gamma\beta + 2\beta] \le \sqrt{L\Delta}[12L\gamma\beta + 4\beta] \le \sqrt{L\Delta}\frac{9\beta}{2} \le \frac{\hbeta\tau}{12},
    \end{align}
    where we used $\beta \le \frac{\hbeta\tau}{54\sqrt{L\Delta}}$ and $24L\gamma \le 1$. Again, since $24L\gamma \le 1$, we have 
    \begin{align}
    	[3L\gamma\beta + \nicefrac{3\beta}{2}](B_{\rm init}-\tau) \le \frac{13\beta}{8}(B_{\rm init}-\tau) \le \frac{\hbeta\tau}{12},
    \end{align}
    where we used $\beta \le \frac{2\hbeta\tau}{39(B_{\rm init}-\tau)}$. Since $24L\gamma \le 1$, we have 
    \begin{align}
    	[3L\gamma\beta + \nicefrac{3\beta}{2} + \beta]b \le \frac{23\beta}{8}b \le \frac{\hbeta\tau}{12},
    \end{align}
    where we use $\beta \le \frac{2\hbeta\tau}{69b}$. Since $24L\gamma \le 1$, we have 
    \begin{align}
        [5\beta L\gamma + \nicefrac{3\beta}{2}]\hbeta\tau T\sqrt{2c\deltabyz} \le \frac{41}{24}\beta\hbeta\tau T\sqrt{2c\deltabyz} \le \frac{\hbeta\tau}{12},
    \end{align}
    where we use $\beta \le \frac{2}{41\sqrt{2c\deltabyz}T}$. Combining all the bounds, we obtain from \eqref{eq:pdnqowndqowjndq} that 
     \begin{align}
        \|v_i^{t+1} - g_i^t\| 
        &\le \|v_i^t - g_i^{t-1}\| - \hbeta\tau 
        + \sqrt{L\Delta}[8L\gamma\beta + 2\beta]
        + (B_{\rm init}-\tau)[3L\gamma\beta + \nicefrac{3\beta}{2}]\notag\\
        &\qquad +\; 
        b[3L\gamma\beta + \nicefrac{5\beta}{2}] 
        + \hbeta a[3L\gamma\beta + \beta]
        + \hbeta\sqrt{2c\deltabyz }\ha[3\beta L\gamma + 4\beta]
        \notag\\
        &\qquad +\; 
        \hbeta\tau T\sqrt{2c\deltabyz}[5\beta L\gamma + \nicefrac{3\beta}{2}]\notag\\
        &\le \|v_i^t - g_i^{t-1}\| - \hbeta\tau 
        + 4\cdot \frac{\tau\hbeta}{12} + \frac{\tau\hbeta}{6}\notag\\
        &= \|v_i^t - g_i^{t-1}\| - \frac{\hbeta\tau}{2}.
    \end{align}
    
\end{proof}

\begin{lemma}[Lemma 17 from \citet{islamov2025double}]\label{lem:descent_Pt_tilde} Let $\|\theta^{t+1}_i\| \le b$ for all $i\in\cG.$ Let each $f_i$ be $L$-smooth. Then, for the iterates generated by \algname{Byz-Clip21-SGD2M} the quantity $\wtilde{P}^{t} \eqdef \frac{1}{G}\sum_{i\in\cG}\|v_i^t - \nabla f_i(x^t)\|^2$ decreases as
\begin{align}
    \wtilde{P}^{t+1} \le (1-\beta)\wtilde{P}^t 
    + \frac{3L^2}{\beta}R^t
    + \beta^2b^2
    + \frac{2}{G}\beta(1-\beta)\sum_{i\in\cG}\<v_i^t - \nabla f_i(x^{t+1}), \theta^{t+1}_i>,
\end{align}
where $R^t \eqdef \|x^{t+1} - x^t\|^2$ and $\theta^t_i = \nabla f_i(x^t,\xi^t_i) - \nabla f_i(x^t).$
\end{lemma}

\begin{lemma}[Lemma 18 from \citet{islamov2025double}]\label{lem:descent_Pt} Let $\|\theta^{t+1}\| \le \frac{\hc}{\sqrt{G}}$. Let each $f_i$ be $L$-smooth. Then, for the iterates generated by \algname{Byz-Clip21-SGD2M} the quantity $\wtilde{P}^{t} \eqdef \|\overline{v}^t - \nabla f(x^t)\|^2$ decreases as
\begin{align}
    \wtilde{P}^{t+1} \le (1-\beta)\wtilde{P}^t 
    + \frac{3L^2}{\beta}R^t
    + \beta^2\frac{c^2}{G}
    + 2\beta(1-\beta)\<\overline{v}^t - \nabla f(x^{t+1}), \theta^{t+1}>,
\end{align}
where $R^t \eqdef \|x^{t+1} - x^t\|^2$ and $\theta^t = \frac{1}{G}\sum_{i\in\cG}\nabla f_i(x^t,\xi^t_i) - \nabla f_i(x^t).$
\end{lemma}

\begin{lemma}[Lemma 19 from \citet{islamov2025double}]\label{lem:descent_Vt_tilde} Let $\|\theta^t_i\| \le b$ for all $i\in\cG,$ each $f_i$ be $L$-smooth, and $\|v_i^t-g_i^{t-1}\|\le B_{\rm init}$ for all $i\in\cG$ and some $B_{\rm init} > \tau,$ and $\hbeta \le \frac{1}{2\eta}$. Then for the iterates generated by \algname{Byz-Clip21-SGD2M} we have 
\begin{align}
\begin{aligned}
    \|g_i^t - v_i^t\|^2 &\le (1-\hbeta\eta)\|g_i^{t-1} - v_i^{t-1}\|^2 
    + \frac{4\beta^2}{\hbeta\eta}\|v_i^{t-1} - \nabla f_i(x^{t-1})\|^2
    + \frac{4\beta^2L^2}{\hbeta\eta}R^{t-1}
    + \beta^2b^2\\
    &\qquad +\;2(1-\hbeta\eta)^2\beta\<(g_i^{t-1} - v_i^{t-1}) + \beta(v_i^{t-1} - \nabla f_i(x^{t-1})), \theta^t_i>\\
    &\qquad +\; 2(1-\hbeta\eta)^2\beta^2\<\nabla f_i(x^{t-1}) - \nabla f_i(x^t), \theta^t_i>,
\end{aligned}
\end{align}
where $R^t \eqdef \|x^{t+1} - x^t\|^2$ and $\eta \eqdef \frac{\tau}{B_{\rm init}} \in(0,1).$ Moreover, averaging the inequalities over $i\in\cG$, we get
\begin{align}
        \wtilde{V}^t &\le (1-\hbeta\eta)\wtilde{V}^{t-1}
        + \frac{4\beta^2}{\hbeta\eta}\wtilde{P}^{t-1}
        + \frac{4\beta^2L^2}{\hbeta\eta}R^{t-1}
        + \beta^2b^2\\
        &\qquad +\; \frac{2}{G}(1-\hbeta\eta)^2\beta\sum_{i\in\cG}\<(g_i^{t-1} - v_i^{t-1}) + \beta(v_i^{t-1} - \nabla f_i(x^{t-1})) + \beta(\nabla f_i(x^{t-1}) - \nabla f_i(x^t)), \theta^t_i>,\notag
\end{align}
where $\wtilde{V}^t \eqdef \frac{1}{G}\sum_{i\in\cG} \|g_i^t - v_i^t\|^2$ and $\wtilde{P}^t \eqdef \frac{1}{G}\sum_{i\in\cG}\|v_i^t - \nabla f_i(x^t)\|^2.$
\end{lemma}

Let us define constants that we will use in the proof of the main theorem. 
\begin{align}\label{eq:constants}
\begin{aligned}
    a &\eqdef \left(\sqrt{2} + 2\sqrt{3\log\frac{8(T+1)}{\alpha}}\right)\sqrt{d}\sigma_{\omega}\sqrt{\frac{T}{G}},\\
    \ha &\eqdef \left(\sqrt{2} + 2\sqrt{3\log\frac{8G(T+1)}{\alpha}}\right)\sqrt{d}\sigma_{\omega}\sqrt{T},\\
    b^2 &\eqdef 2\sigma^2\log\left(\frac{16(T+1)G}{\alpha}\right),\\
    \hc^2 &\eqdef \left(\sqrt{2} + 2\sqrt{3\log\frac{8(T+1)}{\alpha}}\right)^2\sigma^2,
\end{aligned}
\end{align}

\section{Proof of Theorem~\ref{th:main_theorem}}\label{apx:main_theorem}


\begin{theorem}[Full statement of \Cref{th:main_theorem}]\label{th:main_theorem_full} Let Assumptions \ref{asmp:smoothness} and \ref{asmp:stoch_grad}, and define $B_{\rm init}\eqdef \max\{3\tau, \max_i\{\|\nabla f_i(x^0)\|\}+b\}$. Let the failure probability $\alpha$ be such that $\alpha\in(0,1)$, and constants $a,\ha, b,$ and $c$ be defined as in (\ref{eq:constants}), and $\Delta \ge \Phi^0$ for $\Phi^0$ defined in \eqref{eq:lyapunov_function}. Consider the run of \algname{Byz-Clip21-SGD2M} (\Cref{alg:byz_clip21_sgd2m}) for $T$ iterations with DP noise variance $\sigma_{\omega}$.
    Assume the following inequalities hold
    \begin{enumerate}
        \item {\bf step-size restrictions:} 
        \begin{enumerate}
       \item 
       $\gamma \le 
       \frac{1}{24L}
       $;
        \item $12L\gamma = \beta$;
        \item \begin{equation}\label{eq:wnojenoqdqw}
            \frac{5}{6} - \frac{32L^2\beta^2}{\hat\beta^2\eta^2}\gamma^2
    - \frac{96L^2}{\hat\beta^2\eta^2}\gamma^2 \ge 0.
        \end{equation}
        \end{enumerate}
        \item {\bf momentum restrictions:}
        \begin{enumerate}
        \item $\beta\le \min\left\{\frac{\hbeta\tau}{54\sqrt{L\Delta}}, 1\right\};$
        \item $\beta\le \min\left\{\frac{2\hbeta\tau}{39(B_{\rm init}-\tau)}, 1\right\};$
        \item $\beta\le \min\left\{\frac{2\hbeta\tau}{69b}, 1\right\};$
        \item $\beta \le \min\left\{\frac{2}{41\sqrt{2c\deltabyz}T}, 1\right\}$;
        \item 
        $\hbeta \in \min\left\{\frac{\sqrt{L\Delta}}{a}, \frac{\sqrt{L\Delta}}{4\sqrt{2c\deltabyz}\ha}, 1\right\};$
        \item and momentum restrictions defined in (\ref{eq:step-size_bound_1}), (\ref{eq:step-size_bound_2}), (\ref{eq:step-size_bound_3}), (\ref{eq:step-size_bound_4}), (\ref{eq:step-size_bound_5}), (\ref{eq:step-size_bound_6}), (\ref{eq:step-size_bound_7}), and (\ref{eq:step-size_bound_8});
        \end{enumerate}
    \end{enumerate}
    Then, with probability $1-\alpha$, we bound $\frac{1}{T}\sum_{t=0}^{T-1}\|\nabla f(x^t)\|^2$ with
   \begin{align*}
        \sqrt{L\Delta}(\sqrt{L\Delta} + B_{\rm init} + \sigma) \wtilde{\cO}\left(\frac{\sqrt{d}\sigma_{\omega}}{\sqrt{GT}\tau} 
        + \frac{d^{1/3}\sigma_{\omega}^{2/3}}{(TG)^{1/3}\tau^{2/3}}
        + \frac{\sqrt{c\deltabyz} \sqrt{d}\sigma_{\omega}}{\sqrt{T}\tau} 
        + \sqrt{c\deltabyz}
        \right),
    \end{align*}
    where $\wtilde{\cO}$ hides constant and logarithmic factors and higher order terms decreasing in $T$. Let additionally \Cref{asmp:bounded_heterogeneity} hold, then 
    \[
    \sqrt{L\Delta}(\sqrt{L\Delta} + B_{\rm init} + \sigma) = \wtilde{\cO}\left((1+B_{\zeta})LF^0 + \zeta^2 + \sigma ((1+\sqrt{B_{\zeta}})\sqrt{LF^0} + \zeta)\right).
    \]
\end{theorem}

\begin{proof} 
    For convenience, we define $\nabla f_i(x^{-1},\xi^{-1}_i) = v_i^{-1} = g_i^{-1} = 0, \Phi^{-1} = \Phi^0$. Next, let us define an event $E^t$ for each $t\in\{0,\dots,T\}$ such that the following inequalities hold for all $k\in\{0,\dots,t\}$
    \begin{enumerate}
        \item $\|v_i^k - g_i^{k-1}\| \le B_{\rm init}$ for $i\in \cI_{k};$
        \item $\|\overline{m}^{t}\| \le \sqrt{64L\Delta} + 3(B_{\rm init}-\tau) 
        + 3b 
        + 3\hbeta a 
        + \hbeta\sqrt{2c\deltabyz}\ha
        + 3\sqrt{2c\deltabyz}\hbeta\tau T$,
        
        \item $\|\overline{g}^{t}\| \le \sqrt{64L\Delta} + 3(B_{\rm init}-\tau) + 3b
        + 3\sqrt{2c\deltabyz}\hbeta\tau T;$

        \item $\|\nabla f_i(x^{t}) - v_i^{t}\| \le \sqrt{4L\Delta}
        + \frac{3}{2}(B_{\rm init}-\tau) 
        + \frac{3}{2}b 
        + \hbeta a 
        + 4\hbeta\sqrt{2c\deltabyz}\ha
        + \frac{3}{2}\sqrt{2c\deltabyz}\hbeta\tau T$ for all $i\in\cG;$
        \item $\|\theta^k_i\|\le b$ for all $i\in\cG$ and $\|\theta^k\|\le \frac{\hc}{\sqrt{G}};$
        \item $\left\|\frac{1}{G}\sum_{l=1}^{k+1}\sum_{i\in\cG}\omega_i^l\right\| \le a$ and $\left\|\sum_{l=1}^{k+1}\omega_i^l\right\|\le \ha$ for all $i\in\cG$;
        \item $\Phi^k \le 2\Delta$;
        \item \begin{align*}
        \frac{7}{8}\Delta & \ge \frac{4\gamma\beta}{n\hat{\beta}\eta}(1-\eta)^2\sum_{l=0}^{k-1}\sum_{i\in\cG}\<(g_i^{l} - v_i^{l}) + \beta(v_i^{l} - \nabla f_i(x^{l})) + \beta(\nabla f_i(x^{l})-\nabla f_i(x^{l+1})), \theta^t_i>\\
        &\quad + \frac{16\gamma\beta^2}{n\hat{\beta}^2\eta^2}(1-\beta)\sum_{l=0}^{k-1}\sum_{i\in\cG}\<v_i^l - \nabla f_i(x^{l}), \theta^{l+1}_i>
        + 4\gamma(1-\beta)\sum_{l=0}^{k-1}\<v^l - \nabla f(x^{l}), \theta^{l+1}>\\
        &\quad +\frac{15\gamma\beta^2}{n\hat{\beta}^2\eta^2}(1-\beta)\sum_{l=0}^{k-1}\sum_{i\in\cG}\<\nabla f_i(x^l) - \nabla f_i(x^{l+1}), \theta^{l+1}_i>\\
        & \quad + 4\gamma(1-\beta)\sum_{l=0}^{k-1}\<\nabla f(x^l) - \nabla f(x^{l+1}), \theta^{l+1}>.
        \end{align*}
    \end{enumerate}
    Then, we will derive the result by induction, i.e., using the induction w.r.t.\ $t$, we will show that $\Prob(E^t) \ge 1-\frac{\alpha(t+1)}{T+1}$ for all $t\in\{0,\dots, T-1\}$.
    
    Before moving on to the proof's induction part, we need to establish several useful bounds. Denote the events $\Theta^t_i, \Theta^t$ and $N^{t+1}$ as 
    \begin{align}
        &\Theta^t_i \eqdef \{\|\theta^t_i\| \ge b \}, \quad \Theta^t \eqdef \left\{ \|\theta^t\|\ge \frac{\hc}{\sqrt{G}} \right\}, \quad N^{t+1} \eqdef \left\{ \left\|\frac{1}{G}\sum_{l=1}^t\sum_{i\in\cG}\omega_i^l\right\| \ge a \right\},\quad \text{and}\notag\\
        &N_i^{t+1} = \left\{\left\|\sum_{l=1}^t\omega_i^l\right\| \ge \ha \right\}
    \end{align}
    respectively. From \Cref{asmp:stoch_grad} we have 
    \[\Prob(\Theta^{t}_i) \le 2\exp\left(-\frac{b^2}{2\sigma^2}\right) = \frac{\alpha}{8(T+1)G}\] 
    where the last equality is by definition of $b^2$. Therefore, $\Prob(\overline{\Theta}^t_i) \ge 1 - \frac{\alpha }{8(T+1)G}.$ 
    Besides, notice that the constant $\hc$ in (\ref{eq:constants}) can be viewed as 
    \[
        \hc^2 = (\sqrt{2}+2b_3)\sigma^2 \quad\text{where} \quad b_3^2 = 3\log\frac{8(T+1)}{\alpha}.
    \]
    Now, we can use \Cref{lem:concentration_lemma} to bound $\Prob(\Theta^t).$ Since all $\theta^t_i$ are independent $\sigma$-sub-Gaussian random vectors, then we have
    \[
    \Prob\left(\left\|\sum_{i\in\cG}\theta^t_i\right\|\ge \hc\sqrt{G}\right) = \Prob\left(\|\theta^t\| \ge \frac{\hc}{\sqrt{G}}\right) \le \exp(-b_3^2/3) = \frac{\alpha}{8(T+1)}.
    \]
    We also use \Cref{lem:concentration_lemma} to bound $\Prob(N^t)$ and $\Prob(N_i^t)$. Indeed, since all $\omega_i^l$ are independent Gaussian random vectors, then we have 
    \[\Prob\left(\left\|\sum_{l=1}^t\sum_{i\in\cG}\omega_i^l\right\| \ge (\sqrt{2} + 2b_2)\sqrt{\sum_{l=1}^t\sum_{i\in\cG}\sigma_{\omega}^2d}\right) \le \exp(-\nicefrac{b_2^2}{3}) = \frac{\alpha}{8(T+1)}.
    \]
    with $b_2^2 = 3\log\left(\frac{8(T+1)}{\alpha}\right).$ 
    This implies that 
    \[
    \Prob\left(\left\|\frac{1}{G}\sum_{l=1}^t\sum_{i\in\cG}\omega_i^l\right\| \ge a\right) \le \frac{\alpha}{8(T+1)}
    \]
    due to the choice of $a$ from (\ref{eq:constants}):
    \[a = (\sqrt{2}+2b_2)\sigma_{\omega}\sqrt{d}\sqrt{\frac{T}{G}}, \quad \text{where} \quad b_2^2 = 3\log\frac{8(T+1)}{\alpha}.
    \]
    Note that with this choice of $a$ we have that the above is true for any $t\in\{1,\ldots, T\}$, i.e., $\Prob(N^t) \ge 1-\frac{\alpha}{8(T+1)}$ for all $t\in\{1,\ldots, T\}.$ Similarly, we derive 
    \begin{align*}
        \Prob\left(\left\|\sum_{l=1}^{t}\omega_i^l\right\| \ge (\sqrt{2}+2b_4)\sqrt{\sum_{l=1}^t\sigma^2_{\omega}d}\right) \le \exp(-\nicefrac{b_4^2}{3})\le\frac{\alpha}{8G(T+1)}
    \end{align*}
    with $b_4^2 = 3\log\left(\frac{8G(T+1)}{\alpha}\right).$ Again, this implies that 
    \[
    \Prob\left(\left\|\sum_{l=1}^t\omega_i^l\right\| \ge \ha \right) \le \frac{\alpha}{8G(T+1)}
    \]
    due to the choice of $\ha$ from \eqref{eq:constants}
    \[
    \ha = (\sqrt{2}+2b_4)\sigma_{\omega}\sqrt{T}, \quad \text{where} \quad b_4^2 = 3\log\left(\frac{8G(T+1)}{\alpha}\right)
    \]
    
    Now, we are ready to prove that $\Prob(E^t) \ge 1-\frac{\alpha(t+1)}{T+1}$ for all $t\in\{0,\dots, T-1\}.$ First, we show that the base of induction holds.

    \paragraph{Base of induction.}

    \begin{enumerate}
        \item $\|v_i^0 - g_i^{-1}\| = \|v_i^0\| = \beta\|\nabla f_i(x^0,\xi^0_i)\| = \beta\|\theta^0_i\| + \beta \|\nabla f_i(x^0)\| \le \frac{1}{2}b + \frac{1}{2}B_{\rm init}\le \frac{1}{2}B_{\rm init} + \frac{1}{2}B_{\rm init} = B_{\rm init}$ holds with probability $1-\frac{\alpha}{8(T+1)}$. Indeed, we have 
        \[
        \Prob(\Theta^0_i) \le 2\exp\left(-\frac{b^2}{2\sigma^2}\right) = \frac{\alpha}{8(T+1)G}.
        \]
        Therefore, we have 
        \[
        \Prob\left(\cap_{i\in\cG}\overline{\Theta}^0_i\right) = 1 - \Prob\left(\cup_{i\in\cG} \Theta^0_i\right) \ge 1 - \sum_{i\in\cG}\Prob(\Theta^0_i) = 1-G\frac{\alpha}{8(T+1)G} = 1-\frac{\alpha}{8(T+1)}.
        \]
        Moreover, we have 
        \[
        \Prob(\Theta^0) \le \frac{\alpha}{8(T+1)}.
        \]
        
        This means that the probability of the event that each $\left\|\frac{1}{G}\sum_{l=1}^t\sum_{i\in\cG}\omega_i^l\right\|\le a$, $\|\sum_{l=1}^t\omega_i^l\| \le \ha$ $\|\theta^0_i\|\le b$, and  $\|\theta^0\|\le \frac{c}{\sqrt{G}},$ and is at least 
        $$1-\frac{\alpha}{8(T+1)} - n\frac{\alpha}{8G(T+1)} - \frac{\alpha}{8(T+1)} - n\frac{\alpha}{8G(T+1)}= 1 - \frac{\alpha}{2(T+1)}.$$
        \item We have already shown that
        \[
        \Prob\left(\left\|\frac{1}{G}\sum_{i\in\cG}\omega_i^1\right\| \ge a\right) \le \frac{\alpha}{8(T+1)},
        \]
       implying that $\left\|\frac{1}{G}\sum_{i\in\cG}\omega_i^1\right\| \le a$ with probability at least $1-\frac{\alpha}{8(T+1)}.$ Similarly, we have shown that $\|\sum_{l=1}^1\omega_i^l\| \le \ha$ with probability $1-\frac{\alpha}{8G(T+1)}.$ Therefore,

        \item $\overline{g}^0 = \frac{1}{G}\sum_{i\in\cG} (g_i^{-1} + \hat{\beta}\clip_\tau(v_i^0 - g_i^{-1}) = \frac{1}{G}\sum_{i\in\cG}\hat{\beta}\clip_\tau(\beta\nabla f_i(x^0,\xi^0_i)).$ Therefore, we have
        \begin{align*}
            \|\overline{g}^0\| &= \left\|\frac{1}{G}\sum_{i\in\cG}\hat{\beta}\beta\nabla f_i(x^0) + \hat{\beta}\beta\theta^0_i + (\hat{\beta}\clip_\tau(\beta\nabla f_i(x^0,\xi^0_i)) - \hat{\beta}\beta\nabla f_i(x^0,\xi^0_i))\right\|\\
            &\le \hat{\beta}\beta\|\nabla f(x^0)\|
            + \frac{\hat{\beta}\beta}{G}\sum_{i\in\cG}\|\theta^0_i\|
            + \frac{1}{G}\sum_{i\in\cG}\max\left\{0, \beta\|\nabla f_i(x^0,\xi^0_i)\| - \tau\right\}\\
            &\le \hat{\beta}\beta\sqrt{2L(f(x^0)-f(x^\star))}
            + \frac{\hat{\beta}\beta}{G}\sum_{i\in\cG}\|\theta^0_i\| 
            + \frac{\hat{\beta}}{G}\sum_{i\in\cG}\max\left\{0, \beta\|\nabla f_i(x^0)\| + \beta\|\theta^0_i\|- \tau\right\}.
        \end{align*}
        Note that if one of the maximum terms in the last sum is actually zero, then the bound is easier to satisfy. Therefore, we consider only the case when the maximum is always equal to the second term. We continue as follows 
        \begin{align*}
            \|\overline{g}^0\| &\le \frac{1}{2}\sqrt{2L\Phi^0} 
            + \frac{2\hat{\beta}\beta}{G}\sum_{i\in\cG}\|\theta^0_i\|
            + \frac{\hat{\beta}\beta}{G}\sum_{i\in\cG}\|\nabla f_i(x^0)\| - \hat{\beta}\tau\\
            &\le \sqrt{64L\Delta} 
            + 2\hat{\beta}\beta b + \hat{\beta}\beta B_{\rm init} - \hat{\beta}\tau\\
            &\le \sqrt{64L\Delta} 
            + 3(B_{\rm init}-\tau) + 3b + 3\sqrt{2c\deltabyz}\hbeta\tau T.
        \end{align*}

        The inequalities above again hold in $\cap_{i\in\cG}\overline{\Theta_i^0}$, i.e., with probability at least $1-\frac{\alpha}{8(T+1)}.$ Therefore, the condition $3$ of the induction is verified. Since at iteration $0$ we have $\overline{m}^0 = \overline{g}^0,$ the condition $2$ holds as well.

        \item We have 
        \begin{align*}
            \|v_i^0 - \nabla f_i(x^0)\| &= \|\beta\nabla  f_i(x^0,\xi_i^0) - \nabla f_i(x^0)\|\\ &\le \beta\|\nabla f_i(x^0, \xi^0_i) - \nabla f_i(x^0)\| + (1-\beta)\|\nabla f_i(x^0)\|\\
            &\le \beta b + (1-\beta)B_{\rm init}
        \end{align*}
        The bound above holds with probability at least $1-\frac{\alpha}{8(T+1)}$ because it holds in $\cap_{i\in\cG}\overline{\Theta_i^0}.$ Therefore, the bound $5$ of the assumption of the induction is verified. 
        
        \item Next, we emphasize that the condition $8$ of the induction assumption also hold, as $\Phi^0 \le 2\Phi^0 \le 2\Delta$ by the choice of $\Delta$.

        \item We finalize the induction base by noting that the condition $9$ of the induction assumption holds since the RHS equals $0$.

    \end{enumerate}
    Therefore, we conclude that the conditions $1$-$8$ hold with a probability of at least 
    \begin{align*}
        \Prob\left(\overline{\Theta^0}\cap \left(\cap_{i\in\cG}\overline{\Theta_i^0}\right)\cap \overline{N}^t \cap \left(\cap_{i\in\cG}\overline{N}_i^t\right)\right) &\ge 1 - \Prob(\Theta^0) - \sum_{i\in\cG} \Prob(\Theta_i^0) - \Prob(N^0) - \sum_{i\in\cG}\Prob(N_i^0)\\
        &\ge 
        1 
        - \frac{\alpha}{8(T+1)}
        - G\cdot \frac{\alpha}{8G(T+1)} 
        - \frac{\alpha}{8(T+1)} - G\cdot \frac{\alpha}{8G(T+1)}\\
        &= 1-\frac{\alpha}{2(T+1)} > 1- \frac{\alpha}{T+1},
    \end{align*}
    i.e., $\Prob(E^0) \ge 1-\frac{\alpha}{T+1}$ holds. This is the base of the induction.

    \paragraph{Transition step of induction.}

    {\bf Case $|\cI_{K+1}| > 0.$}  Assume that all events $\overline{\Theta}^{K+1}, \overline{\Theta}^{K+1}_i$, $\overline{N}^{K+1},$ and $\overline{N}^{K+1}_i$ take place, i.e., $\|\theta^{K+1}_i\| \le b, \|\theta^{K+1}\|\le \frac{c}{\sqrt{G}}$ for all $i\in\cG$, $\left\|\frac{1}{G}\sum_{l=1}^{K}\sum_{i\in\cG}\omega_i^l\right\|\le a,$ and $\left\|\sum_{l=1}^{K}\omega_i^l\right\| \le \ha$ for all $i\in\cG$. That is, we assume that the event 
    $$\overline{\Theta}^{K+1} \cap \left(\cap_{i\in\cG} \overline{\Theta}^{K+1}_i\right) \cap \overline{N}^{K+1} \cap \left(\cap_{i\in\cG} \overline{N}_i^{K+1}\right) \cap E^K$$ holds. Then, by the assumptions of the induction and from \Cref{lem:control_clip} we get for all $i\in \cI_{K+1}$
    \[
    \|v_i^{K+1} - g_i^{K}\| \le \|v_i^{K} - g_i^{K-1}\| - \frac{\hat{\beta}\tau}{2} \le B_{\rm init} - \frac{\hat{\beta}\tau}{2}.
    \]
    Therefore, from \Cref{lem:bound_mt_bar} we get that 
    \[
    \|\overline{m}^{K+1}\| \le \sqrt{64L\Delta} + 3(B_{\rm init}-\tau) + 3b + 3\hat{\beta}a + \hbeta\sqrt{2c\deltabyz}\ha + 3\sqrt{2c\deltabyz}\hbeta\tau T,
    \]
    from \Cref{lem:bound_gt_hat} we get that
    \[
    \|\overline{g}^{K+1}\| \le \sqrt{64L\Delta} + 3(B_{\rm init}-\tau) + 3b + 3\sqrt{2c\deltabyz}\hbeta\tau T,
    \]

    from \Cref{lem:bound_nabla_fi_xt_vi_t}
    \[
    \|\nabla f_i(x^{K+1}) - v_i^{K+1}\| \le \sqrt{4L\Delta} + \frac{3}{2}(B_{\rm init}-\tau) + \frac{3}{2}b + \hat{\beta}a + 4\hbeta\sqrt{2c\deltabyz}\ha + \frac{3}{2}\sqrt{2c\deltabyz}\hbeta\tau T\quad \text{for all } i\in\cG.
    \]
    This means that conditions 1-7 in the induction assumption are also verified for the step $K+1$. Since for all $t\in\{0, \dots, K+1\}$ inequalities $1$-$7$ are verified, we can write for each $t\in \{0,\ldots,K\}$ by \Cref{lem:descent_lemma_in_f,lem:descent_Pt,lem:descent_Pt_tilde,lem:descent_Vt_tilde} the following

    \begin{align*}
        \Phi^{t+1} &= \delta^{t+1} 
        + \frac{2\gamma}{\hat{\beta}\eta}\wtilde{V}^{t+1} 
        + \frac{8\gamma\beta}{\hat{\beta}^2\eta^2}\wtilde{P}^{t+1} 
        + \frac{2\gamma}{\beta}P^{t+1}\\
        &\le \delta^t - \frac{\gamma}{2}\|\nabla f(x^t)\|^2  {\color{red} - \frac{1}{4\gamma}R^t}
        {\color{blue} + 2\gamma \wtilde{V}^t }
       {\color{orange}+ 2\gamma P^t}
       + 2\gamma\hbeta^2\|\Omega^t\|^2
       + 16\gamma\hbeta^2c\deltabyz\tau^2T^2
       + \frac{16\gamma\hbeta^2c\deltabyz}{G}\sum_{i\in\cG}\|\Omega_i^t\|^2\\
        & \;+\; \frac{2\gamma}{\hat{\beta}\eta}\left(
    {\color{blue} (1-\hat{\beta}\eta)\wtilde{V}^{t} }
    {\color{pink} + \frac{4\beta^2}{\hat{\beta}\eta}\wtilde{P}^{t}}
    {\color{red} + \frac{4\beta^2L^2}{\hat{\beta}\eta}R^{t}}
    + \beta^2b^2
    \right.\\
    &\;+\; \left.\frac{2}{G}\beta(1-\hat{\beta}\eta)^2\sum_{i\in\cG}\<(g_i^{t} - v_i^{t}) + \beta(v_i^{t} - \nabla f_i(x^{t})) + \beta(\nabla f_i(x^{t})-\nabla f_i(x^{t+1})), \theta^{t+1}_i>\right)\\
    &\;+\; \frac{8\gamma\beta}{\hat{\beta}^2\eta^2}\left({\color{pink}(1-\beta)\wtilde{P}^t} 
    {\color{red} + \frac{3L^2}{\beta}R^t }
    + \beta^2b^2 
    + \frac{2}{G}\beta(1-\beta)\sum_{i\in\cG}\<v_i^t - \nabla f_i(x^{t+1}), \theta^{t+1}_i>\right)\\
    &\;+\; \frac{2\gamma}{\beta}\left({\color{orange} (1-\beta)P^t }
    {\color{red} 
    + \frac{3L^2}{\beta}R^t }
    + \beta^2\frac{\hc^2}{G}
    + 2\beta(1-\beta)\<v^t - \nabla f(x^{t+1}), \theta^{t+1}>\right)
    \end{align*}
    
    Rearranging terms, we get
    \begin{align*}
    \Phi^{t+1} &\le \delta^t 
    - \frac{\gamma}{2}\|\nabla f(x^t)\|^2
    + \frac{2\gamma}{\hat{\beta}\eta}\wtilde{V}^t\left(\hat{\beta}\eta + 1-\hat{\beta}\eta\right)
    + \frac{8\gamma\beta}{\hat{\beta}^2\eta^2}\wtilde{P}^t\left(\beta + 1 - \beta\right)
    + \frac{2\gamma}{\beta}P^t\left(\beta + 1-\beta\right)\\
    &\quad + 2\gamma\hbeta^2\|\Omega^t\|^2
    + 16\gamma\hbeta^2c\deltabyz\tau^2T^2
    + 16\gamma\hbeta^2\frac{c\deltabyz}{G}\sum_{i\in\cG}\|\Omega_i^t\|^2\\
    & \quad - \frac{1}{4\gamma}R^t\left(1 - \frac{32L^2\beta^2}{\hat\beta^2\eta^2}\gamma^2
    - \frac{96L^2}{\hat\beta^2\eta^2}\gamma^2
    - \frac{24L^2}{\beta^2}\gamma^2\right)
    + b^2\left(\frac{2\beta^2\gamma}{\hat{\beta}\eta} + \frac{8\gamma\beta^3}{\hat{\beta}^2\eta^2}\right)
    + \hc^2\frac{2\gamma\beta}{G}\\
    & \quad + \frac{4\gamma\beta}{G\hat{\beta}\eta}(1-\hat{\beta}\eta)^2\sum_{i\in\cG}\<(g_i^{t} - v_i^{t}) + \beta(v_i^{t} - \nabla f_i(x^{t})) + \beta(\nabla f_i(x^{t})-\nabla f_i(x^{t+1})), \theta^{t+1}_i>\\
    &\quad + \frac{16\gamma\beta^2}{G\hat{\beta}^2\eta^2}(1-\beta)\sum_{i\in\cG}\<v_i^t - \nabla f_i(x^{t}), \theta^{t+1}_i>
    + 4\gamma(1-\beta)\<v^t - \nabla f(x^{t}), \theta^{t+1}>\\
    &\quad + \frac{16\gamma\beta^2}{G\hat{\beta}^2\eta^2}(1-\beta)\sum_{i\in\cG}\<\nabla f_i(x^t) - \nabla f_i(x^{t+1}), \theta^{t+1}_i>\\
    &\quad + 4\gamma(1-\beta)\<\nabla f(x^t) - \nabla f(x^{t+1}), \theta^{t+1}>.
    \end{align*}
    Using step-size restriction $(c)$, and assumption $6$ of the induction, we get rid of the term with $R^t$ and obtain
    \begin{align*}
        \Phi^{t+1} 
    &\le \Phi^t 
    - \frac{\gamma}{2}\|\nabla f(x^t)\|^2
    + b^2\left(\frac{2\beta^2\gamma}{\hat{\beta}\eta} + \frac{8\gamma\beta^3}{\hat{\beta}^2\eta^2}\right)
    + \hc^2\frac{2\gamma\beta}{G}
    + 2\gamma\hbeta^2a^2
    + 16\gamma\hbeta^2c\deltabyz\tau^2T^2
    + 16\gamma\hbeta^2c\deltabyz\ha^2\\
    & \quad + \frac{4\gamma\beta}{G\hat{\beta}\eta}(1-\hat{\beta}\eta)^2\sum_{i\in\cG}\<(g_i^{t} - v_i^{t}) + \beta(v_i^{t} - \nabla f_i(x^{t})) + \beta(\nabla f_i(x^{t})-\nabla f_i(x^{t+1})), \theta^{t+1}_i>\\
    &\quad + \frac{16\gamma\beta^2}{G\hat{\beta}^2\eta^2}(1-\beta)\sum_{i\in\cG}\<v_i^t - \nabla f_i(x^{t}), \theta^{t+1}_i>
    + 4\gamma(1-\beta)\<v^t - \nabla f(x^{t}), \theta^{t+1}>\\
    &\quad +\frac{16\gamma\beta^2}{G\hat{\beta}^2\eta^2}(1-\beta)\sum_{i\in\cG}\<\nabla f_i(x^t) - \nabla f_i(x^{t+1}), \theta^{t+1}_i>\\
    & \quad + 4\gamma(1-\beta)\<\nabla f(x^t) - \nabla f(x^{t+1}), \theta^{t+1}>.
    \end{align*}
    Now we sum all the inequalities above for $t\in\{0,\dots,K\}$ and get 
    \begin{align}
    \Phi^{K+1} 
    &\le \Phi^0 
    - \frac{\gamma}{2}\sum_{t=0}^{K}\|\nabla f(x^t)\|^2
    + Kb^2\left(\frac{2\beta^2\gamma}{\hat{\beta}\eta} + \frac{8\gamma\beta^3}{\hat{\beta}^2\eta^2} \right)
    + K\hc^2\frac{2\gamma\beta}{G} 
    + 2\gamma\hbeta^2a^2 K
    + 16\gamma\hbeta^2c\deltabyz\ha^2K\notag\\
    &\quad + 16\gamma\hbeta^2c\deltabyz\tau^2T^2K\notag\\
    & \quad + \frac{4\gamma\beta}{G\hat{\beta}\eta}(1-\hat{\beta}\eta)^2\sum_{t=0}^{K}\sum_{i\in\cG}\<(g_i^{t} - v_i^{t}) + \beta(v_i^{t} - \nabla f_i(x^{t})) + \beta(\nabla f_i(x^{t})-\nabla f_i(x^{t+1})), \theta^{t+1}_i>\notag\\
    &\quad + \frac{16\gamma\beta^2}{G\hat{\beta}^2\eta^2}(1-\beta)\sum_{t=0}^{K}\sum_{i\in\cG}\<v_i^t - \nabla f_i(x^{t}), \theta^{t+1}_i>
    + 4\gamma(1-\beta)\sum_{t=0}^{K}\<v^t - \nabla f(x^{t}), \theta^{t+1}>\notag\\
    &\quad + \frac{16\gamma\beta^2}{G\eta^2}(1-\beta)\sum_{t=0}^{K}\sum_{i\in\cG}\<\nabla f_i(x^t) - \nabla f_i(x^{t+1}), \theta^{t+1}_i>\notag\\
    & \quad + 4\gamma(1-\beta)\sum_{t=0}^{K}\<\nabla f(x^t) - \nabla f(x^{t+1}), \theta^{t+1}>.\label{eq:njqnsfqknj}
    \end{align}
    Rearranging terms and using step-size restriction $(b)$, we get 
    \begin{align*}
    &\frac{\gamma}{2}\sum_{t=0}^{K}\|\nabla f(x^t)\|^2
    \le \Phi^0 - \Phi^{K+1}
    + K b^2\left(\frac{2\beta^2\gamma}{\hat{\beta}\eta} + \frac{8\gamma\beta^3}{\hat{\beta}^2\eta^2}\right)
    + K\hc^2\frac{2\gamma\beta}{G}\\
    &\quad + \frac{\beta}{6L}\hbeta^2a^2 K 
    + \frac{4\beta}{3L}\hbeta^2c\deltabyz\ha^2K
    + \frac{4\beta}{3L}\hbeta^2\tau^2c\deltabyz T^2K \\
    & \quad + \frac{4\gamma\beta}{G\hat{\beta}\eta}(1-\hat{\beta}\eta)^2\sum_{t=0}^{K}\sum_{i\in\cG}\<(g_i^{t} - v_i^{t}) + \beta(v_i^{t} - \nabla f_i(x^{t})) + \beta(\nabla f_i(x^{t})-\nabla f_i(x^{t+1})), \theta^{t+1}_i>\\
    &\quad + \frac{16\gamma\beta^2}{G\hat{\beta}^2\eta^2}(1-\beta)\sum_{t=0}^{K}\sum_{i\in\cG}\<v_i^t - \nabla f_i(x^{t}), \theta^{t+1}_i>
    + 4\gamma(1-\beta)\sum_{t=0}^{K}\<v^t - \nabla f(x^{t}), \theta^{t+1}>\\
    &\quad + \frac{16\gamma\beta^2}{G\hat{\beta}^2\eta^2}(1-\beta)\sum_{t=0}^{K}\sum_{i\in\cG}\<\nabla f_i(x^t) - \nabla f_i(x^{t+1}), \theta^{t+1}_i>\\
    & \quad + 4\gamma(1-\beta)\sum_{t=0}^{K}\<\nabla f(x^t) - \nabla f(x^{t+1}), \theta^{t+1}>
    .
    \end{align*}
    Taking into account that $\frac{\gamma}{2}\sum_{t=0}^{K}\|\nabla f(x^t)\|^2 \ge 0$ and using momentum restriction (a), we get that the event $E^K\cap \left(\cap_{i\in\cG}\overline{\Theta}^{K+1}_i\right)\cap\overline{N}^{K+1} \cap \overline{\Theta}^{K+1}\cap\left(\cap_{i\in\cG}\overline{N}^{K+1}_i\right)$ implies 
    \begin{align*}
    &\Phi^{K+1}
    \le \Phi^0 
    + K b^2\left(\frac{2\beta^2\gamma}{\hat{\beta}\eta} + \frac{8\gamma\beta^3}{\hat{\beta}^2\eta^2}\right)
    + K\hc^2\frac{2\gamma\beta}{G}
    + \frac{\tau\hbeta^3 a^2 K}{324L\sqrt{L\Delta}}
    + \frac{2\tau \hbeta^3c\deltabyz\ha^2K}{81L\sqrt{L\Delta}}
    + \frac{2\tau^3\hbeta^3c\deltabyz T^2K}{81L\sqrt{L\Delta}}
    \\
    & \quad + \frac{4\gamma\beta}{G\hat{\beta}\eta}(1-\hat{\beta}\eta)^2\sum_{t=0}^{K}\sum_{i\in\cG}\<(g_i^{t} - v_i^{t}) + \beta(v_i^{t} - \nabla f_i(x^{t})) + \beta(\nabla f_i(x^{t})-\nabla f_i(x^{t+1})), \theta^{t+1}_i>\\
    &\quad + \frac{16\gamma\beta^2}{G\hat{\beta}^2\eta^2}(1-\beta)\sum_{t=0}^{K}\sum_{i\in\cG}\<v_i^t - \nabla f_i(x^{t}), \theta^{t+1}_i>
    + \frac{4\gamma(1-\beta)}{G}\sum_{t=0}^{K}\sum_{i\in\cG}\<v^t - \nabla f(x^{t}), \theta^{t+1}_i>\\
    &\quad + \frac{16\gamma\beta^2}{G\hat{\beta}^2\eta^2}(1-\beta)\sum_{t=0}^{K}\sum_{i\in\cG}\<\nabla f_i(x^t) - \nabla f_i(x^{t+1}), \theta^{t+1}_i>\\
    &\quad + \frac{4\gamma(1-\beta)}{G}\sum_{t=0}^{K}\sum_{i\in\cG}\<\nabla f(x^t) - \nabla f(x^{t+1}), \theta^{t+1}_i>.
    \end{align*}
    Next, we define the following random vectors:
    \begin{align*}
    &\zeta_{1,i}^t \eqdef \begin{cases} 
    g_i^t - v_i^t, &\text{ if } \|g_i^t-v_i^t\| \le B_{\rm init}\\
    0, &\text{otherwise}
    \end{cases},\\
    &\zeta_{2,i}^t \eqdef \begin{cases} 
    v_i^t - \nabla f_i(x^t), &\text{ if } \|v_i^t - \nabla f_i(x^t)\| \le \sqrt{4L\Delta} + \frac{3}{2}(B_{\rm init}-\tau) + \frac{3}{2}b + \frac{3}{2}\sqrt{2c\deltabyz}\hbeta\tau T + \hbeta a\\
    &\hspace{10cm} +\; 4\hbeta\sqrt{2c\deltabyz}\ha \\
	0, &\text{otherwise}
    \end{cases}, \\
    &\zeta_{3,i}^t \eqdef \begin{cases} 
    \nabla f_i(x^t) - \nabla f_i(x^{t+1}), & \text{ if } \|\nabla f_i(x^t) - \nabla f_i(x^{t+1})\| \le L\gamma\left(\sqrt{64L\Delta} + 3(B_{\rm init}-\tau) + 3 b + 3\hbeta a\right.\\
    &\hspace{6.5cm}\left. +\;  3\sqrt{2c\deltabyz}\ha\hbeta + 
 5\sqrt{2c\deltabyz}\hbeta\tau T \right)\\
	0, &\text{otherwise}
    \end{cases}, \\
    &\zeta_{4}^t \eqdef \begin{cases} 
    v^t - \nabla f(x^{t}), &\text{ if } \|v^t - \nabla f(x^{t})\| \le \sqrt{4L\Delta} + \frac{3}{2}(B_{\rm init}-\tau) + \frac{3}{2}b + \frac{3}{2}\sqrt{2c\deltabyz}\hbeta\tau T + \hat{\beta}a\\
    & \hspace{10cm}+\; 4\hbeta\sqrt{2c\deltabyz}\ha\\
	0, &\text{otherwise}
    \end{cases}, \\
    &\zeta_{5}^t \eqdef \begin{cases} 
    \nabla f(x^t) - \nabla f(x^{t+1}), &\text{ if } \|\nabla f(x^t) - \nabla f(x^{t+1})\| \le L\gamma\left(\sqrt{64L\Delta} + 3(B_{\rm init}-\tau) + 3 b + 3\hbeta a\right.\\
    &\hspace{6.5cm}\left. +\;  3\sqrt{2c\deltabyz}\ha\hbeta + 
 5\sqrt{2c\deltabyz}\hbeta\tau T \right)\\
	0, &\text{otherwise}
    \end{cases}.
    \end{align*}
    By definition, all introduced random vectors $\zeta_{l,i}^t, l\in[3], i\in\cG, \zeta_{4,5}^t$ are bounded with probability $1$. Moreover, by the definition of $E^t$ we get that the event $E^K\cap \overline{\Theta}^{K+1}\cap \left(\cap_{i\in\cG}\overline{\Theta}^{K+1}_i\right) \cap \overline{N}^{K+1}\cap\left(\cap_{i\in\cG}\overline{N}^{K+1}_i\right)$ implies 
    \begin{align*}
    &\zeta_{1,i}^t = g_i^t-v_i^t,\quad  \zeta_{2,i}^t = v_i^t-\nabla f_i(x^t), \quad \zeta_{3,i}^t = \nabla f_i(x^t) - \nabla f_i(x^{t+1}),\\
    &\zeta_{4}^t = v^t-\nabla f(x^t),\quad \zeta_{5}^t = \nabla f(x^t) - \nabla f(x^{t+1}).
    \end{align*}
    Therefore, the event $E^K\cap \overline{\Theta}^{K+1} \cap \left(\cap_{i\in\cG}\overline{\Theta}^{K+1}_i\right) \cap \overline{N}^{K+1}\cap\left(\cap_{i\in\cG}\overline{N}^{K+1}_i\right)$ implies 
    \begin{align*}
    &\Phi^{K+1}
    \le \Phi^0 
    + \underbrace{K b^2\left(\frac{2\beta^2\gamma}{\hat{\beta}\eta} + \frac{8\gamma\beta^3}{\hat{\beta}^2\eta^2}\right) + K\hc^2\frac{2\gamma\beta}{G}}_{\circledOne}
    + \underbrace{\frac{4\gamma\beta}{G\hat{\beta}\eta}(1-\eta)^2\sum_{t=0}^{K}\sum_{i\in\cG}\<\zeta_{1,i}^t, \theta^{t+1}_i>}_{\circledTwo}\\
    &\;+\; \underbrace{\frac{4\gamma\beta^2}{G\hat{\beta}\eta}(1-\hat{\beta}\eta)^2\sum_{t=0}^{K}\sum_{i\in\cG}\<\zeta_{2,i}^t, \theta^{t+1}_i>}_{\circledThree}
    +  \underbrace{\frac{4\gamma\beta^2}{G\hat{\beta}\eta}(1-\hat{\beta}\eta)^2\sum_{t=0}^{K}\sum_{i\in\cG}\<\zeta_{3,i}^t, \theta^{t+1}_i>}_{\circledFour}\\
    &\;+\; \underbrace{\frac{16\gamma\beta^2}{G\hat{\beta}^2\eta^2}(1-\beta)\sum_{t=0}^{K}\sum_{i\in\cG}\<\zeta_{2,i}^t, \theta^{t+1}_i>}_{\circledFive}
    + \underbrace{\frac{4\gamma(1-\beta)}{G}\sum_{t=0}^{K}\sum_{i\in\cG}\<\zeta_{4}^t, \theta^{t+1}_i>}_{\circledSix}\\
    &\;+\; \underbrace{\frac{16\gamma\beta^2}{G\hat{\beta}^2\eta^2}(1-\beta)\sum_{t=0}^{K}\sum_{i\in\cG}\<\zeta_{3,i}^t, \theta^{t+1}_i>}_{\circledSeven}
    + \underbrace{\frac{4\gamma(1-\beta)}{G}\sum_{t=0}^{K}\sum_{i\in\cG}\<\zeta_{5}^t, \theta^{t+1}_i>}_{\circledEight}\\
    &\;+\; \underbrace{\frac{\tau\hbeta^3 a^2 K}{324L\sqrt{L\Delta}}}_{\circledNine}
    + \underbrace{
    \frac{2\hbeta^3c\deltabyz\ha^2 K}{81L\sqrt{L\Delta}} + \frac{2\tau^3\hbeta^3c\deltabyz T^2K}{81L\sqrt{L\Delta}}}_{\circledTen}.
    \end{align*}
    
    \subparagraph{Bound of the term $\circledOne$.} Since $12L\gamma \leq \beta$, for the term $\circledOne$ we have 
    \begin{align*}
        K b^2\left(\frac{2\beta^2\gamma}{\hat{\beta}\eta} + \frac{8\gamma\beta^3}{\hat{\beta}\eta^2}\right)
        + K\hc^2\frac{2\gamma\beta}{G}&
        \le K b^2\left(\frac{\beta^{3}}{6L\hat{\beta}\eta} + \frac{2\beta^4}{3L\hat{\beta}^2\eta^2}\right)
        + K\hc^2\frac{\beta^2}{6LG}.
    \end{align*}
    By choosing $\beta$ such that 
    \begin{equation}\label{eq:step-size_bound_1}
    \beta \le \min\left\{
    \left(\frac{L\Delta\hat{\beta}\eta}{12Tb^2}\right)^{1/3},
    \left(\frac{L\Delta\hat{\beta}^2\eta^2}{48Tb^2}\right)^{1/4},
    \left(\frac{L\Delta G}{12T\hc^2}\right)^{1/2}
    \right\}
    \end{equation}
    we get that 
    \[
        K b^2\left(\frac{2\beta^2\gamma}{\hat{\beta}\eta} + \frac{8\gamma\beta^3}{\hat{\beta}^2\eta^2}\right)
        + K\hc^2\frac{2\gamma\beta}{G} \le 3 \cdot \frac{\Delta}{72} = \frac{\Delta}{24}.
    \]
    This bound holds with probability $1.$ Note that the worst dependency in the restriction on $\beta$ w.r.t. $T$ is $\cO(\nicefrac{1}{T^{3/4}})$ since $\hat{\beta} \sim \frac{1}{a} \sim \frac{1}{T}$ that comes from the second term in \eqref{eq:step-size_bound_1}.

    \subparagraph{Bound of the term $\circledTwo$.}

    For term $\circledTwo$, let us enumerate random variables as 
    \[
    \<\zeta_{1,1}^0,\theta^1_1>, \dots, \<\zeta_{1,G}^0,\theta^1_G>, \<\zeta_{1,1}^1,\theta^2_1>,\dots, \<\zeta_{1,G}^1,\theta^2_G>, \dots 
    \<\zeta_{1,1}^K,\theta^{K+1}_1>,\dots,
    \<\zeta_{1,G}^K,\theta^{K+1}_G>,
    \]
    i.e., first by index $i$, then by index $t$. Then we have that the event $E^K\cap \left(\cap_{i\in\cG}\overline{\Theta}^{K+1}_i\right)$ implies 
    \[
    \E{\frac{4\gamma\beta}{G\hat{\beta}\eta}(1-\eta)^2\<\zeta^l_{1,i},\theta^{l+1}_i>\mid  \<\zeta_{1,i-1}^l,\theta^{l+1}_{i-1}>, \dots, \<\zeta_{1,1}^l,\theta^{l+1}_{1}>, \dots, \<\zeta_{1,1}^{0},\theta^{1}_{1}>} = 0,
    \]
    because $\{\theta^{l+1}_i\}_{i\in\cG}$ are independent. Let 
    \[
    \sigma_{2}^2 \eqdef \frac{16\gamma^2\beta^2}{G^2\hat{\beta}^2\eta^2}\cdot B_{\rm init}^2 \cdot \sigma^2.
    \]
    Since $\theta^{l+1}_i$ is $\sigma$-sub-Gaussian random vector, for $$\E{\cdot \mid l,i-1} \eqdef \E{ \cdot \mid  \<\zeta_{1,i-1}^l,\theta^{l+1}_{i-1}>, \dots, \<\zeta_{1,1}^l,\theta^{l+1}_{1}>, \dots, \<\zeta_{1,1}^{0},\theta^{1}_{1}>}$$ we have
    \begin{align*}
    &\E{\exp\left(\left|\frac{1}{\sigma_2^2}\frac{16\gamma^2\beta^2}{G^2\hat{\beta}^2\eta^2}(1-\eta)^4\<\zeta^l_{1,i},\theta^{l+1}_i>^2\right| \right) \mid l, i-1}\\
    &\le 
    \E{\exp\left(\frac{1}{\sigma^2_1}\frac{16\gamma^2\beta^2}{G^2\hat{\beta}^2\eta^2}\|\zeta_{1,i}^l\|^2 \cdot \|\theta^{l+1}_i\|^2 \right) \mid l,i-1}\\
    &\le \E{\exp\left(\frac{1}{\sigma_2^2}\frac{16\gamma^2\beta^2}{G^2\hat{\beta}^2\eta^2}\cdot B_{\rm init}^2\|\theta^{l+1}_i\|^2\right) \mid l, i-1} \\
    &\le \E{\exp\left(\frac{  G^2\hat{\beta}^2\eta^2}{16\gamma^2\beta^2\cdot B_{\rm init}^2 \cdot \sigma^2}\frac{16\gamma^2\beta^2}{G^2\hat{\beta}^2\eta^2}\cdot B_{\rm init}^2\|\theta^{l+1}_i\|^2\right) \mid l, i-1} \\  
    &= \E{\exp\left(\frac{\|\theta^{l+1}_i\|^2}{\sigma^2}\mid l, i-1\right)} \le \exp(1).
    \end{align*}
    Therefore, we have by \Cref{lem:concentration_lemma} with $\sigma_k^2 \equiv \sigma_2^2$ that 
    \begin{align*}
    &\Prob\left(\frac{4\gamma\beta}{G\hat{\beta}\eta}(1-\hat{\beta}\eta)^2\left\|\sum_{t=0}^K\sum_{i\in\cG}\<\zeta_{1,i}^t,\theta^{t+1}_i>\right\| 
    \ge (\sqrt{2}+\sqrt{2}b_1)\sqrt{\sum_{t=0}^K\sum_{i\in\cG}\frac{16B_{\rm init}^2\gamma^2\beta^2\sigma^2}{G^2\hat{\beta}^2\eta^2}}\right) \\
    & \quad \le \exp(-\nicefrac{b_1^2}{3}) \\
    &\quad = \frac{\alpha}{20(T+1)}
    \end{align*}
    with $b_1^2 = 3\log\left(\frac{20(T+1)}{\alpha}\right)$.  Note that since $12L\gamma\le \beta$
    \begin{align*}
    (\sqrt{2} + \sqrt{2}b_1)\sqrt{\sum_{t=0}^K\sum_{i\in\cG}\frac{16B_{\rm init}^2\gamma^2\beta^2\sigma^2}{G^2\hat{\beta}^2\eta^2}} 
    &\le (\sqrt{2} + \sqrt{2}b_1)\sqrt{\sum_{t=0}^K\sum_{i\in\cG}\frac{B_{\rm init}^2\beta^{4}\sigma^2}{9L^2G^2\hat{\beta}^2\eta^2}} \\
    &= (\sqrt{2} + \sqrt{2}b_1)\frac{B_{\rm init}\beta^{2} \sigma}{3LG\hat{\beta}\eta}\sqrt{(K+1)G}\\
    & \le \frac{\Delta}{8},
    \end{align*}
    because we choose $\beta$ such that 
    \begin{equation}\label{eq:step-size_bound_2}
        \beta \le \left(\frac{3L\Delta\sqrt{G}\hat{\beta}\eta}{8\sqrt{2}(1+b_1)B_{\rm init}\sigma\sqrt{T}}\right)^{1/2}, \quad \text{ and } \quad K+1 \le T.
    \end{equation}
    This implies that 
    \[
    \Prob\left(\frac{4\gamma\beta}{G\hat{\beta}\eta}(1-\hat{\beta}\eta)^2\left\|\sum_{t=0}^K\sum_{i\in\cG}\<\zeta_{1,i}^t,\theta^{t+1}_i>\right\| \ge \frac{\Delta}{8}\right) \le \frac{\alpha}{20(T+1)}
    \]
    with this choice of momentum parameter. The dependency of \eqref{eq:step-size_bound_2} on $T$ is $\wtilde\cO(\nicefrac{1}{T^{3/4}})$ since $\hat{\beta} \sim \frac{1}{T}.$

    \subparagraph{Bound of the term $\circledThree$.} The bound in this case is similar to the previous one. Let 
    \[
    \sigma_3^2 \eqdef \frac{16\gamma^2\beta^4}{G^2\hat{\beta}^2\eta^2}\cdot \left(\sqrt{4L\Delta} + \frac{3}{2}(B_{\rm init}-\tau) + \frac{3}{2}\sqrt{2c\deltabyz}\hbeta\tau T + \frac{3}{2}b +  \hat{\beta}a + 4\hbeta\sqrt{2c\deltabyz}\ha\right)^2\cdot \sigma^2.
    \]
    
    Then,
    \begin{align*}
    &\E{\exp\left(\left|\frac{1}{\sigma_3^2}\frac{16\gamma^2\beta^4}{  G^2\hat{\beta}^2\eta^2}(1-\hat{\beta}\eta)^4\<\zeta^l_{2,i},\theta^{l+1}_i>^2\right|\right) \mid l,i-1} \\
    &\le 
    \E{\exp\left(\frac{1}{\sigma_3^2}\frac{16\gamma^2\beta^4}{G^2\hat{\beta}^2\eta^2}\|\zeta_{2,i}^l\|^2 \cdot \|\theta^{l+1}_i\|^2\right)}\\
    &\le \EE\left[\exp\left(\frac{1}{\sigma^3_2}\frac{16\gamma^2\beta^4}{G^2\hat{\beta}^2\eta^2}\cdot \left(\sqrt{4L\Delta} + \frac{3}{2}(B_{\rm init}-\tau) + \frac{3}{2}b + \frac{3}{2}\sqrt{2c\deltabyz}\hbeta\tau T + \frac{3}{2}\sqrt{2c\deltabyz}\hbeta\tau T + \hat{\beta}a\right.\right.\right. &\\
    &\hspace{10cm} +\; 4\hbeta\sqrt{2c\deltabyz}\ha\bigg)^2\cdot \|\theta_i^{l+1}\|^2\bigg)\mid l,i-1\bigg]\\
    &\le \mathbb{E}\left[\exp\left(\left[\frac{16\gamma^2\beta^4}{G^2\hat{\beta}^2\eta^2}\cdot \left(\sqrt{4L\Delta} + \frac{3}{2}(B_{\rm init}-\tau) + \frac{3}{2}b + \frac{3}{2}\sqrt{2c\deltabyz}\hbeta\tau T + \hat{\beta}a + 4\hbeta\sqrt{2c\deltabyz}\ha\right)^2\cdot \sigma^2\right]^{-1}\times\right.\right.\\
    &\qquad \left.\left.\frac{16\gamma^2\beta^4}{G^2\hat{\beta}^2\eta^2}\cdot \left(\sqrt{4L\Delta} + \frac{3}{2}(B_{\rm init}-\tau) + \frac{3}{2}b + \frac{3}{2}\sqrt{2c\deltabyz}\hbeta\tau T + \hat{\beta}a + 4\hbeta\sqrt{2c\deltabyz}\ha\right)^2\cdot \|\theta_i^{l+1}\|^2\right)\mid l,i-1\right]\\
    &= \E{\exp\left(\frac{\|\theta^{l+1}_i\|^2}{\sigma^2}\right)\mid l,i-1} \le \exp(1).
    \end{align*}
    Therefore, we have by \Cref{lem:concentration_lemma} that 
    \begin{align*}
    &\Prob\left[\frac{4\gamma\beta^2}{G\hat{\beta}\eta}(1-\hat{\beta}\eta)^2\left\|\sum_{t=0}^K\sum_{i\in\cG}\<\zeta_{2,i}^t,\theta^{t+1}_i>\right\| \right.\\
    &\ge\;  \left.(\sqrt{2}+\sqrt{2}b_1)\sqrt{\sum_{t=0}^K\sum_{i\in\cG} \frac{16\gamma^2\beta^4\sigma^2}{G^2\hat{\beta}^2\eta^2} \left(\sqrt{4L\Delta}+ \frac{3}{2}(B_{\rm init}-\tau) + \frac{3}{2}b + \frac{3}{2}\sqrt{2c\deltabyz}\hbeta\tau T + \hat{\beta}a + 4\hbeta\sqrt{2c\deltabyz}\ha\right)^2}\right]\\ 
    &\le \exp(-\nicefrac{b_1^2}{3}) = \frac{\alpha}{20(T+1)}.
    \end{align*}
    Note that by using the restrictions $\hat{\beta} \le \min\left\{\frac{\sqrt{L\Delta}}{a}, \frac{\sqrt{L\Delta}}{4\sqrt{2c\deltabyz}\ha}\right\}$ and $12L\gamma \le \beta$ we get
    \begin{align*}
    & (\sqrt{2}+\sqrt{2}b_1)\sqrt{(K+1)G}\frac{4\gamma\beta^2\sigma}{\hat{\beta}\eta G}\left(\sqrt{4L\Delta}+ \frac{3}{2}(B_{\rm init}-\tau) + \frac{3}{2}b + \frac{3}{2}\sqrt{2c\deltabyz}\hbeta\tau T + \hat{\beta}a + 4\hbeta\sqrt{2c\deltabyz}\ha\right)\\
    \le\; & (\sqrt{2}+\sqrt{2}b_1)\sqrt{(K+1)G}\frac{\beta^{3}\sigma}{3L\hat{\beta}\eta G}\left(\sqrt{4L\Delta}+ \frac{3}{2}(B_{\rm init}-\tau) + \frac{3}{2}b + \frac{3}{2}\sqrt{2c\deltabyz}\hbeta\tau T + 2\sqrt{L\Delta}\right)\\
    =\; & (\sqrt{2}+\sqrt{2}b_1)\sqrt{(K+1)G}\frac{\beta^{3}\sigma}{3L\hat{\beta}\eta G}\left(4\sqrt{L\Delta} + \frac{3}{2}(B_{\rm init}-\tau) + \frac{3}{2}b \right) + \\
    &\hspace{9cm}  (1+b_1)\sqrt{(K+1)G}\frac{\beta^{3}\sigma}{L\eta G}\sqrt{c\deltabyz}\tau T\\
    \le\;& \frac{\Delta}{16} + \frac{\Delta}{16} = \frac{\Delta}{8} 
    \end{align*}
    holds because we choose 
    \begin{align}\label{eq:step-size_bound_3}
    \beta &\le \min\left\{\left(\frac{3L\Delta\hat{\beta}\eta \sqrt{G}}{16\sqrt{2}(1+b_1)\sigma\sqrt{T}\left(4\sqrt{L\Delta} + \frac{3}{2}(B_{\rm init}-\tau) + \frac{3}{2}b \right)}\right)^{1/3}\right.,\notag\\
    &\hspace{8.5cm} \left. \left(\frac{L\Delta\eta \sqrt{G}}{16(1+b_1)\sqrt{c\deltabyz}\sigma\tau T^{3/2} }\right)^{1/3}\right\},\\
    &\quad \text{ and } \quad K+1 \le T.\notag
    \end{align}
    This implies 
    \begin{align*}
    &\Prob\left(\frac{4\gamma\beta^2}{G\hat{\beta}\eta}(1-\hat{\beta}\eta)^2\left\|\sum_{t=0}^K\sum_{i\in\cG}\<\zeta_{2,i}^t,\theta^{t+1}_i>\right\| \ge \frac{\Delta}{8}\right) \le \frac{\alpha}{20(T+1)}.
    \end{align*}
    Note that both terms in the choice of $\beta$ have the dependency w.r.t. $T$ of order $\wtilde\cO(\nicefrac{1}{T^{1/2}})$ since $\hat{\beta} \sim \frac{1}{T}.$
    \subparagraph{Bound of the term $\circledFour$.} The bound in this case is similar to the previous one. Let
    \[
    \sigma_4^2 \eqdef \frac{16L^2\gamma^4\beta^4}{G^2\hat{\beta}^2\eta^2}\left(\sqrt{64L\Delta} + 3(B_{\rm init}-\tau) + 3b + 3\hbeta a  + 3\hbeta\sqrt{2c\deltabyz}\ha + 5\sqrt{2c\deltabyz}\hbeta\tau T\right)^2\cdot\sigma^2.
    \]
    Then we have
    \begin{align*}
    &\E{\exp\left(\left|\frac{1}{\sigma_4^2}\frac{16\gamma^2\beta^4}{G^2\hat{\beta}^2\eta^2}(1-\hat{\beta}\eta)^4\<\zeta^l_{3,i},\theta^{l+1}_i>^2\right| \right)\mid l,i-1}\\
    &\le 
    \E{\exp\left(\frac{1}{\sigma_4^2}\frac{16\gamma^2\beta^4}{G^2\hat{\beta}^2\eta^2}\|\zeta_{3,i}^l\|^2 \cdot \|\theta^{l+1}_i\|^2\right)\mid l,i-1}\\
    &\le \EE\left[\exp\left(\frac{1}{\sigma_4^2}\frac{16\gamma^2\beta^4}{G^2\hat{\beta}^2\eta^2}\cdot L^2\gamma^2\left(\sqrt{64L\Delta} + 3(B_{\rm init}-\tau) + 3b + 3\hbeta a  + 3\hbeta\sqrt{2c\deltabyz}\ha\right.\right.\right.\\
    &\hspace{8.5cm} \left.\left.\left.+\; 5\sqrt{2c\deltabyz}\hbeta\tau T\right)^2\cdot \|\theta^{l+1}_i\|^2\right)\mid l,i-1\right]\\
    &\le \mathbb{E}\left[\exp\left(\left[\frac{16L^2\gamma^4\beta^4}{G^2\hat{\beta}^2\eta^2}\left(\sqrt{64L\Delta} + 3(B_{\rm init}-\tau) + 3b + 3\hat{\beta}a + 3\hbeta\sqrt{2c\deltabyz}\ha + 5\sqrt{2c\deltabyz}\hbeta\tau T\right)^2\cdot\sigma^2\right]^{-1}\times\right.\right.\\
    &\qquad \left.\left.\frac{16L^2\gamma^4\beta^4}{G^2\hat{\beta}^2\eta^2}\left(\sqrt{64L\Delta} + 3(B_{\rm init}-\tau) + 3b + 3\hat{\beta}a + 3\hbeta\sqrt{2c\deltabyz}\ha + 5\sqrt{2c\deltabyz}\hbeta\tau T\right)^2\cdot \|\theta^{l+1}_i\|^2\right)\mid l,i-1\right]\\
    &= \E{\exp\left(\frac{\|\theta^{l+1}_i\|^2}{\sigma^2}\right)} \le \exp(1).
    \end{align*}
    Therefore, we have by \Cref{lem:concentration_lemma} that 
    \begin{align*}
    &\Prob\left(\frac{4\gamma\beta^2}{G\hat{\beta}\eta}(1-\hat{\beta}\eta)^2\left\|\sum_{t=0}^K\sum_{i\in\cG}\<\zeta_{3,i}^t,\theta^{t+1}_i>\right\| \right.\\
    &\ge \left.\sqrt{2}(1+b_1)\sqrt{\sum_{t=0}^K\sum_{i\in\cG} \frac{16L^2\gamma^4\beta^{4}\sigma^2}{G^2\hat{\beta}^2\eta^2} \left(\sqrt{64L\Delta} + 3(B_{\rm init}-\tau + b) + 3\hbeta a + 3\hbeta\sqrt{2c\deltabyz}\ha + 5\sqrt{2c\deltabyz}\hbeta\tau T\right)^2}\right)\\ 
    &\le \exp(-\nicefrac{b_1^2}{3}) = \frac{\alpha}{20(T+1)}.
    \end{align*}
    Using the restrictions $\hat{\beta} \le \min\left\{\frac{\sqrt{L\Delta}}{a},\frac{\sqrt{L\Delta}}{4\sqrt{2c\deltabyz}\ha}\right\}$ and $12L\gamma\le \beta$ we get
    \begin{align*}
    &\sqrt{2}(1+b_1)\sqrt{(K+1)G}\frac{4L\gamma^2\beta^2\sigma}{\hat{\beta}\eta G}\left(\sqrt{64L\Delta} + 3(B_{\rm init}-\tau + b) + 3\hat{\beta}a + 3\hbeta\sqrt{2c\deltabyz}\ha + 5\sqrt{2c\deltabyz}\hbeta\tau T\right)\\
    \le\; & 
    \sqrt{2}(1+b_1)\sqrt{(K+1)G}\frac{\beta^{4}\sigma}{36L\hat{\beta}\eta G}\left(\sqrt{64L\Delta} + 3(B_{\rm init}-\tau+b) + 4\sqrt{L\Delta} + 5\sqrt{2c\deltabyz}\hbeta\tau T\right)\\ 
    =\; & 
    \sqrt{2}(1+b_1)\sqrt{(K+1)G}\frac{\beta^{4}\sigma}{36L\hat{\beta}\eta G}\left(12\sqrt{L\Delta} + 3(B_{\rm init}-\tau+b)\right) + \\ 
    &  \hspace{8.5cm} (1+b_1)\sqrt{(K+1)G}\frac{5\beta^{4}\sigma}{18L\eta G} \sqrt{c\deltabyz}\tau T\\
    \le\; & \frac{\Delta}{16} + \frac{\Delta}{16} = \frac{\Delta}{8},
    \end{align*}
    because we choose $\beta$ such that
    \begin{align}\label{eq:step-size_bound_4}
    \beta &\le \min\left\{\left(\frac{9L\Delta\hbeta\eta\sqrt{G}}{16\sqrt{2}(1+b_1)\sigma\sqrt{T}\left(12\sqrt{L\Delta} + 3(B_{\rm init}-\tau+b)\right)}\right)^{1/4},\right.\notag\\
    &\hspace{9cm} \left. \left(\frac{9L\Delta\eta\sqrt{G}}{80(1+b_1)\sigma\sqrt{c\deltabyz}T^{3/2}\tau}\right)^{1/4}\right\},\\
    & \text{and} \quad K+1\le T.\notag
    \end{align}
    This implies 
    \begin{align*}
    &\Prob\left(\frac{4\gamma\beta^2}{G\hat{\beta}\eta}(1-\hat{\beta}\eta)^2\left\|\sum_{t=0}^K\sum_{i\in\cG}\<\zeta_{2,i}^t,\theta^{t+1}_i>\right\| \ge \frac{\Delta}{8}\right) \le \frac{\alpha}{20(T+1)},
    \end{align*}
    Note that both terms in the choice of $\beta$ have the dependency w.r.t. $T$ of order $\wtilde\cO(\nicefrac{1}{T^{3/8}})$ since $\hat{\beta} \sim \frac{1}{T}.$

    \subparagraph{Bound of the term $\circledFive$.} The bound in this case is similar to the previous one. 
    Let 
    \[
    \sigma_5^2 \eqdef \frac{256\gamma^2\beta^4}{G^2\hat{\beta}^4\eta^4}\cdot \left(\sqrt{4L\Delta} + \frac{3}{2}(B_{\rm init}-\tau)+\frac{3}{2}b + \hat{\beta}a + 4\hbeta\sqrt{2c\deltabyz}\ha + \frac{3}{2}\sqrt{2c\deltabyz}\hbeta\tau T\right)^2\cdot \sigma^2.
    \]
    Then we have 
    \begin{align*}
    &\E{\exp\left(\left|\frac{1}{\sigma_5^2}\frac{256\gamma^2\beta^4}{G^2\hat{\beta}^4\eta^4}(1-\beta)^2\<\zeta^l_{2,i},\theta^{l+1}_i>^2\right|\right)\mid l,i-1} \\
    &\le 
    \E{\exp\left(\frac{1}{\sigma^2_5}\frac{256\gamma^2\beta^4}{G^2\hat{\beta}^4\eta^4}\|\zeta_{2,i}^l\|^2 \cdot \|\theta^{l+1}_i\|^2\right)\mid l,i-1}\\
    &\le \EE\left[\exp\left(\frac{1}{\sigma_5^2}\frac{256\gamma^2 \beta^4}{G^2\hat{\beta}^4\eta^4}\cdot \left(\sqrt{4L\Delta} + \frac{3}{2}(B_{\rm init}-\tau + b)+\hat{\beta}a + 4\hbeta\sqrt{2c\deltabyz}\ha\; + \right.\right.\right.\\
    & \hspace{8.5cm}\left.\left.\left.\frac{3}{2}\sqrt{2c\deltabyz}\hbeta\tau T\right)^2\cdot \|\theta^{l+1}_i\|^2\right) \mid l,i-1\right]\\
    &= \mathbb{E}\left[\exp\left( \left[\frac{256\gamma^2\beta^4}{L^2G^2\hat{\beta}^4\eta^4}\cdot \left(\sqrt{4L\Delta} + \frac{3}{2}(B_{\rm init}-\tau + b) +\hat{\beta}a + 4\hbeta\sqrt{2c\deltabyz}\ha + \frac{3}{2}\sqrt{2c\deltabyz}\hbeta\tau T\right)^2\cdot \sigma^2\right]^{-1}\right.\right.\\
    &\qquad \left.\left. \frac{256\gamma^2\beta^4}{G^2\hat{\beta}^4\eta^4}\cdot \left(\sqrt{4L\Delta} + \frac{3}{2}(B_{\rm init}-\tau)+\frac{3}{2}b+\hat{\beta}a + 4\hbeta\sqrt{2c\deltabyz}\ha + \frac{3}{2}\sqrt{2c\deltabyz}\hbeta\tau T\right)^2\cdot \|\theta^{l+1}_i\|^2\right) \mid l,i-1\right]\\
    &= \E{\exp\left(\frac{\|\theta^{l+1}_i\|^2}{\sigma^2}\right) \mid l,i-1} \le \exp(1).
    \end{align*}
    Therefore, we have by \Cref{lem:concentration_lemma} that 
    \begin{align*}
    &\Prob\left[\frac{16\gamma\beta^2}{G\hat{\beta}^2\eta^2}(1-\beta)\left\|\sum_{t=0}^K\sum_{i\in\cG}\<\zeta_{2,i}^t,\theta^{t+1}_i>\right\|\right.\\ 
    &\ge \left. (\sqrt{2}+\sqrt{2}b_1)\sqrt{\sum_{t=0}^K\sum_{i\in\cG} \frac{256\gamma^2\beta^4\sigma^2}{G^2\hat{\beta}^4\eta^4}\left(\sqrt{4L\Delta} + \frac{3}{2}(B_{\rm init}-\tau+b)+\hat{\beta}a + 4\hbeta\sqrt{2c\deltabyz}\ha + \frac{3}{2}\sqrt{2c\deltabyz}\hbeta\tau T\right)^2}\right]\\ 
    &\le \exp(-\nicefrac{b_1^2}{3}) = \frac{\alpha}{20(T+1)}.
    \end{align*}
    Using the restrictions $12L\gamma \le \beta$ and $\hat{\beta} \le\left\{\frac{\sqrt{L\Delta}}{a}, \frac{\sqrt{L\Delta}}{4\sqrt{2c\deltabyz}\ha}\right\}$ we get 
    \begin{align*}
    &\sqrt{2}(1+b_1)\sqrt{(K+1)G}\frac{16\gamma\beta^2\sigma}{G\hat{\beta}^2\eta^2}\left(\sqrt{4L\Delta} + \frac{3}{2}(B_{\rm init}-\tau)+\frac{3}{2}b+\hat{\beta}a + 4\hbeta\sqrt{2c\deltabyz}\ha + \frac{3}{2}\sqrt{2c\deltabyz}\hbeta\tau T\right)\\
    \le\; &\sqrt{2}(1+b_1)\sqrt{(K+1)G}\frac{4\beta^{3}\sigma}{3LG\hat{\beta}^2\eta^2}\left(\sqrt{4L\Delta} + \frac{3}{2}(B_{\rm init}-\tau)+\frac{3}{2}b + \sqrt{L\Delta} + \sqrt{L\Delta} + \frac{3}{2}\sqrt{2c\deltabyz}\hbeta\tau T\right)\\
    =\; &\sqrt{2}(1+b_1)\sqrt{(K+1)G}\frac{4\beta^{3}\sigma}{3LG\hat{\beta}^2\eta^2}\left(4\sqrt{L\Delta} + \frac{3}{2}(B_{\rm init}-\tau + b)\right) + \\
    &\hspace{9cm}(1+b_1)\sqrt{(K+1)G}\frac{4\beta^{3}\sigma}{3LG\hat{\beta}\eta^2}\sqrt{c\deltabyz}\tau T\\
    \le\; &\frac{\Delta}{16} + \frac{\Delta}{16} = \frac{\Delta}{8}
    \end{align*}
    because we choose $\beta$ such that 
    \begin{align}\label{eq:step-size_bound_5}
        \beta &\le \min \left\{\left(\frac{3L\Delta\hat{\beta}^2\eta^2\sqrt{G}}{64\sqrt{2}(1+b_1)\sigma\sqrt{T}\left(4\sqrt{L\Delta} + \frac{3}{2}(B_{\rm init}-\tau+b)\right)}\right)^{1/3},\left(\frac{3L\Delta\hbeta\eta^2\sqrt{G}}{32(1+b_1)\tau T^{3/2}\sqrt{c\deltabyz}}\right)^{1/3}  \right\},\\
        & \text{and } K+1 \le T.\notag
    \end{align}
    This implies 
    \begin{align*}
    &\Prob\left(\frac{16\gamma\beta^2}{G\hat{\beta}^2\eta^2}(1-\hat{\beta}\beta)\left\|\sum_{t=0}^K\sum_{i\in\cG}\<\zeta_{2,i}^t,\theta^{t+1}_i>\right\| \ge \frac{\Delta}{8}\right) \le \frac{\alpha}{20(T+1)}.
    \end{align*}
    Note that both terms in the choice of $\beta$ have the dependency w.r.t. $T$ of order $\wtilde\cO(\nicefrac{1}{T^{5/6}})$ since $\hat{\beta} \sim \frac{1}{T}.$
    
     \subparagraph{Bound of the term $\circledSeven$.} The bound in this case is similar to the previous one. Let
     \[
        \sigma_7^2 \eqdef \frac{256L^2\gamma^4\beta^{4}}{G^2\hat{\beta}^4\eta^4} \left(\sqrt{64L\Delta} +3(B_{\rm init}-\tau+b)+3\hat{\beta}a + 3\sqrt{2c\deltabyz}\hbeta\ha + 5\sqrt{2c\deltabyz}\hbeta\tau T\right)^2\cdot \sigma^2.
     \]
     Then we have 
    \begin{align*}
    &\E{\exp\left(\left|\frac{1}{\sigma_7^2}\frac{256L^2\gamma^4\beta^4}{G^2\hat{\beta}^4\eta^4}(1-\beta)^2\<\zeta^l_{3,i},\theta^{l+1}_i>^2\right|\right) \mid l,i-1}\\
    &\le 
    \E{\exp\left(\frac{1}{\sigma_7^2}\frac{256\gamma^2\beta^4}{G^2\hat{\beta}^4\eta^4}\|\zeta_{3,i}^l\|^2 \cdot \|\theta^{l+1}_i\|^2\right)\mid l,i-1}\\
    &\le \EE\left[\exp\left(\frac{256\gamma^2\beta^4}{G^2\hat{\beta}^4\eta^4}\cdot L^2\gamma^2\left(\sqrt{64L\Delta} +3(B_{\rm init}-\tau+b)+3\hat{\beta}a + 3\sqrt{2c\deltabyz}\ha\hbeta \;+ \right.\right.\right.\\
    &\hspace{9cm}\left. \left.\left.4\sqrt{2c\deltabyz}\hbeta\tau T\right)^2\cdot \|\theta^{l+1}_i\|^2\right)\mid l,i-1\right]\\
    &\le \mathbb{E}\left[\exp\left(\left[\frac{256L^2\gamma^4\beta^{4}}{G^2\hat{\beta}^4\eta^4}\left(\sqrt{64L\Delta} +3(B_{\rm init}-\tau+b)+3\hat{\beta}a + 3\sqrt{2c\deltabyz}\ha\hbeta + 4\sqrt{2c\deltabyz}\hbeta\tau T\right)^2\cdot \sigma^2\right]^{-1}\times \right.\right. \\
    & \quad \left.\left.\frac{256L^2\gamma^4\beta^{4}}{G^2\hat{\beta}^4\eta^4}\left(\sqrt{64L\Delta} +3(B_{\rm init}-\tau+b)+3\hat{\beta}a + 3\sqrt{2c\deltabyz}\ha\hbeta + 4\sqrt{2c\deltabyz}\hbeta\tau T\right)^2\cdot \|\theta^{l+1}_i\|^2 \right) \mid l,i-1\right]\\
    &= \E{\exp\left(\frac{\|\theta^{l+1}_i\|^2}{\sigma^2}\right) \mid l,i-1} \le \exp(1).
    \end{align*}
    Therefore, we have by \Cref{lem:concentration_lemma} that 
    \begin{align*}
    &\Prob\left[\frac{16\gamma\beta^2}{G\hat{\beta}^2\eta^2}(1-\beta)\left\|\sum_{t=0}^K\sum_{i\in\cG}\<\zeta_{3,i}^t,\theta^{t+1}_i>\right\| \ge\right.\\ 
    & \left.\sqrt{2}(1+b_1)\sqrt{\sum_{t=0}^K\sum_{i\in\cG} \frac{256L^2\gamma^4\beta^{4}\sigma^2}{G^2\hat{\beta}^4\eta^4}\cdot \left(\sqrt{64L\Delta} + 3(B_{\rm init}-\tau + b) + 3\hat{\beta}a + 3\sqrt{2c\deltabyz}\ha\hbeta + 4\sqrt{2c\deltabyz}\hbeta\tau T\right)^2}\right]\\ 
    &\le \exp(-\nicefrac{b_1^2}{3}) = \frac{\alpha}{20(T+1)}.
    \end{align*}
    Using the restrictions $12L\gamma\le\beta$ and $\hat{\beta} \le \min\left\{\frac{\sqrt{L\Delta}}{a}, \frac{\sqrt{L\Delta}}{4\sqrt{2c\deltabyz}\ha}\right\}$ we get
    \begin{align*}
    &\sqrt{2}(1+b_1)\sqrt{(K+1)G}\frac{16L\gamma^2\beta^2\sigma}{\hat{\beta}^2\eta^2G}\left(\sqrt{64L\Delta} + 3(B_{\rm init}-\tau + b)+3\hat{\beta}a + 3\sqrt{2c\deltabyz}\ha\hbeta + 4\sqrt{2c\deltabyz}\hbeta\tau T\right)\\
    \le &\sqrt{2}(1+b_1)\sqrt{(K+1)G}\frac{\beta^{4}\sigma}{9L\hat{\beta}^2\eta^2G}\left(8\sqrt{L\Delta} + 3(B_{\rm init}-\tau +b)+ 3\sqrt{L\Delta} + \frac{3}{4}\sqrt{L\Delta} + 4\sqrt{2c\deltabyz}\hbeta\tau T\right) +\\
    &\hspace{8cm} \sqrt{2}(1+b_1)\sqrt{(K+1)G}\frac{\beta^{4}\sigma}{9L\hat{\beta}^2\eta^2G}\left( 4\sqrt{2c\deltabyz}\hbeta\tau T\right)\\
    &= \sqrt{2}(1+b_1)\sqrt{(K+1)G}\frac{\beta^{4}\sigma}{9L\hat{\beta}^2\eta^2G}\left(12\sqrt{L\Delta} + 3(B_{\rm init}-\tau +b)\right)+ \\
    &\hspace{8cm}+ \sqrt{2}(1+b_1)\sqrt{(K+1)G}\frac{4\beta^{4}\sigma}{9L\hbeta\eta^2G}\sqrt{2c\deltabyz}\tau T\\
    \le\;& \frac{\Delta}{16} + \frac{\Delta}{16} = \frac{\Delta}{8}   
    \end{align*}
    because we choose
    \begin{align}\label{eq:step-size_bound_7}
    \beta &\le \min\left\{\left(\frac{9L\Delta\hat{\beta}^2\eta^2\sqrt{G}}{16\sqrt{2}(1+b_1)\sigma\sqrt{T}\left(12\sqrt{L\Delta} + 3(B_{\rm init}-\tau+b)\right)}\right)^{1/4},\left(\frac{9L\Delta\hbeta\eta^2\sqrt{G}}{128(1+b_1)\sqrt{c\deltabyz}\tau\sigma T^{3/2}}\right)^{1/4}\right\},\\
    &\text{and} \quad K+1 \le T.\notag
    \end{align}
    This implies 
    \begin{align*}
    &\Prob\left(\frac{8\gamma\beta^2}{G\eta^2}(1-\beta)\left\|\sum_{t=0}^K\sum_{i\in\cG}\<\zeta_{3,i}^t,\theta^{t+1}_i>\right\| \ge \frac{\Delta}{8}\right) \le \frac{\alpha}{20(T+1)}.
    \end{align*}
    Note that both terms in the choice of $\beta$ have the dependency w.r.t. $T$ of order $\wtilde\cO(\nicefrac{1}{T^{5/8}})$ since $\hat{\beta} \sim \frac{1}{T}.$

     \subparagraph{Bound of the term $\circledSix$.} The bound in this case is similar to the previous one. Let
     \[
        \sigma_6^2 \eqdef \frac{16\gamma^2}{G^2}\left(\sqrt{4L\Delta} +\frac{3}{2}(B_{\rm init}-\tau + b) + \hat{\beta}a + 4\hbeta\sqrt{2c\deltabyz}\ha + \frac{3}{2}\sqrt{2c\deltabyz}\hbeta\tau T\right)^2\cdot \sigma^2.
     \]
      
     Then we have 
    \begin{align*}
    &\E{\exp\left(\left|\frac{1}{\sigma_6^2}\frac{16\gamma^2}{G^2}(1-\beta)^2\<\zeta^l_{4},\theta^{l+1}_i>^2\right|\right)\mid l,i-1 }\\
    &\le 
    \E{\exp\left(\frac{1}{\sigma_6^2}\frac{16\gamma^2}{G^2}\|\zeta_{4}^l\|^2 \cdot \|\theta^{l+1}_i\|^2\right)\mid l,i-1}\\
    &\le \EE\left[\exp\left(\frac{1}{\sigma^2_6}\frac{16\gamma^2}{G^2}\left(\sqrt{4L\Delta} +\frac{3}{2}(B_{\rm init}-\tau + b) + \hat{\beta}a + 4\hbeta\sqrt{2c\deltabyz}\ha \;+ \right.\right.\right.\\
    &\hspace{9cm} \left.\left.\left. \frac{3}{2}\sqrt{2c\deltabyz}\hbeta\tau T)\right)^2\cdot \|\theta^{l+1}_i\|^2\right)\mid l,i-1\right]\\
    &\le \mathbb{E}\left[\exp\left(\left[\frac{16\gamma^2}{G^2}\left(\sqrt{4L\Delta} + \frac{3}{2}(B_{\rm init}-\tau + b) + \hat{\beta}a + 4\hbeta\sqrt{2c\deltabyz}\ha + \frac{3}{2}\sqrt{2c\deltabyz}\hbeta\tau T)\right)^2\cdot \sigma^2\right]^{-1} \right.\right.\\
    &\qquad\left.\left. \frac{16\gamma^2}{G^2}\left(\sqrt{4L\Delta} +\frac{3}{2}(B_{\rm init}-\tau) + \frac{3}{2}b + \hat{\beta}a + 4\hbeta\sqrt{2c\deltabyz}\ha + \frac{3}{2}\sqrt{2c\deltabyz}\hbeta\tau T)\right)^2\cdot\|\theta^{l+1}_i\|^2\right)\mid l,i-1\right]\\
    &= \E{\exp\left(\frac{\|\theta^{t+1}_i\|^2}{\sigma^2}\right)\mid l,i-1} \le \exp(1).
    \end{align*}
    Therefore, we have by \Cref{lem:concentration_lemma} that 
    \begin{align*}
    &\Prob\left[\frac{4\gamma(1-\beta)}{G}\left\|\sum_{t=0}^K\sum_{i\in\cG}\<\zeta_{4,i}^t,\theta^{t+1}_i>\right\|\right.\\ 
    &\ge \left. \sqrt{2}(1+b_1)\sqrt{\sum_{t=0}^K\sum_{i\in\cG} \frac{16\gamma^2}{G^2}\sigma^2\cdot \left(\sqrt{4L\Delta} +\frac{3}{2}(B_{\rm init}-\tau+b) + \hat{\beta} a + 4\hbeta\sqrt{2c\deltabyz}\ha + \frac{3}{2}\sqrt{2c\deltabyz}\hbeta\tau T)\right)^2}\right]\\ 
    &\le \exp(-\nicefrac{b_1^2}{3}) = \frac{\alpha}{20(T+1)},
    \end{align*}
    
    Using the restrictions $12L\gamma\le\beta$ and $\hat{\beta} \le \min\left\{\frac{\sqrt{L\Delta}}{a}, \frac{\sqrt{L\Delta}}{4\sqrt{2c\deltabyz}\ha} \right\}$ we get
    \begin{align*}
    &(\sqrt{2}+\sqrt{2}b_1)\sqrt{(K+1)G} \cdot \frac{4\gamma}{G}\sigma\left(\sqrt{4L\Delta} +\frac{3}{2}(B_{\rm init}-\tau+b) + \hat{\beta}a + 4\hbeta\sqrt{2c\deltabyz}\ha + \frac{3}{2}\sqrt{2c\deltabyz}\hbeta\tau T\right)\\
    \le\; & (\sqrt{2}+\sqrt{2}b_1)\sqrt{(K+1)G} \cdot \frac{\beta}{3LG}\sigma\left(\sqrt{4L\Delta} +\frac{3}{2}(B_{\rm init}-\tau+b) + \sqrt{L\Delta} + \sqrt{L\Delta} + \frac{3}{2}\sqrt{2c\deltabyz}\hbeta\tau T\right)\\
    =\; & (\sqrt{2}+\sqrt{2}b_1)\sqrt{(K+1)G} \cdot \frac{\beta}{3LG}\sigma\left(4\sqrt{L\Delta} +\frac{3}{2}(B_{\rm init}-\tau+b)\right) + \\
     & \hspace{7cm} (\sqrt{2}+\sqrt{2}b_1)\sqrt{(K+1)G} \cdot \frac{\beta}{LG}\sigma \sqrt{2c\deltabyz}\hbeta\tau T\\
     \le\; & \sqrt{2}(1+b_1)\sqrt{(K+1)} \cdot \frac{\beta}{3L\sqrt{G}}\sigma\left(4\sqrt{L\Delta} +\frac{3}{2}(B_{\rm init}-\tau+b)\right) + \\
     & \hspace{7cm} (1+b_1)\sqrt{(K+1)} \cdot \frac{4}{29\sqrt{c\deltabyz}TL\sqrt{G}}\sigma \sqrt{c\deltabyz}\hbeta\tau T\\
     \le\; & \sqrt{2}(1+b_1)\sqrt{(K+1)} \cdot \frac{\beta}{3L\sqrt{G}}\sigma\left(4\sqrt{L\Delta} +\frac{3}{2}(B_{\rm init}-\tau+b)\right) + (1+b_1)\sqrt{(K+1)} \frac{4}{29L\sqrt{G}}\sigma \hbeta\tau \\
    \le \;&\frac{\Delta}{16} + \frac{\Delta}{16} = \frac{\Delta}{8},
    \end{align*}
    where we use the restriction $\beta \le \frac{2}{29\sqrt{c\deltabyz}T}$ and choose $\beta$ and $\hbeta$ such that
    \begin{align}\label{eq:step-size_bound_6}
    \beta &\le \left(\frac{3L\Delta\sqrt{G}}{16\sqrt{2}(1+b_1)\sigma\sqrt{T}\left(4\sqrt{L\Delta} +\frac{3}{2}(B_{\rm init}-\tau+b)\right)}\right), \quad \hbeta \le \left(\frac{29L\Delta\sqrt{G}}{64(1+b_1)\sigma\sqrt{T}}\right),\\
    &\quad \text{and} \quad K+1\le T.\notag
    \end{align}
    This implies 
    \begin{align*}
    &\Prob\left(\frac{4\gamma(1-\beta)}{G}\left\|\sum_{t=0}^K\sum_{i\in\cG}\<\zeta_{4,i}^t,\theta^{t+1}_i>\right\| \ge \frac{\Delta}{8}\right) \le \frac{\alpha}{20(T+1)}.
    \end{align*}
    Note that the worst dependency in the choice of $\beta$ and $\hbeta$ w.r.t. $T$ is $\wtilde{\cO}(\nicefrac{1}{T^{1/2}})$.


    \subparagraph{Bound of the term $\circledEight$.} The bound in this case is similar to the previous one. Let
    \[
        \sigma^2_8 \eqdef \frac{16L^2\gamma^4}{G^2}\cdot \left(\sqrt{64L\Delta} +3(B_{\rm init}-\tau+b)+3\hat{\beta}a + 3\sqrt{2c\deltabyz}\hbeta\ha + 5\sqrt{2c\deltabyz}\hbeta\tau T\right)^2\cdot \sigma^2.
    \]
    Then we have 
    \begin{align*}
    &\E{\exp\left(\left|\frac{1}{\sigma^2_8}\frac{16\gamma^2}{G^2}(1-\beta)^2\<\zeta^l_{5},\theta^{l+1}_i>^2\right|\right)\mid l,i-1}\\
    &\le 
    \E{\exp\left(\frac{1}{\sigma^2_8}\frac{16\gamma^2}{G^2}\|\zeta_{5}^l\|^2 \cdot \|\theta^{l+1}_i\|^2\right)\mid l,i-1}\\
    &\le \EE\left[\exp\left(\frac{1}{\sigma^2_8}\frac{16\gamma^2}{G^2}L^2\gamma^2\left(\sqrt{64L\Delta} +3(B_{\rm init}-\tau+b)+3\hbeta a + 3\sqrt{2c\deltabyz}\hbeta\ha \;+ \right.\right.\right.\\
    &\hspace{10cm} \left.\left.\left. 5\sqrt{2c\deltabyz}\hbeta\tau T\right) \cdot \|\theta^{l+1}_i\|^2\right)^2\mid l,i-1\right].
    \end{align*}
    Since $\theta^{l+1}_i$ is sub-Gaussian with parameter $\sigma^2$, then we can continue the chain of inequalities above using the definition of $\sigma_8^2$
    \begin{align*}
        &\mathbb{E}\left[\exp\left(\left[\frac{16L^2\gamma^4}{G^2}\cdot \left(\sqrt{64L\Delta} +3(B_{\rm init}-\tau+b)+3\hat{\beta}a + 3\sqrt{2c\deltabyz}\hbeta\ha + 5\sqrt{2c\deltabyz}\hbeta\tau T\right)^2\cdot \sigma^2\right]^{-1}\cdot   \right.\right.\\
        &\qquad \left.\left. \frac{16L^2\gamma^4}{G^2}\cdot \left(\sqrt{64L\Delta} +3(B_{\rm init}-\tau+b)+3\hat{\beta}a + 3\sqrt{2c\deltabyz}\hbeta\ha + 5\sqrt{2c\deltabyz}\hbeta\tau T\right)^2\cdot\|\theta^{l+1}_i\|^2\right)\mid l,i-1\right]\\
        &= \E{\exp\left(\frac{\|\theta^{l+1}_i\|^2}{\sigma^2}\right)} \le \exp(1).
    \end{align*}
    Therefore, we have by \Cref{lem:concentration_lemma} that 
    \begin{align*}
    &\Pr\left[\frac{4\gamma(1-\beta)}{G}\left\|\sum_{t=0}^K\sum_{i\in\cG}\<\zeta_{5,i}^t,\theta^{t+1}>\right\|\right.\\ 
    &\ge \left. (\sqrt{2}+\sqrt{2}b_1)\sqrt{\sum_{t=0}^K\sum_{i\in\cG} \frac{16L^2\gamma^4}{G^2}\sigma^2 \left(\sqrt{64L\Delta} +3(B_{\rm init}-\tau+b)+3\hat{\beta}a + 3\sqrt{2c\deltabyz}\hbeta\ha + 5\sqrt{2c\deltabyz}\hbeta\tau T\right)^2}\right]\\ 
    &\le \exp(-\nicefrac{b_1^2}{3}) = \frac{\alpha}{20(T+1)}.
    \end{align*}
    Using the restrictions $12L\gamma\le \beta$ and $\hat{\beta}\le \min\left\{\frac{\sqrt{L\Delta}}{a}, \frac{\sqrt{L\Delta}}{4\sqrt{2c\deltabyz}\ha}\right\}$ we get
    \begin{align*}
    &(\sqrt{2}+\sqrt{2}b_1)\sqrt{(K+1)G} \cdot \frac{4L\gamma^2}{G}\sigma\left(\sqrt{64L\Delta} +3(B_{\rm init}-\tau+b)+3\hat{\beta}a + 3\sqrt{2c\deltabyz}\hbeta\ha + 5\sqrt{2c\deltabyz}\hbeta\tau T\right)\\
    \le\; & (\sqrt{2}+\sqrt{2}b_1)\sqrt{(K+1)G} \cdot \frac{\beta^2\sigma}{36LG}\left(8\sqrt{L\Delta} +3(B_{\rm init}-\tau)+3b + 3\sqrt{L\Delta} + \frac{3}{4}\sqrt{L\Delta} + 5\sqrt{2c\deltabyz}\hbeta\tau T\right)\\
    \le\; & \sqrt{2}(1+b_1)\sqrt{(K+1)G} \cdot \frac{\beta^2\sigma}{36LG}\left(12\sqrt{L\Delta} +3(B_{\rm init}-\tau+b) \right) + \\
    &\hspace{8cm}(1+b_1)\sqrt{(K+1)G} \cdot \frac{5\beta^2\sigma}{18LG}\sqrt{c\deltabyz}\hbeta\tau T\\
    \le\; & (\sqrt{2}+\sqrt{2}b_1)\sqrt{(K+1)G} \cdot \frac{\beta^2\sigma}{36LG}\left(8\sqrt{L\Delta} +3(B_{\rm init}-\tau)+3b + 3\sqrt{L\Delta} + \frac{3}{4}\sqrt{L\Delta} + 5\sqrt{2c\deltabyz}\hbeta\tau T\right)\\
    \le\; & \sqrt{2}(1+b_1)\sqrt{(K+1)G} \cdot \frac{\beta^2\sigma}{36LG}\left(12\sqrt{L\Delta} +3(B_{\rm init}-\tau+b) \right) + \\
    &\hspace{8cm}(1+b_1)\sqrt{(K+1)G} \cdot \frac{64\cdot 5\tau^2\sigma\hbeta^2}{431^2\cdot 18L^2\Delta G}\sqrt{c\deltabyz}\hbeta\tau T\\
    \le\; &\frac{\Delta}{16} + \frac{\Delta}{16} = \frac{\Delta}{8}
    \end{align*}
    because we choose $\beta$ and $\hbeta$ such that
    \begin{align}\label{eq:step-size_bound_8}
    \beta &\le \left(\frac{9L\Delta\sqrt{G}}{4\sqrt{2}(1+b_1)\sigma\sqrt{T}\left(12\sqrt{L\Delta} +3(B_{\rm init}-\tau+b) \right)}\right)^{1/2}, \quad \hbeta \le \left(\frac{462\cdot (L\Delta)^2\sqrt{G}}{(1+b_1)\sigma T^{3/2}\tau^3}\right)^{1/2}\\
    & \text{and} \quad K+1\le T.\notag
    \end{align}
    This implies 
    \begin{align*}
    &\Prob\left(4\gamma(1-\beta)\left\|\sum_{t=0}^K\sum_{i\in\cG}\<\zeta_{5,i}^t,\theta^{t+1}>\right\| \ge \frac{\Delta}{8}\right) \le \frac{\alpha}{20(T+1)}.
    \end{align*}
    Note that the worst dependency in the choice of $\beta$ and $\hbeta$ w.r.t $T$ is $\wtilde{\cO}(\nicefrac{1}{T^{3/4}}).$

    \subparagraph{Bound of the term $\circledNine$.} We need to choose hyperparameters such that 
    \[
    \frac{\tau\hbeta^3 a^2 K}{324L\sqrt{L\Delta}} \le \frac{\Delta}{24}.
    \]

    This is possible by choosing $\hbeta$ such that 
    \begin{align}\label{eq:step-size_bound_9}
        \hbeta \le \left(\frac{27(L\Delta)^{3/2}}{2a^2T\tau}\right)^{1/3} \quad \text{and} \quad K \le T+1.
    \end{align}
    Note that since $a \sim T$ we obtain $\hbeta \sim \frac{1}{T}$.

    \subparagraph{Bound of the term $\circledTen$.} We need to choose hyperparameters such that 
    \[
        \frac{2\tau\hbeta^3c\deltabyz\ha^2 K}{81L\sqrt{L\Delta}} + \frac{2\tau^3\hbeta^3c\deltabyz T^2K}{81L\sqrt{L\Delta}}\le \frac{\Delta}{24}.
    \]
    By choosing $\hbeta$ such that
    \begin{align}\label{eq:step-size_bound_10}
        &\hbeta\le \min\left\{\left(\frac{81(L\Delta)^{3/2}}{96c\deltabyz\ha^2\tau T}\right)^{1/3}, \left(\frac{81(L\Delta)^{3/2}}{96c\deltabyz\tau^3T^3}\right)^{1/3}\right\}, \quad \text{ and }\quad  K \le T+1
    \end{align}
    we have
    \[
        \frac{2\tau\hbeta^3c\deltabyz\ha^2 K}{81L\sqrt{L\Delta}} + \frac{2\tau^3\hbeta^3c\deltabyz T^2K}{81L\sqrt{L\Delta}} \le \frac{\Delta}{48}+ \frac{\Delta}{48} = \frac{\Delta}{24}.
    \]

    \paragraph{Final probability.}
    Therefore, the probability event 
    \[
    \Omega \eqdef E^K 
    \cap \overline{\Theta}^{K+1}
    \cap \left(\cap_{i\in\cG}\overline{\Theta}^{K+1}_i\right)
    \cap \overline{N}^{K+1}
    \cap \left(\cap_{i\in\cG}\overline{N}^{K+1}_i\right)
    \cap E_{\circledOne}
    \cap E_{\circledTwo}
    \cap E_{\circledThree}
    \cap E_{\circledFour}
    \cap E_{\circledFive}
    \cap E_{\circledSix}
    \cap E_{\circledSeven}
    \cap E_{\circledEight}
    \cap E_{\circledNine}
    \cap E_{\circledTen},
    \]
    where each $E_{\circledOne}$-$E_{\circledTen}$ denotes that each of $2$-$8$-th terms is smaller than $\frac{\Delta}{8}$ and each of the terms $1,9,10$ is smaller than $\frac{\Delta}{24}$. This implies that 
    \[
    \circledOne + \circledTwo + \circledThree + \circledFour + \circledFive + \circledSix + \circledSeven + \circledEight + \circledNine + \circledTen \le 7\cdot \frac{\Delta}{8} + 3 \cdot \frac{\Delta}{24} = \Delta,
    \]
    i.e., condition $7$ in the induction assumption holds. Moreover, this also implies that 
    \[
    \Phi^{K+1} \le \Phi^0 + \Delta \le \Delta + \Delta = 2\Delta,
    \]
    i.e., condition $6$ in the induction assumption holds. The probability $\Prob(E_{K+1})$ can be lower bounded as follows 
    \begin{align*}
    \Prob(E_{K+1}) &\ge \Prob(\Omega)\\
    &= \Prob\bigg(E_K 
    \cap \overline{\Theta}^{K+1}
    \cap \left(\cap_{i\in\cG}\overline{\Theta}^{K+1}_i\right)
    \cap \overline{N}^{K+1}
    \cap \left(\cap_{i\in\cG}\overline{N}^{K+1}_i\right)
    \cap E_{\circledOne}
    \cap E_{\circledTwo}
    \cap E_{\circledThree}
    \cap E_{\circledFour}
    \cap E_{\circledFive}
    \cap E_{\circledSix}\\
    &\qquad 
    \cap E_{\circledSeven}
    \cap E_{\circledEight} \cap E_{\circledNine}
    \cap E_{\circledTen}\bigg)\\
    &= 1 - \Prob\bigg(\overline{E}_K\cup
    \Theta^{K+1} \cup 
    \left(\cup_{i\in\cG}\Theta^{K+1}_i\right)
    \cup N^{K+1}
    \left(\cup_{i\in\cG}N^{K+1}_i\right)
    \cup \overline{E}_{\circledOne}
    \cup \overline{E}_{\circledTwo}
    \cup \overline{E}_{\circledThree}
    \cup \overline{E}_{\circledFour}
    \cup \overline{E}_{\circledFive}
    \cup \overline{E}_{\circledSix}\\
    &\qquad \cup \overline{E}_{\circledSeven}
    \cup \overline{E}_{\circledEight}  \cup \overline{E}_{\circledNine}  \cup \overline{E}_{\circledTen}\bigg)\\
    &\ge 1 - \Prob(\overline{E}_K) 
    - \Prob(\Theta^{K+1})
    - \sum_{i\in\cG}\Prob(\Theta^{K+1}_i)
    -\Prob(N^{K+1})
    - \sum_{i\in\cG}\Prob(N^{K+1}_i)
    - \Prob(\overline{E}_{\circledOne})
    - \Prob(\overline{E}_{\circledTwo})\\
    & \qquad 
    - \Prob(\overline{E}_{\circledThree})
    - \Prob(\overline{E}_{\circledFour})
    - \Prob(\overline{E}_{\circledFive})
    - \Prob(\overline{E}_{\circledSix})
    - \Prob(\overline{E}_{\circledSeven})
    - \Prob(\overline{E}_{\circledEight})
    - \Prob(\overline{E}_{\circledNine})
    - \Prob(\overline{E}_{\circledTen})\\
    &\ge 1 - \frac{\alpha(K+1)}{T+1}
    - \frac{\alpha}{8(T+1)}
    - \sum_{i\in\cG} \frac{\alpha}{8G(T+1)}
    - \frac{\alpha}{8(T+1)}
    - 0- 10\cdot \frac{\alpha}{20(T+1)}\\
    &= 1-\frac{\alpha(K+2)}{T+1}.
    \end{align*}
    
    This finalizes the transition step of induction. The result of the theorem follows by setting $K=T-1$. Indeed, from (\ref{eq:njqnsfqknj}) we obtain
    \begin{align}\label{eq:mfweodjnwej}
        \frac{\gamma}{2}\sum_{t=0}^K\|\nabla f(x^t)\|^2 
        \le \Phi^0 
        - \Phi^{K+1} 
        + \Delta 
        \le 2\Delta 
        \Rightarrow \frac{1}{T}\sum_{t=0}^{T-1}\|\nabla f(x^t)\|^2 
        \le \frac{4\Delta}{\gamma T}.
    \end{align}

    \paragraph{Final rate.}

    We highlight that we are interested in the functional dependency of the rate on the problem constants. Therefore, in the rest of the proof we omit using numerical constants. Translating momentum restrictions (\ref{eq:step-size_bound_1}), (\ref{eq:step-size_bound_2}), (\ref{eq:step-size_bound_3}), (\ref{eq:step-size_bound_4}), (\ref{eq:step-size_bound_5}), (\ref{eq:step-size_bound_6}), (\ref{eq:step-size_bound_7}), and (\ref{eq:step-size_bound_8}) to the step-size restriction using $12L\gamma= \beta$ equality we get that the step-size should satisfy 
    \begin{align}\label{eq:ninwofejnasdadwe}
        \gamma = &\frac{1}{L}\wtilde{\cO}\left(\min\left\{
         \underbrace{\left(\frac{L\Delta G}{T\sigma^2}\right)^{1/2}, \left(\frac{L\Delta\hat{\beta}^2\eta^2}{T\sigma^2}\right)^{1/4} , \left(\frac{L\Delta\hat{\beta}\eta}{T\sigma^2}\right)^{1/3} }_{\text{from term } 1~\eqref{eq:step-size_bound_1}}, 
         \underbrace{\left(\frac{L\Delta\sqrt{G}\hat{\beta}\eta}{B_{\rm init}\sigma \sqrt{T}}\right)^{1/2} }_{\text{from term } 2~\eqref{eq:step-size_bound_2}},\right.\right.\notag\\
         &\left.\left.\underbrace{\left(\frac{L\Delta \sqrt{G}\hat{\beta}\eta}{\sigma(\sqrt{L\Delta} + B_{\rm init} + \sigma)\sqrt{T} }\right)^{\frac{1}{3}}, 
         \left(\frac{L\Delta \sqrt{G}\eta}{\sigma\sqrt{c\deltabyz}\tau T^{3/2}}\right)^{\frac{1}{3}} }_{\text{from term } 3~\eqref{eq:step-size_bound_3}}, \underbrace{\left(\frac{L\Delta\hat{\beta}\eta \sqrt{G}}{\sigma(\sqrt{L\Delta} + B_{\rm init} + \sigma)\sqrt{T}}\right)^{1/4}, \left(\frac{L\Delta\eta \sqrt{G}}{\sigma\sqrt{c\deltabyz}T^{3/2}\tau}\right)^{1/4}}_{\text{from term } 4~\eqref{eq:step-size_bound_4}},\right.\right.\notag\\
         &\left.\left.\underbrace{\left(\frac{L\Delta\hat{\beta}^2\eta^2\sqrt{G}}{\sigma(\sqrt{L\Delta} + B_{\rm init} + \sigma) \sqrt{T}}\right)^{1/3}, \left(\frac{L\Delta\hbeta\eta^2\sqrt{G}}{\sigma T^{3/2}\sqrt{c\deltabyz}\tau}\right)^{1/3}}_{\text{from term } 5~\eqref{eq:step-size_bound_5}},\right.\right.\notag\\
         &\left.\left.
        \underbrace{\left(\frac{L\Delta\hat{\beta}^2\eta^2\sqrt{G}}{\sigma(\sqrt{L\Delta} + B_{\rm init} + \sigma)\sqrt{T} }\right)^{1/4}, \left(\frac{L\Delta\hbeta\eta^2\sqrt{G}}{\sqrt{c\deltabyz}\sigma\tau T^{3/2}}\right)^{1/4}}_{\text{from term } 7~\eqref{eq:step-size_bound_7}},\right.\right.\notag\\
        &\left.\left.
        \underbrace{\left(\frac{L\Delta \sqrt{G}}{\sigma(\sqrt{L\Delta} + B_{\rm init} +\sigma) \sqrt{T}}\right)}_{\text{from term } 6~\eqref{eq:step-size_bound_6}},
        \underbrace{\left(\frac{L\Delta \sqrt{G}}{\sigma (\sqrt{L\Delta} + B_{\rm init} + \sigma)\sqrt{T}}\right)^{\frac{1}{2}}}_{\text{from term } 8~\eqref{eq:step-size_bound_8}}\right\}
        \right),
    \end{align} 
    and 
    \begin{align*}
        \hbeta \le \wtilde{\cO}\left(\underbrace{\left(\frac{L\Delta\sqrt{G}}{\sigma\sqrt{T}}\right)}_{\text{from term 6}~\eqref{eq:step-size_bound_6}}, \underbrace{\left(\frac{ (L\Delta)^2\sqrt{G}}{\sigma T^{3/2}\tau^3}\right)^{1/2}}_{\text{from term 8}~\eqref{eq:step-size_bound_8}}, 
        \underbrace{ \frac{\sqrt{L\Delta}}{a^{2/3}T^{1/3}\tau^{1/3}}}_{\text{from term 9}~\eqref{eq:step-size_bound_9}}, 
        \underbrace{\frac{\sqrt{L\Delta}}{(c\deltabyz)^{1/3}\ha^{2/3}\tau^{1/3}T^{1/3}}, \frac{\sqrt{L\Delta}}{(c\deltabyz)^{1/3} T\tau}}_{\text{from term 10}~\eqref{eq:step-size_bound_10}}\right).
    \end{align*}
    To guarantee the convergence to the neighborhood, we should choose $\hbeta \sim \frac{1}{a} \sim \frac{1}{T}$. Therefore, the first two restrictions on $\hbeta$ will not contribute to the size of the neighborhood. Moreover, the worst power of $T$ in \eqref{eq:ninwofejnasdadwe} comes from the term $\circledFive$ and equals $\gamma\sim \frac{1}{T^{5/6}}.$ The second worst comes from terms $\circledOne$ and $\circledTwo$, and equals to $\gamma \sim \frac{1}{T^{3/4}}$. These terms give the rate of the form
    \begin{align}\label{eq:ojnflqkmqowkdqw}
        & \frac{L\Delta}{T}\wtilde{\cO}\left(
        \left(\frac{T\sigma^2}{L\Delta \hbeta^2\eta^2}\right)^{1/4} 
        + \left(\frac{B_{\rm init}\sigma\sqrt{T}}{L\Delta\sqrt{G}\hbeta\eta}\right)^{1/2}
        + \left(\frac{\sigma(\sqrt{L\Delta}+B_{\rm init}+\sigma)\sqrt{T}}{L\Delta\sqrt{G}\hbeta^2\eta^2}\right)^{1/3}
        + \left(\frac{\sigma\sqrt{c\deltabyz}\tau T^{3/2}}{L\Delta\sqrt{G}\hbeta\eta}\right)^{1/3}
        \right).
    \end{align}

    Now we need to plug in all restrictions on $\hbeta$ in \eqref{eq:ojnflqkmqowkdqw}. We remind that in total we should satisfy (up to numerical constants) 
    \begin{align*}
    \hat{\beta} \le \sqrt{L\Delta}\cdot \wtilde{\cO}\left(\min\left\{\frac{1}{a}, 
    \frac{1}{\sqrt{c\deltabyz}\ha},
    \frac{1}{a^{2/3}T^{1/3}\tau^{1/3}},
    \frac{1}{(c\deltabyz)^{1/3}\ha^{2/3}\tau^{1/3}T^{1/3}},
    \frac{1}{(c\deltabyz)^{1/3}\tau T}\right\}\right).
    \end{align*}
    We can simplify this constraint to 
    \begin{align}\label{eq:jnoqneqoejnfqe}
    \hat{\beta} \le \sqrt{L\Delta}\cdot \wtilde{\cO}\left(\min\left\{\frac{1}{aC}, 
    \frac{1}{a^{2/3}T^{1/3}\tau^{1/3}}, 
    \frac{1}{(c\deltabyz)^{1/3}\ha^{2/3}\tau^{1/3}T^{1/3}},
    \frac{1}{(c\deltabyz)^{1/3}\tau T},
    \right\}\right),
    \end{align}
    where $C\eqdef (1+\sqrt{G}\sqrt{c\deltabyz})$. Let us do this one by one.
    \begin{enumerate}
        \item $\hbeta \sim \frac{\sqrt{L\Delta}}{aC}$, where $a \sim \wtilde{\cO}(\nicefrac{\sqrt{d}\sigma_{\omega}\sqrt{T}}{\sqrt{G}})$, gives the term in the rate of the form
        \begin{align*}
        & \frac{L\Delta}{T}\wtilde{\cO}\left(
        \left(\frac{T\sigma^2}{L\Delta \hbeta^2\eta^2}\right)^{1/4} 
        + \left(\frac{B_{\rm init}\sigma\sqrt{T}}{L\Delta\sqrt{G}\hbeta\eta}\right)^{1/2}
        + \left(\frac{\sigma(\sqrt{L\Delta}+B_{\rm init}+\sigma)\sqrt{T}}{L\Delta\sqrt{G}\hbeta^2\eta^2}\right)^{1/3}
        + \left(\frac{\sigma\sqrt{c\deltabyz}\tau T^{3/2}}{L\Delta\sqrt{G}\hbeta\eta}\right)^{1/3}
        \right) \notag\\
        =\;& \frac{L\Delta}{T}\wtilde{\cO}\left(
        \left(\frac{B_{\rm init}^2C^2a^2T\sigma^2}{(L\Delta)^2\tau^2}\right)^{1/4} 
        + \left(\frac{B^2_{\rm init}Ca\sigma T^{1/2}}{(L\Delta)^{3/2}\sqrt{G}\tau}\right)^{1/2}
        + \left(\frac{ B_{\rm init}^2C^2a^2\sigma(\sqrt{L\Delta}+B_{\rm init}+\sigma)\sqrt{T}}{(L\Delta)^2\sqrt{G}\tau^2}\right)^{1/3}
        \right.
        \notag\\
        &\hspace{9cm}\left. +\; 
         \left(\frac{B_{\rm init}Ca\sigma\sqrt{c\deltabyz} T^{3/2}}{(L\Delta)^{3/2}\sqrt{G}}\right)^{1/3}
        \right)\\
        =\;& \wtilde{\cO}\left(
        \left(\frac{(L\Delta)^2B_{\rm init}^2C^2a^2\sigma^2}{T^3\tau^2}\right)^{1/4} 
        + \left(\frac{\sqrt{L\Delta}B^2_{\rm init}Ca\sigma}{T^{1/2}\sqrt{G}\tau}\right)^{1/2}
        + \left(\frac{ L\Delta B_{\rm init}^2C^2a^2\sigma(\sqrt{L\Delta}+B_{\rm init}+\sigma)}{T^{5/2}\sqrt{G}\tau^2}\right)^{1/3}
        \right.
        \notag\\
        &\hspace{9cm}\left. +\; 
         \left(\frac{(L\Delta)^{3/2}B_{\rm init}Ca\sigma\sqrt{c\deltabyz}}{\sqrt{G} T^{3/2}}\right)^{1/3}
        \right).
    \end{align*}
    Now we use the specific form of $a \sim \nicefrac{\sqrt{d}\sigma_{\omega}\sqrt{T}}{\sqrt{G}}$ to obtain the rate
     \begin{align}\label{eq:onoqwndq}
         & \wtilde{\cO}\left(
        \left(\frac{(L\Delta)^2B_{\rm init}^2C^2d\sigma_{\omega}^2\sigma^2}{GT^2\tau^2}\right)^{1/4} 
        + \left(\frac{\sqrt{L\Delta}B_{\rm init}^2C\sqrt{d}\sigma_{\omega}\sigma}{G\tau}\right)^{1/2}
        \right.
        \notag\\
        &\left. +\; 
        \left(\frac{L\Delta B_{\rm init}^2C^2d\sigma_{\omega}^2\sigma(\sqrt{L\Delta}+B_{\rm init}+\sigma)}{(TG)^{3/2}\tau^2}\right)^{1/3}
        + \left(\frac{(L\Delta)^{3/2}B_{\rm init}C\sqrt{d}\sigma_{\omega}\sigma\sqrt{c\deltabyz}\tau }{TG\tau}\right)^{1/3}
        \right).
    \end{align}
    
    Since $\sigma_{\omega} \sim \sqrt{T}$, we omit the first two terms in \eqref{eq:onoqwndq} as they have better dependency on $T.$ The final rate in this regime is 
    \begin{align}
        &\wtilde{\cO}\left(
        \left(\frac{L\Delta B_{\rm init}^2C^2d\sigma_{\omega}^2\sigma(\sqrt{L\Delta}+B_{\rm init}+\sigma)}{(TG)^{3/2}\tau^2}\right)^{1/3}
        + \left(\frac{(L\Delta)^{3/2}B_{\rm init}C\sqrt{d}\sigma_{\omega}\sigma\sqrt{c\deltabyz}\tau }{TG\tau}\right)^{1/3}
        \right).
    \end{align}

    \item $\hbeta \sim \frac{\sqrt{L\Delta}}{C_1a^{2/3}T^{1/3}\tau^{1/3}}$, where $C_1 = 1+G^{1/3}(c\deltabyz)^{1/3}$ (we combine cases 2 and 3 in \eqref{eq:jnoqneqoejnfqe} together to simplify calculations). We obtain 
    \begin{align*}
        & \frac{L\Delta}{T}\wtilde{\cO}\left(
        \left(\frac{T\sigma^2}{L\Delta \hbeta^2\eta^2}\right)^{1/4} 
        + \left(\frac{B_{\rm init}\sigma\sqrt{T}}{L\Delta\sqrt{G}\hbeta\eta}\right)^{1/2}
        + \left(\frac{\sigma(\sqrt{L\Delta}+B_{\rm init}+\sigma)\sqrt{T}}{L\Delta\sqrt{G}\hbeta^2\eta^2}\right)^{1/3}
        + \left(\frac{\sigma\sqrt{c\deltabyz}\tau T^{3/2}}{L\Delta\sqrt{G}\hbeta\eta}\right)^{1/3}
        \right) \notag\\
        =\;& \wtilde{\cO}\left(
        \left(\frac{(L\Delta)^2B_{\rm init}^2C^2_1a^{4/3}\sigma^2}{T^{7/3}\tau^{4/3}}\right)^{1/4} 
        + \left(\frac{\sqrt{L\Delta}B_{\rm init}^2C_1a^{2/3}\sigma}{T^{7/6}\sqrt{G}\tau^{2/3}}\right)^{1/2}
        \right.
        \notag\\
        &\left. +\; 
         \left(\frac{L\Delta B_{\rm init}^2C^2_1a^{4/3}\sigma(\sqrt{L\Delta}+B_{\rm init}+\sigma)}{T^{11/6}\sqrt{G}\tau^{4/3}}\right)^{1/3}
         + \left(\frac{(L\Delta)^{3/2}B_{\rm init}C_1a^{2/3}\sigma\sqrt{c\deltabyz}\tau }{T^{7/6}\sqrt{G}\tau^{2/3}}\right)^{1/3}
        \right).
    \end{align*}
    Now we use the specific form of $a \sim \nicefrac{\sqrt{d}\sigma_{\omega}\sqrt{T}}{\sqrt{G}}$ to obtain the rate
    \begin{align}\label{eq:nrjoqneasdasdafq}
        &\wtilde{\cO}\left(
        \left(\frac{(L\Delta)^2B_{\rm init}^2C_1^2d^{2/3}\sigma_{\omega}^{4/3}\sigma^2}{G^{2/3}T^{5/3}\tau^{4/3}}\right)^{1/4} 
        + \left(\frac{\sqrt{L\Delta}B_{\rm init}^2C_1d^{1/3}\sigma_{\omega}^{2/3}\sigma}{(GT)^{5/6}\tau^{2/3}}\right)^{1/2}\right.\\
        &+\;\left. \left(\frac{L\Delta B_{\rm init}^2C_1^2d^{2/3}\sigma_{\omega}^{4/3}\sigma(\sqrt{L\Delta}+B_{\rm init}+\sigma)}{(TG)^{7/6}\tau^{4/3}}\right)^{1/3}
         + \left(\frac{(L\Delta)^{3/2}B_{\rm init}C_1d^{1/3}\sigma_{\omega}^{2/3}\sigma\sqrt{c\deltabyz}\tau }{(GT)^{5/6}\tau^{2/3}}\right)^{1/3}
        \right).\notag
    \end{align}
    Since $\sigma_{\omega} \sim \sqrt{T}$, then the first term in \eqref{eq:nrjoqneasdasdafq} has better dependency on $T$ than the other two, therefore we omit them. The final rate in this case is 
    \begin{align}
         \wtilde{\cO}\left(\left(\frac{L\Delta B_{\rm init}^2C^2_1d^{2/3}\sigma_{\omega}^{4/3}\sigma(\sqrt{L\Delta}+B_{\rm init}+\sigma)}{(TG)^{7/6}\tau^{4/3}}\right)^{1/3}
         + \left(\frac{(L\Delta)^{3/2}B_{\rm init}Cd^{1/3}\sigma_{\omega}^{2/3}\sigma\sqrt{c\deltabyz}\tau }{(GT)^{5/6}\tau^{2/3}}\right)^{1/3}
        \right).
    \end{align}
    
    \item $\hbeta\sim \frac{\sqrt{L\Delta}}{(c\deltabyz)^{1/3}\tau T}$, gives the term in the rate of the form
    \begin{align*}
        & \frac{L\Delta}{T}\wtilde{\cO}\left(
        \left(\frac{T\sigma^2}{L\Delta \hbeta^2\eta^2}\right)^{1/4} 
        + \left(\frac{B_{\rm init}\sigma\sqrt{T}}{L\Delta\sqrt{G}\hbeta\eta}\right)^{1/2}
        + \left(\frac{\sigma(\sqrt{L\Delta}+B_{\rm init}+\sigma)\sqrt{T}}{L\Delta\sqrt{G}\hbeta^2\eta^2}\right)^{1/3}
        + \left(\frac{\sigma\sqrt{c\deltabyz}\tau T^{3/2}}{L\Delta\sqrt{G}\hbeta\eta}\right)^{1/3}
        \right) \notag\\
        =\;& \wtilde{\cO}\left(
        \left(\frac{(L\Delta)^2(c\deltabyz)^{2/3}B_{\rm init}^2\sigma^2}{T}\right)^{1/4} 
        + \left(\frac{(L\Delta)^{1/2} B_{\rm init}^2(c\deltabyz)^{1/3}\sigma}{\sqrt{GT}}\right)^{1/2}
        \right.
        \notag\\
        &\left. +\; 
        \left(\frac{ L\Delta B_{\rm init}^2(c\deltabyz)^{2/3}\sigma(\sqrt{L\Delta}+B_{\rm init}+\sigma)}{\sqrt{TG}}\right)^{1/3}
        + 
         \left(\frac{(c\deltabyz)^{5/6}(L\Delta)^{3/2} B_{\rm init}\sigma\tau }{\sqrt{TG}}\right)^{1/3}
        \right).
    \end{align*}
    We keep only the last two terms, since they have the worst dependence on $T$. Therefore, the final rate in this case is 
    \begin{align}\label{eq:onoqasfqwwndq}
        &\wtilde{\cO}\left(
        \left(\frac{L\Delta B_{\rm init}^2\sigma(\sqrt{L\Delta}+B_{\rm init}+\sigma)}{\sqrt{TG}}\right)^{1/3}
        + 
         \left(\frac{\sqrt{c\deltabyz}(L\Delta)^{3/2} B_{\rm init}\sigma\tau }{\sqrt{TG}}\right)^{1/3}
        \right).
    \end{align}
    \end{enumerate}

    We observe that, in all three cases above, we still have descent with $T$. Therefore, those terms do not contribute to the utility bound. Now we should consider the momentum constraints (a)-(d) that can be combined into one as 
    $$\beta \le \wtilde{\cO}\left(\min\left\{\frac{\hbeta\tau}{\sqrt{L\Delta} + B_{\rm init} + \sigma}, \frac{1}{\sqrt{c\deltabyz}T}\right\}\right).$$ 
    This translates into the rate of the form
    \[
    \wtilde{\cO}\left(\frac{L\Delta(\sqrt{L\Delta} + B_{\rm init} + \sigma)}{T\hbeta\tau} + L\Delta\sqrt{c\deltabyz}\right).
    \]
    We should plug in the restrictions \eqref{eq:jnoqneqoejnfqe} in the rate above. This leads to the rate of the form
    \begin{align*}
        &\wtilde{\cO}\left(\frac{L\Delta a C(\sqrt{L\Delta} + B_{\rm init} + \sigma)}{T\sqrt{L\Delta}\tau} 
        + \frac{L\Delta C_1a^{2/3}T^{1/3}\tau^{1/3}(\sqrt{L\Delta} + B_{\rm init} + \sigma)}{T\sqrt{L\Delta}\tau}
        + \frac{L\Delta\tau T(\sqrt{L\Delta} + B_{\rm init} + \sigma)}{T\sqrt{L\Delta}\tau} \right.\\
        &\left.+\; L\Delta\sqrt{c\deltabyz}
        \right)\\
        =\;& \wtilde{\cO}\left(\frac{\sqrt{L\Delta} a C(\sqrt{L\Delta} + B_{\rm init} + \sigma)}{T\tau} 
        + \frac{\sqrt{L\Delta} C_1a^{2/3}(\sqrt{L\Delta} + B_{\rm init} + \sigma)}{T^{2/3}\tau^{2/3}}\right.\\
        &\left.+\; \sqrt{L\Delta}(\sqrt{L\Delta} + B_{\rm init} + \sigma)((c\deltabyz)^{1/3}+\sqrt{c\deltabyz})
        \right).
    \end{align*}

    Now we use $a \sim \frac{\sqrt{d}\sigma_{\omega}\sqrt{T}}{\sqrt{G}}$, $C= (1+\sqrt{G})$, and $C_1 = (1+G^{1/3}(c\deltabyz)^{1/3})$ to obtain the rate
    \begin{align*}
         &(\sqrt{L\Delta} + B_{\rm init} + \sigma)\cdot \wtilde{\cO}\left(\frac{\sqrt{L\Delta} \sqrt{d}\sigma_{\omega}C}{\sqrt{GT}\tau} 
        + \frac{\sqrt{L\Delta}C_1 d^{1/3}\sigma_{\omega}^{2/3}}{(TG)^{1/3}\tau^{2/3}}
        + \sqrt{L\Delta}((c\deltabyz)^{1/3}+\sqrt{c\deltabyz})
        \right)\\
        =\;& (\sqrt{L\Delta} + B_{\rm init} + \sigma)
        \cdot 
        \wtilde{\cO}\left(\frac{\sqrt{L\Delta} \sqrt{d}\sigma_{\omega}(1+\sqrt{c\deltabyz G})}{\sqrt{GT}\tau} 
        + \frac{\sqrt{L\Delta}(1+(c\deltabyz G)^{1/3}) d^{1/3}\sigma_{\omega}^{2/3}}{(TG)^{1/3}\tau^{2/3}}\right.\\
        &\left.
        + \sqrt{L\Delta}((c\deltabyz)^{1/3}+\sqrt{c\deltabyz})
        \right)\\
        =\;& (\sqrt{L\Delta} + B_{\rm init} + \sigma) 
        \wtilde{\cO}\left(\frac{\sqrt{L\Delta d} \sigma_{\omega}}{\sqrt{GT}\tau} 
        + \frac{\sqrt{L\Delta} d^{1/3}\sigma_{\omega}^{2/3}}{(TG)^{1/3}\tau^{2/3}}
        + \frac{\sqrt{c\deltabyz L\Delta d}\sigma_{\omega}}{\sqrt{T}\tau} 
        + \frac{\sqrt{L\Delta} (c\deltabyz d\sigma_{\omega}^2)^{1/3}}{T^{1/3}\tau^{2/3}}\right.\\
        &\left. + \sqrt{L\Delta} ((c\deltabyz)^{1/3}+\sqrt{c\deltabyz})
        \right).
    \end{align*}
    We can combine the terms above in a simpler way as follows
    \begin{align*}
        & (\sqrt{L\Delta} + B_{\rm init} + \sigma)\sqrt{L\Delta} 
        \wtilde{\cO}\left( \left(\frac{\sqrt{d}\sigma_{\omega}}{\sqrt{GT}\tau} 
        + \frac{\sqrt{c\deltabyz d}\sigma_{\omega}}{\sqrt{T}\tau} 
        + \sqrt{c\deltabyz}\right) + \left(\frac{\sqrt{d}\sigma_{\omega}}{\sqrt{GT}\tau} 
        + \frac{\sqrt{c\deltabyz d}\sigma_{\omega}}{\sqrt{T}\tau} 
        + \sqrt{c\deltabyz}\right)^{2/3}
        \right).
    \end{align*}
    We observe that the first term is purely due to DP noise, the middle term is a mix of DP noise and robust aggregation, and the last term is purely due to robust aggregation.

    Finally, the restriction \eqref{eq:wnojenoqdqw} translate to 
    \begin{align}\label{eq:knweokfnewe}
    \gamma \le \frac{1}{L}\wtilde{\cO}\left(\frac{\hbeta\tau}{LB_{\rm init}}, \frac{\sqrt{\hbeta\tau}}{L\sqrt{B_{\rm init}}}\right).
    \end{align}
    Note that the second restriction in \eqref{eq:knweokfnewe} does not contribute to the neighborhood. The first restriction in \eqref{eq:knweokfnewe} is already taken into account by the restriction $\beta \le \wtilde{\cO}\left(\frac{\hbeta\tau}{\sqrt{L\Delta} + B_{\rm init}+ \sigma}\right)$.

    \paragraph{Case $\cI_{K+1} = 0.$} This case is even easier. The only change will be with the term next to $R^t$. We will get 
    \[
    1 - \frac{96L^2}{\hat\beta^2\eta^2}\gamma^2 
        - \frac{24L^2}{\beta^2}\gamma^2 \ge \frac{1}{3} - \frac{96L^2}{\hat\beta^2\eta^2}\gamma^2 \ge 0
    \]
    instead of 
    \[
    1 - \frac{32\beta^2L^2}{\hat\beta^2\eta^2}\gamma^2 
    - \frac{96L^2}{\hat\beta^2\eta^2}\gamma^2 
    - \frac{24L^2}{\beta^2}\gamma^2 \ge 0
    \]
     as in the previous case. This difference comes from \Cref{lem:descent_Vt_tilde} because $\wtilde{V}^{K+1} = 0$. The rest is a repetition of the previous derivations.

    \paragraph{Bounds on $\Delta$, $B_{\rm init}$, and $\Delta_0$.}

    It remains to give explicit bounds on $\Delta$ and $B_{\rm init}$ using problem-dependent constants. With the initialization $v_i^0=g_i^0 = 0$ for all $i\in\cG$ we have that 
    \begin{align}
    \Phi^0 &= f(x^0) - f^\star + \frac{8\gamma\beta}{\hbeta^2\eta^2}\frac{1}{G}\sum_{i\in\cG}\|\nabla f_i(x^0)\|^2 + \frac{2\gamma}{\beta}\|\nabla f(x^0)\|^2\notag\\
    &\overset{(i)}{\le} f(x^0)-f^\star 
    + \frac{8\gamma\beta}{\hbeta^2\eta^2}B_{\zeta}\|\nabla f(x^0)\|^2 
    + \frac{8\gamma\beta}{\hbeta^2\eta^2}\zeta^2 + \frac{4\gamma}{\beta}L(f(x^0)-f^\star)\notag\\
    &\overset{(ii)}{\le} \left(1 
    + \frac{16L\gamma\beta}{\hbeta^2\eta^2}B_{\zeta} 
    + \frac{4L\gamma}{\beta}\right)(f(x^0)-f^\star) + \frac{8\gamma\beta}{\hbeta^2\eta^2}\zeta^2.\label{eq:kenoqkqnewok}
    \end{align}
    where $(i)$ follow from \Cref{asmp:bounded_heterogeneity} and $L$-smoothness, $(ii)$ -- from $L$-smoothness. Note that we choose $\beta=12L\gamma$, which implies that $\frac{4L\gamma}{\beta} = \frac{1}{3}$, and from \eqref{eq:wnojenoqdqw} we have
    \[
        \frac{8\gamma\beta}{\hbeta^2\eta^2} = \frac{96L\gamma^2}{\hbeta^2\eta^2} \le 96L\cdot \frac{5}{6\cdot 96L^2} = \frac{5}{6L}.
    \]
    Thus, we have the bound on $\Phi^0$ from \eqref{eq:kenoqkqnewok}
    \[
    \Phi^0 \le  \left(\nicefrac{4}{3} 
    + \frac{5}{3}B_{\zeta}\right)(f(x^0)-f^\star) + \frac{5}{6L}\zeta^2.
    \]
    Therefore, we have that 
    \[
    L\Delta = \cO\left((1+B_{\zeta})
    L(f(x^0)-f^\star) + \zeta^2\right) = \cO\left(
    (1+B_{\zeta})L
    F^0 + \zeta^2\right)
    \]
    Next, we bound $B_{\rm init} = \max\{3\tau, \max_{i\in\cG}\|\nabla f_i(x^0)\| + b\} > \tau$. Note that in the theorem, we assume that $\max_{i\in\cG}\|\nabla f_i(x^0)\| \ge 3\tau$. This implies that $B_{\rm init} = \max_{i\in\cG}\|\nabla f_i(x^0)\| + b$. Therefore, we provide the bound with high probability as follows 
    \begin{align*} 
    B_{\rm init} &= b + \max_{i\in\cG}\|\nabla f_i(x^0)\|\\
    &= \wtilde{\cO}\left(\sigma + \max_{i\in\cG}\|\nabla f_i(x^0)-\nabla f(x^0)\| + \|\nabla f(x^0)\|\right)\\
    &= \wtilde{\cO}\left(\sigma + \sqrt{B_{\zeta}LF^0} + \zeta + \sqrt{LF^0}\right).
    \end{align*}
    where we again use \Cref{asmp:bounded_heterogeneity} and $L$-smoothness. This allows to bound $\Delta_0$ as follows
    \begin{align*}
        \Delta_0 &= \sqrt{L\Delta}\left(\sqrt{L\Delta} + B_{\rm init} + \sigma\right)\\
        &= \wtilde{\cO}\left((1+\sqrt{B_{\zeta}})(\sqrt{LF^0} + \zeta)(\sigma + (1+\sqrt{B_{\zeta}})\sqrt{LF^0} + \zeta)\right)\\
        &= \wtilde{\cO}\left((1+B_{\zeta})LF^0 + \zeta^2 + \sigma ((1+\sqrt{B_{\zeta}})\sqrt{LF^0} + \zeta)\right).
    \end{align*}

     \end{proof}

     \begin{theorem}[Theorem 3.22 in \citep{dwork2014algorithmic}] \label{th:gaussian_machanism}
    Let $\varepsilon\in(0,1)$ be arbitrary. For $c^2 > 2\log(1.25/\delta)$, the Gaussian mechanism with parameter $\sigma_{\omega} \ge c\Delta_2(f)/\varepsilon$ is $(\varepsilon, \delta)$-differentially private.
\end{theorem}

     \begin{corollary}\label{cor:cor_main_theorem_full}
         Under the setup of \Cref{th:main_theorem} (\Cref{th:main_theorem_full}), if we set $\sigma_{\omega} = \Theta\left(\frac{\tau}{\varepsilon}\sqrt{T\log(\frac{T}{\delta})\log(\frac{1}{\delta})}\right)$ for some $\varepsilon,\delta\in(0,1)$, then each iteration of \algname{Byz-Clip21-SGD2M} satisfies local $(\varepsilon, \delta)$-DP.
     \end{corollary}
     \begin{proof}
        We need to plug in the choice of $\sigma_{\omega}$ to \Cref{th:main_theorem_full} and omit all terms that decay with $T$. Note that each step of \algname{Byz-Clip21-SGD2M} satisfies $(\wtilde{\varepsilon}, \wtilde{\delta})$-local DP where
        \[
            \wtilde{\varepsilon} = \frac{\varepsilon}{2\sqrt{2T\log \frac{1}{\delta}}}, \quad \wtilde{\delta} = \frac{\delta}{T},
        \]
        by \Cref{th:gaussian_machanism}. Combining the above result with advanced composition theorem (\citep[Theorem 3.20]{dwork2014algorithmic}), we obtain that all $T$ iterations of \algname{Byz-Clip21-SGD2M} satisfy local $(\varepsilon, \delta)$-differentially private.
     \end{proof}

\section{Proofs in Special Cases}\label{apx:special_cases}

\subsection{Convergence in the Absence of Both DP and Byzantine Adversaries}

    \begin{corollary}[Full statement of \Cref{eq:corollary_no_dp_no_byz}] Let Assumptions \ref{asmp:smoothness}, \ref{asmp:stoch_grad}, and \ref{asmp:bounded_heterogeneity} and define $B_{\rm init}\eqdef \max\{3\tau, \max_i\{\|\nabla f_i(x^0)\|\}+b\}$. Let the failure probability $\alpha$ be such that $\alpha\in(0,1)$, and constants $b$ and $\hc$ be defined as in (\ref{eq:constants}), and $\Delta \ge \Phi^0$ for $\Phi^0$ defined in \eqref{eq:lyapunov_function}. Consider the run of \algname{Byz-Clip21-SGD2M} (\Cref{alg:byz_clip21_sgd2m}) for $T$ iterations with no DP and Byzantine adversaries, i.e., $\sigma_{\omega} = 0$ and $\deltabyz=0$.
    Assume the following inequalities hold
    \begin{enumerate}
        \item {\bf step-size restrictions:} 
        \begin{enumerate}
       \item $\gamma \le \frac{1}{24L}$;
        \item $12L\gamma = \beta$;
        \item \begin{equation}
            \frac{5}{6} - \frac{32L^2\beta^2}{\hat\beta^2\eta^2}\gamma^2
    - \frac{96L^2}{\hat\beta^2\eta^2}\gamma^2 \ge 0.
        \end{equation}
        \end{enumerate}
        \item {\bf momentum restrictions:}
        \begin{enumerate}
        \item $\hbeta=1;$
        \item $\beta\le \min\left\{\frac{8\tau}{431\sqrt{L\Delta}}, 1\right\};$
        \item $\beta\le \min\left\{\frac{2\tau}{39(B_{\rm init}-\tau)}, 1\right\};$
        \item $\beta\le \min\left\{\frac{2\tau}{69b}, 1\right\};$
        \item and momentum restrictions defined in (\ref{eq:step-size_bound_1}), (\ref{eq:step-size_bound_2}), (\ref{eq:step-size_bound_3}), (\ref{eq:step-size_bound_4}), (\ref{eq:step-size_bound_5}), (\ref{eq:step-size_bound_6}), (\ref{eq:step-size_bound_7}), (\ref{eq:step-size_bound_8}), (\ref{eq:step-size_bound_9}), and (\ref{eq:step-size_bound_10}) with $\hbeta=1$, $a, \ha = 0, \deltabyz=0$;
        \end{enumerate}
    \end{enumerate}
    Then, with probability $1-\alpha$, we bound $\frac{1}{T}\sum_{t=0}^{T-1}\|\nabla f(x^t)\|^2$ with
   \begin{align*}
        \wtilde{\cO}\left(
          \frac{L\Delta(1+\nicefrac{B_{\rm init}}{\tau})}{T}
        + \frac{\sigma(\sqrt{L\Delta}+B_{\rm init}+\sigma)}{\sqrt{TG}}\right),
    \end{align*}
    where $\wtilde{\cO}$ hides constant and logarithmic factors and higher order terms decreasing in $T$.
\end{corollary}

     \begin{proof}
         In the case, when $\hat{\beta}=1$ the worst dependency in \eqref{eq:ninwofejnasdadwe} w.r.t. $T$ comes from the terms $\circledOne$, $\circledThree,$ $\circledFive,$ and $\circledSix$, and equals $\gamma\sim \frac{1}{T^{1/2}}$. We also have restriction $\gamma \le \cO(\nicefrac{1}{L})$ and \eqref{eq:knweokfnewe}. All of those restrictions give the rate of the form
    \begin{align}\label{eq:iqoendfqowjdnwq111}
        &\frac{L\Delta}{T}\wtilde{\cO}\left(
        1
        + \frac{B_{\rm init}}{\tau}
        + \frac{\sqrt{B_{\rm init}}}{\sqrt{\tau}}
        + \left(\frac{T\sigma^2}{L\Delta G}\right)^{1/2}
        + \frac{\sigma(\sqrt{L\Delta}+B_{\rm init}+\sigma)T^{1/2}}{L\Delta\sqrt{G}}\right)\notag\\
        =\;&\wtilde{\cO}\left(
          \frac{L\Delta(1+\nicefrac{B_{\rm init}}{\tau})}{T}
        + \frac{\sigma\sqrt{L\Delta}}{\sqrt{TG}}
        + \frac{\sigma(\sqrt{L\Delta}+B_{\rm init}+\sigma)}{\sqrt{TG}}\right)\notag\\
        =\;&\wtilde{\cO}\left(
          \frac{L\Delta(1+\nicefrac{B_{\rm init}}{\tau})}{T}
        + \frac{\sigma(\sqrt{L\Delta}+B_{\rm init}+\sigma)}{\sqrt{TG}}\right),
    \end{align}
    where we hide all logarithmic factors in the $\wtilde{\cO}$ notation.
     \end{proof}

    \subsection{Convergence in the Absence of Byzantine Adversaries}

    \begin{corollary}[Full statement of \Cref{cor:withdpnobyz}] Let Assumptions \ref{asmp:smoothness}, \ref{asmp:stoch_grad}, and \ref{asmp:bounded_heterogeneity} and define $B_{\rm init}\eqdef \max\{3\tau, \max_i\{\|\nabla f_i(x^0)\|\}+b\}$. Let the failure probability $\alpha$ be such that $\alpha\in(0,1)$, and constants $a,\ha, b,$ and $\hc$ be defined as in (\ref{eq:constants}), and $\Delta \ge \Phi^0$ for $\Phi^0$ defined in \eqref{eq:lyapunov_function}. Consider the run of \algname{Byz-Clip21-SGD2M} (\Cref{alg:byz_clip21_sgd2m}) for $T$ iterations with DP noise variance $\sigma_{\omega}$ and $\deltabyz=0$.
    Assume the following inequalities hold
    \begin{enumerate}
        \item {\bf step-size restrictions:} 
        \begin{enumerate}
       \item $\gamma \le \frac{1}{24L}$;
        \item $12L\gamma = \beta$;
        \item \begin{equation}
            \frac{5}{6} - \frac{32L^2\beta^2}{\hat\beta^2\eta^2}\gamma^2
    - \frac{96L^2}{\hat\beta^2\eta^2}\gamma^2 \ge 0.
        \end{equation}
        \end{enumerate}
        \item {\bf momentum restrictions:}
        \begin{enumerate}
        \item $\beta\le \min\left\{\frac{8\hbeta\tau}{431\sqrt{L\Delta}}, 1\right\};$
        \item $\beta\le \min\left\{\frac{2\hbeta\tau}{39(B_{\rm init}-\tau)}, 1\right\};$
        \item $\beta\le \min\left\{\frac{2\hbeta\tau}{69b}, 1\right\};$
        \item $\hbeta \in \min\left\{\frac{\sqrt{L\Delta}}{a}, 1\right\};$
        \item and momentum restrictions defined in (\ref{eq:step-size_bound_1}), (\ref{eq:step-size_bound_2}), (\ref{eq:step-size_bound_3}), (\ref{eq:step-size_bound_4}), (\ref{eq:step-size_bound_5}), (\ref{eq:step-size_bound_6}), (\ref{eq:step-size_bound_7}), (\ref{eq:step-size_bound_8}), (\ref{eq:step-size_bound_9}), and (\ref{eq:step-size_bound_10}) with $\deltabyz=0$;
        \end{enumerate}
    \end{enumerate}
    Then, with probability $1-\alpha$, we bound $\frac{1}{T}\sum_{t=0}^{T-1}\|\nabla f(x^t)\|^2$ with
   \begin{align*}
        \sqrt{L\Delta}(\sqrt{L\Delta} + B_{\rm init} + \sigma) \wtilde{\cO}\left(\frac{\sqrt{d}\sigma_{\omega}}{\sqrt{GT}\tau} 
        + \frac{d^{1/3}\sigma_{\omega}^{2/3}}{(TG)^{1/3}\tau^{2/3}}
        \right),
    \end{align*}
    where $\wtilde{\cO}$ hides constant and logarithmic factors and higher order terms decreasing in $T$.
\end{corollary}
\begin{proof}
    We should simply set $\deltabyz=0$ in the convergence proof of \Cref{th:main_theorem} (\Cref{th:main_theorem_full}) to obtain the necessary result.
\end{proof}

\subsection{Convergence in the Absence of DP Adversaries}

In this setting, server-side momentum and clipping are unnecessary. We thus focus on the specialization of \algname{Byz-Clip21-SGD2M} with $\hbeta=1$ and $\tau=+\infty$, shown in \Cref{alg:byz_clip21_sgd2m_nodp}. Under these settings, the method reduces to the algorithm of \citet{karimireddy2020byzantine}. In particular, the error-feedback mechanism is disabled, so the server does not store $g_i^t$. The analysis correspondingly simplifies and only requires a simpler Lyapunov function
\begin{equation}\label{eq:lyapunov_function_nodp}
    \wtilde{\Phi}^t \eqdef \delta^t + \frac{\gamma}{\beta}\|\overline{v}^t - \nabla f(x^t)\|^2.
\end{equation}
Also, constants $b$ and $c$ from \eqref{eq:constants} should be adjusted accordingly. We make use of their replacements $\wtilde{b}$ and $\wtilde{c}$ defined as 
\begin{align}\label{eq:constants_nodp}
    &\tilde{b}^2 \eqdef 2\sigma^2\log\left(\frac{8(T+1)G}{\alpha}\right), \quad \tilde{c}^2 \eqdef \left(\sqrt{2} + 2\sqrt{3\log\frac{8(T+1)}{\alpha}}\right)^2\sigma^2,\notag\\
    &\tilde{z}^2 \eqdef \left(\sqrt{2} + 2\sqrt{3\log\frac{8G(T+1)}{\alpha}}\right)^2\sigma^2.
\end{align}

\begin{algorithm}[t]
\caption{\algname{Byz-Clip21-SGD2M} without DP adversaries}
\label{alg:byz_clip21_sgd2m_nodp}
\begin{algorithmic}[1]
\STATE \textbf{Input:} $x^0 \in X,$ momentum parameter $\beta\in(0,1],$ step-size $\gamma > 0$
    \FOR{$t=0, \ldots, T-1$}
        \STATE $x^{t+1} = x^t - \gamma g^t$
        \FOR{$i\in\cG$}
            \STATE $v_i^{t+1} = (1-\beta)v_i^t + \beta\nabla f_i(x^{t+1}, \xi^{t+1}_i)$
        \ENDFOR
        \FOR{$i\in\cB$}
            \STATE $v_i^{t+1} = (*)$ \hfill sends arbitrary vector
        \ENDFOR
         \STATE $g^{t+1} = \aragg(v_1^{t+1},\dots,v_n^{t+1})$
        
    \ENDFOR
\end{algorithmic}	
\end{algorithm}

\subsubsection{Useful Lemmas}

We will need to re-derive several descent lemmas as well as introduce new ones. 

\begin{lemma}\label{lem:bound_vt_vt_average} Let $\{v_i^t\}$ be generated by \Cref{alg:byz_clip21_sgd2m_nodp}. Then we have with probability at least $1-\frac{\alpha}{8(T+1)}$
\[
\frac{1}{G}\sum_{i\in\cG}\|\overline{v}^t-v_i^t\|^2 \le 2\zeta^2 + 2\beta\tilde{z}^2.
\]
\end{lemma}
\begin{proof}
    We have 
    \begin{align}
        \frac{1}{G}\sum_{i\in\cG}\left\|v_i^t - \overline{v}^t\right\|^2 &\overset{(i)}{=} \frac{1}{G}\sum_{i\in\cG}\|(1-\beta)(v_i^{t-1} - \overline{v}^{t-1}) + \beta(\nabla f_i(x^{t}) - \nabla f(x^t)) + \beta(\theta_i^t - \theta^t)\|^2\notag\\
        &\overset{(ii)}{=} \frac{1}{G}\left\|\beta\sum_{k=0}^t(1-\beta)^{t-k}(\nabla f_i(x^k) - \nabla f(x^k)) + \beta\sum_{k=0}^t (1-\beta)^{t-k}(\theta_i^k-\theta^k)\right\|^2\notag\\
        &\overset{(iii)}{\le} \frac{2}{G}\sum_{i\in\cG}\left\|\beta\sum_{k=0}^t(1-\beta)^{t-k}(\nabla f_i(x^k) - \nabla f(x^k))\right\|^2 + \frac{2\beta^2}{G}\left\|\sum_{k=0}^t(1-\beta)^{t-k}(\theta_i^k - \theta^k)\right\|^2\notag\\
        &\overset{(iv)}{\le} \frac{2(1-(1-\beta)^{t+1})}{G}\sum_{i\in\cG}\sum_{k=0}^{t}\beta(1-\beta)^{t-k}\|\nabla f_i(x^k) - \nabla f(x^k)\|^2 \notag\\
        &\hspace{7cm}+ \frac{2\beta^2}{G}\sum_{i\in\cG}\left\|\sum_{k=0}^t(1-\beta)^{t-k}(\theta_i^k - \theta^k)\right\|^2,\notag\\
        &\overset{(v)}{\le} 2\sum_{k=0}^t\beta(1-\beta)^{t-k}\zeta^2 + \frac{2\beta^2}{G}\sum_{i\in\cG}\left\|\sum_{k=0}^t(1-\beta)^{t-k}(\theta_i^k-\theta^k)\right\|^2\notag\\
        &\le 2\zeta^2 + \frac{2\beta^2}{G}\sum_{i\in\cG}\left\|\sum_{k=0}^t(1-\beta)^{t-k}(\theta_i^k-\theta^k)\right\|^2,
        \label{eq:agg_inequality_nodp}
    \end{align}
    where $(i)$ -- from the update rules of $v_i^t$ and decomposition $\nabla f_i(x^t)+\theta_i^t = \nabla f_i(x^t,\xi_i^t)$, $(ii)$ -- unrolling the recursion till the zero iteration and using the initialization with zeros, $(iii)$ -- from Young's inequality, $(iv)$ -- from Jensen's inequality, $(v)$ -- from \Cref{asmp:bounded_heterogeneity}. For the second term in \eqref{eq:agg_inequality_nodp}, we use \Cref{lem:concentration_lemma}. We know that $\theta_i^k-\theta^k$ is a zero-centered sub-Gaussian r.v. conditioned on all events before iteration $k$. Therefore, we have 
    \begin{equation}
        \left\|\sum_{k=0}^t(1-\beta)^{t-k}(\theta_i^k-\theta^k)\right\|^2 \le (\sqrt{2}+2u)\sqrt{\sum_{k=0}^t (1-\beta)^{2(t-k)}\sigma^2} \le (\sqrt{2}+2u)\frac{\sigma}{\sqrt{\beta}} \le \frac{\tilde{z}}{\sqrt{\beta}},
    \end{equation}
    where $u = \sqrt{3\log\frac{8G(T+1)}{\alpha}}.$ Thus, we obtain with probability at least $1-G\cdot \frac{\alpha}{8G(T+1)}$ that 
    \begin{equation}
        \frac{1}{G}\sum_{i\in\cG}\|v_i^t - \overline{v}^t\|^2 \le 2\zeta^2 + 2\beta\tilde{z}^2.
    \end{equation}

\end{proof}

Next, we re-derive the descent lemma for the function value.

\begin{lemma}\label{lem:descent_lemma_in_f_no_dp}
Let $f$ be $L$-smooth, Assumption~\ref{asmp:bounded_heterogeneity} hold with $A=0$, $F^t \eqdef f(x^t) - f^\star,$ $\{x^t\}$ be generated by \Cref{alg:byz_clip21_sgd2m_nodp}. Assume that 
\begin{enumerate}
    \item $\gamma \le \frac{1}{2L}$;
    
    \item $\|\theta_i^t\| \le \wtilde{b}$;

    \item $\frac{1}{G}\sum_{i\in\cG}\|\overline{v}^{t-1} - v_i^{t-1}\|^2 \le 2\zeta^2 + 2\beta\tilde{z}^2$;
\end{enumerate}
Then we have 
\begin{align}\label{eq:descent_lemma_in_f_nodp}
        f(x^{t+1}) &\le f(x^t)
        - \frac{\gamma}{2}\|\nabla f(x^t)\|^2
        - \frac{1}{4\gamma}\|x^{t+1} - x^t\|^2
        + \gamma\|\nabla f(x^t) - \overline{v}^t\|^2
        + 2c\deltabyz\zeta^2 + 2c\deltabyz\beta\tilde{z}^2.
    \end{align}
    
\end{lemma}
\begin{proof}
    Using the derivations from \citep{islamov2025double}, Lemma 2 we first get
    \begin{align*}
        f(x^{t+1}) &\le f(x^t) 
        - \frac{\gamma}{2}\|\nabla f(x^t)\|^2
        - \frac{\gamma}{4}\|g^t\|^2
        + \frac{\gamma}{2}\|\nabla f(x^t) - g^t\|^2.
    \end{align*}
    We continue the derivations as follows 
    \begin{align}\label{eq:jefnjenadasdaqoldwnq}
        f(x^{t+1}) &\overset{(i)}{\le} f(x^t) 
        - \frac{\gamma}{2}\|\nabla f(x^t)\|^2
        - \frac{\gamma}{4}\|g^t\|^2
        + \gamma\|\nabla f(x^t) - \overline{v}^t\|^2
        + \gamma\|\overline{v}^t - g^t\|^2\notag\\
        &\overset{(ii)}{=} f(x^t) 
        - \frac{\gamma}{2}\|\nabla f(x^t)\|^2
        - \frac{1}{4\gamma}\|x^{t+1} - x^t\|^2
        + \gamma\|\nabla f(x^t) - \overline{v}^t\|^2
        + \gamma\|\overline{v}^t - g^t\|^2,
    \end{align}
    where $(i)$ follows from Jensen's inequality applied to $\|\cdot\|^2$, $(ii)$ --- from the update rule of $x^t$. We bound the term $\|\overline{v}^t-g^t\|^2$ using properties of the aggregator.
    \begin{align}\label{eq:vt_bar_gt}
        \|\overline{v}^t-g^t\|^2 &= \|\overline{v}^t-\aragg(v_1^t, \dots, v_n^t)\|^2\notag\\
        &\le \frac{c\deltabyz}{G}\sum_{i\in\cG}\|v_i^t - \overline{v}^t\|^2\notag\\
        &\overset{(iii)}{\le} 2c\deltabyz\zeta^2+2c\deltabyz\beta\tilde{z}^2.
    \end{align}
    where $(iii)$ holds by \Cref{lem:bound_vt_vt_average}. Plugging in \eqref{eq:vt_bar_gt} into \eqref{eq:jefnjenadasdaqoldwnq}, we obtain the statement of the lemma.
    
\end{proof}

\begin{lemma}\label{lem:bound_vt_bar_nodp} Let each $f_i$ be $L$-smooth, and $\wtilde{\Delta} \ge \Phi^0.$ Assume that the following inequalities hold for the iterates generated by \algname{Byz-Clip21-SGD2M}
\begin{enumerate}
    \item $\gamma \le \frac{1}{6L}$,
    \item $\beta\in[0,1]$,
    \item $\|\overline{v}^{t-1}\| \le 
    \sqrt{64L\wtilde{\Delta} }+ 3\frac{\wtilde{c}}{\sqrt{G}} + 3\sqrt{2c\deltabyz}\zeta + 3\sqrt{\beta}\sqrt{2c\deltabyz} \tilde{z}$,
    \item $\|\nabla f(x^{t-1}) - \overline{v}^{t-1}\| \le \sqrt{4L\wtilde{\Delta} }
    + \frac{3}{2}\frac{\wtilde{c}}{\sqrt{G}}
    + \frac{3}{2}\sqrt{2c\deltabyz}\zeta
    + \frac{3}{2}\sqrt{\beta}\sqrt{2c\deltabyz}\tilde{z}$,
    \item $\|\theta^t\| \le \frac{\wtilde{c}}{\sqrt{G}}$,
    \item $\frac{1}{G}\sum_{i\in\cG}\|\overline{v}^{t-1} - v_i^{t-1}\|^2 \le 2c\deltabyz\zeta^2+2c\deltabyz\beta\tilde{z}^2$,
    \item $\Phi^{t-1} \le 2\wtilde{\Delta}.$
\end{enumerate}
    Then we have 
    \begin{align}
        \|\overline{v}^t\| \le 
        \sqrt{64L\wtilde{\Delta}}
        + 3\frac{\wtilde{c}}{\sqrt{G}} + 3\sqrt{2c\deltabyz}\zeta + 3\sqrt{2c\deltabyz}\sqrt{\beta}\tilde{z}.
    \end{align}
\end{lemma}
\begin{proof}
    We start as follows 
    \begin{align}
        \|\overline{v}^t\|
        &\overset{(i)}{=}\|
        (1-\beta)\overline{v}^{t-1} + \beta\nabla f(x^t,\xi^t)\|\notag\\
        &= \|
        (1-\beta)(\overline{v}^{t-1} - \nabla f(x^{t-1})) + \beta(\nabla f(x^t) - \nabla f(x^{t-1}))  + \beta(\nabla f(x^t,\xi^t) - \nabla f(x^t)) + \nabla f(x^{t-1})\|\notag\\
        &\overset{(ii)}{\le} (1-\beta)\|\overline{v}^{t-1}-\nabla f(x^{t-1})\| 
        + \beta L\gamma\|g^{t-1}\|
        + \beta\|\theta^t\|
        + \|\nabla f(x^{t-1})\|.
    \end{align}
    where $(i)$ follows from the update rule of $\overline{v}^t$, $(ii)$ --- from triangle inequality and $L$-smoothness. We continue the derivations as follows 
    \begin{align}
        \|\overline{v}^t\| &\overset{(iii)}{\le} \sqrt{2L(f(x^{t-1})-f^\star)} + (1-\beta)\|\overline{v}^{t-1}-\nabla f(x^{t-1})\| + \beta\frac{\wtilde{c}}{\sqrt{G}} + \beta L\gamma\|\overline{v}^{t-1}\| + \beta L\gamma\|\overline{v}^{t-1} - g^{t-1}\|.
    \end{align}
    where $(iii)$ follows from $L$-smoothness, assumption 5 of the lemma, and triangle inequality. The term $\|\overline{v}^{t-1}-g^{t-1}\|$ can be bounded by $\sqrt{2c\deltabyz\zeta^2+2c\deltabyz\beta\tilde{c}^2}$ by \eqref{eq:vt_bar_gt} and assumptions 5, 6 of the lemma. We continue further bounding as follows
    \begin{align}
        \|\overline{v}^t\| &\overset{(iv)}{\le} \sqrt{4L\wtilde{\Delta}} + (1-\beta)\left(\sqrt{4L\wtilde{\Delta}} + \frac{3}{2}\frac{\wtilde{c}}{\sqrt{G}} + \frac{3}{2}\sqrt{2c\deltabyz}\zeta + \frac{3}{2}\sqrt{2c\deltabyz}\sqrt{\beta}\tilde{z}\right) + \beta\frac{\wtilde{c}}{\sqrt{G}}\notag\\
        &\quad+ L\gamma\beta\left(\sqrt{64L\wtilde{\Delta}} + 3\frac{\wtilde{c}}{\sqrt{G}} + 3\sqrt{2c\deltabyz}\zeta + 3\sqrt{2c\deltabyz}\sqrt{\beta}\tilde{z}\right) + \beta L\gamma\sqrt{2c\deltabyz}\zeta + \beta L\gamma\sqrt{2c\deltabyz}\sqrt{\beta}\tilde{z}\notag\\
        &= \sqrt{L\wtilde{\Delta}}(2 + 2(1-\beta) + 8L\gamma\beta)
        + \frac{\wtilde{c}}{\sqrt{G}}(\nicefrac{3}{2}(1-\beta) + \beta + 3L\gamma\beta) + \sqrt{2c\deltabyz}\zeta(\nicefrac{3}{2}(1-\beta) + 3L\gamma\beta + \beta L\gamma)\notag\\
        &\quad + \sqrt{2c\deltabyz}\sqrt{\beta}\tilde{z}(\nicefrac{3}{2}(1-\beta) + 3L\gamma\beta + \beta L \gamma).
    \end{align}
    For the first coefficient preceding $\wtilde{\Delta}$, we have 
    \[
    2+2(1-\beta) + 8L\gamma\beta \le 8 \Leftarrow 1-\beta + 4L\gamma\beta \le 3 \le L\gamma\beta \le 1,
    \]
    where the last inequality holds by the choice of $\beta$ and $\gamma$. For the second coefficient preceding $\frac{\wtilde{c}}{\sqrt{G}}$, we have 
    \[
    \frac{3}{2}(1-\beta) + \beta + 3L\gamma\beta \le 3 \Leftarrow 4L\gamma\beta \le \frac{3}{2}(1+\beta),
    \]
    where the last inequality holds by the choice of $\beta$ and $\gamma$. For the third coefficient preceding $\sqrt{2c\deltabyz}\zeta$, we have
    \[
    \frac{3}{2}(1-\beta) + 3L\gamma\beta + \beta L\gamma \le 3 \Leftarrow 3L\gamma\beta \le 1,
    \]
    where the last inequality holds by the choice of $\beta$ and $\gamma$. For the fourth coefficient preceding $\sqrt{2c\deltabyz}\sqrt{\beta}\tilde{z}$, we have
    \[
    \frac{3}{2}(1-\beta) + 3L\gamma\beta + \beta L\gamma \le 3\Leftarrow 4L\gamma\beta \le \frac{3}{2}(1+\beta),
    \]
    where the last inequality holds by the choice of $\beta$ and $\gamma$.
\end{proof}

\begin{lemma}\label{lem:bound_nabla_fi_xt_vi_t_nodp} Let each $f_i$ be $L$-smooth, $\wtilde{\Delta} \ge \Phi^0$. Assume that the following inequalities hold for the iterates generated by \algname{Byz-Clip21-SGD2M} 
\begin{enumerate}
    \item $\gamma \le \frac{1}{6L}$;
    \item $4L\gamma \le \beta$;
    \item $\beta\in[0,1];$
    \item $\frac{1}{G}\sum_{i\in\cG}\|v_i^{t-1}-\overline{v}^{t-1}\|^2 \le 2c\deltabyz\zeta^2+2c\deltabyz\beta\tilde{z}^2;$
    \item $\|\overline{v}^{t-1}\| \le
    \sqrt{64L\wtilde{\Delta}} 
    + 3\frac{\wtilde{c}}{\sqrt{G}}
    + 3\sqrt{2c\deltabyz}\zeta
    + 3\sqrt{2c\deltabyz}\sqrt{\beta}\tilde{z}$,
    
    \item $\|\nabla f(x^{t-1}) - \overline{v}^{t-1}\| \le
    \sqrt{4L\wtilde{\Delta}}
    + \frac{3}{2}\frac{\wtilde{c}}{\sqrt{G}}
    + \frac{3}{2}\sqrt{2c\deltabyz}\zeta
    + \frac{3}{2}\sqrt{2c\deltabyz}\sqrt{\beta}\tilde{z}$;
    \item $\|\theta^t\| \le \frac{\wtilde{c}}{\sqrt{G}}$;
    \item $\|\theta_i^t\| \le \wtilde{b}$ for all $i\in\cG;$
    \item $\Phi^{t-1} \le 2\wtilde{\Delta}.$
\end{enumerate}
Then we have 
\begin{align}
    \|\nabla f(x^t) - \overline{v}^t\| \le
    \sqrt{4L\wtilde{\Delta}} + \frac{3}{2}\frac{\wtilde{c}}{\sqrt{G}} + \frac{3}{2}\sqrt{2c\deltabyz}\zeta 
    + \frac{3}{2}\sqrt{2c\deltabyz}\sqrt{\beta}\tilde{z}.
\end{align}
\end{lemma}

\begin{proof}
    We have
    \begin{align}
        \|\nabla f(x^t) - v^t\| &\overset{(i)}{=} \|\nabla f(x^t) - (1-\beta)\overline{v}^{t-1} - \beta\nabla f_i(x^t,\xi^t)\|\notag\\
        &\overset{(ii)}{\le} (1-\beta)\|\nabla f(x^t) - \overline{v}^{t-1}\|
        + \beta\|\nabla f(x^{t}) - \nabla f(x^t,\xi^t)\|\notag\\
        &\overset{(iii)}{\le} (1-\beta)\|\nabla f(x^t) - \nabla f(x^{t-1})\|
        + (1-\beta)\|\nabla f(x^{t-1}) - \overline{v}^{t-1}\|
        + \beta\|\theta^t\|\notag\\
        &\overset{(iv)}{\le} (1-\beta)L\gamma\|g^{t-1}\|
        + (1-\beta)\|\nabla f(x^{t-1}) - \overline{v}^{t-1}\| 
        + \beta\|\theta^{t}\|\notag\\
        &\overset{(v)}{\le} (1-\beta)L\gamma\|\overline{v}^{t-1} - g^{t-1}\|
        + (1-\beta)L\gamma\|\overline{v}^{t-1}\|
        + (1-\beta)\|\nabla f(x^{t-1}) - \overline{v}^{t-1}\| 
        + \beta\|\theta^{t}\|,
    \end{align}
    where $(i)$ follows from the update rule of $v^t$, $(ii)$-$(iii)$ --- from the triangle inequality, $(iv)$ --- from the update rule of $x^t$ and $L$-smoothness, $(v)$ --- from triangle inequality. We continue as follows
    \begin{align*}
        \|\nabla f(x^t) - \overline{v}^t\| &\overset{(vi)}{\le} 
        \sqrt{2c\deltabyz}(1-\beta)L\gamma\zeta + \sqrt{2c\deltabyz}(1-\beta)L\gamma\sqrt{\beta}\tilde{z}
        + (1-\beta)L\gamma\|\overline{v}^{t-1}\|
        + \beta\|\theta^{t}\|
        \notag\\
        &\qquad +\; 
        (1-\beta)\|\nabla f(x^{t-1}) - \overline{v}^{t-1}\| \notag\\
        &\overset{(vii)}{\le} \sqrt{2c\deltabyz}(1-\beta)L\gamma\zeta + \sqrt{2c\deltabyz}(1-\beta)L\gamma\sqrt{\beta}\tilde{z}\notag\\
        &\quad 
        + (1-\beta)L\gamma\left(\sqrt{64L\wtilde{\Delta}} + 3\frac{\wtilde{c}}{\sqrt{G}} + 3\sqrt{2c\deltabyz}\zeta + 3\sqrt{2c\deltabyz}\sqrt{\beta}\tilde{z}\right)\notag\\
        &\quad + \beta \frac{\wtilde{c}}{\sqrt{G}} + (1-\beta)\left(\sqrt{4L\wtilde{\Delta}} + \frac{3}{2}\frac{\wtilde{c}}{\sqrt{G}} + \frac{3}{2}\sqrt{2c\deltabyz}\zeta + \frac{3}{2}\sqrt{2c\deltabyz}\sqrt{\beta}\tilde{z}\right)\notag\\
        &= \sqrt{L\wtilde{\Delta}}(8(1-\beta)L\gamma + 2(1-\beta)) 
        + \frac{\wtilde{c}}{\sqrt{G}}(3(1-\beta)L\gamma + \nicefrac{3}{2}(1-\beta)) 
        \notag\\
        &\quad + \sqrt{2c\deltabyz}\zeta((1-\beta)L\gamma + 3(1-\beta)L\gamma + \nicefrac{3}{2}(1-\beta))\notag\\
        &\quad + \sqrt{2c\deltabyz}\sqrt{\beta}\tilde{z}((1-\beta)L\gamma + 3(1-\beta)L\gamma + \nicefrac{3}{2}(1-\beta)).
    \end{align*}
    where $(vi)$ follows from \eqref{eq:vt_bar_gt} and assumptions 4, 8 of the lemma, $(vii)$ --- from assumptions 5, 6, 7 of the lemma.

    For the first coefficient preceding $\sqrt{L\wtilde{\Delta}}$, we have 
    \[
    8(1-\beta)L\gamma + 2(1-\beta) \le 2 \Leftarrow 4L\gamma \le \beta,
    \]
    where the last inequality holds by assumption 2 of the lemma. For the second coefficient preceding $\frac{\wtilde{c}}{\sqrt{G}}$, we have 
    \[
    3(1-\beta)L\gamma + \frac{3}{2}(1-\beta) \le \frac{3}{2} \Leftarrow 2L\gamma \le \beta,
    \]
    where the last inequality holds by assumption 2 of the lemma. For the third coefficient preceding $\sqrt{2c\deltabyz}\zeta$, we have 
    \[
    4(1-\beta)L\gamma + \frac{3}{2}(1-\beta) \le \frac{3}{2} \Leftarrow 4L\gamma \le \frac{3}{2}\beta \Leftarrow \frac{8}{3}L\gamma \le \beta,
    \]
    where the last inequality holds by assumption 2 of the lemma. Finally, or the fourth coefficient preceding $\sqrt{2c\deltabyz}\sqrt{\beta}\tilde{z}$, we have 
    \[
    4(1-\beta)L\gamma + \frac{3}{2}(1-\beta) \le \frac{3}{2} \Leftarrow 4L\gamma \le \frac{3}{2}\beta \Leftarrow \frac{8}{3}L\gamma \le \beta,
    \]
    where the last inequality holds by assumption 2 of the lemma. This finalizes the proof.
\end{proof}

\subsubsection{Main Convergence Theorem}

\begin{theorem}[Full statement of \Cref{th:nodp_withbyz}] Let Assumptions~\ref{asmp:smoothness}, \ref{asmp:stoch_grad}, and \ref{asmp:bounded_heterogeneity_avg} hold. Let the failure probability $\alpha\in(0,1)$, and constants $\wtilde{b}$ and $\wtilde{c}$ be defined as in \eqref{eq:constants_nodp}, and $\wtilde{\Delta} \ge \wtilde{\Phi}^0$ for $\wtilde{\Phi}^0$ defined in \eqref{eq:lyapunov_function_nodp}. Consider the run of \algname{Byz-Clip21-SGD2M} with $\sigma_{\omega}=0$ and $\tau=+\infty$ (\Cref{alg:byz_clip21_sgd2m_nodp}) for T iterations. Assume that the following inequalities are satisfied 
\begin{enumerate}
    \item $\gamma \le \frac{1}{6L}$,
    \item $\beta \ge 4L\gamma$,
    \item and momentum restrictions defined in \eqref{eq:step-size_bound_1_nodp}, \eqref{eq:step-size_bound_4_nodp}, and \eqref{eq:step-size_bound_8_nodp}.
\end{enumerate}
Then, with probability at least $1-\alpha$, we bound $\frac{1}{T}\sum_{t=0}^{T-1} \|\nabla f(x^t)\|^2$ with 
\begin{align}
        \wtilde{\cO}\left(\frac{L\wtilde{\Delta}}{T} 
        + \left(\frac{\sigma^2L\wtilde{\Delta}}{GT}\right)^{1/2} 
        + \frac{\sigma(\sqrt{L\wtilde{\Delta}} + \sigma/\sqrt{G} + \sqrt{c\deltabyz}\zeta + \sqrt{c\deltabyz}\sigma)}{\sqrt{GT}}
        + c\deltabyz \zeta^2\right),
    \end{align}
    where we can choose $\wtilde{\Delta} = 2(f(x^0)-f^\star) = 2F^0.$
\end{theorem}
\begin{proof}
    For convenience, we define $\nabla f_i(x^{-1},\xi^{-1}_i) = v_i^{-1} = g_i^{-1} = 0, \Phi^{-1} = \Phi^0$. Next, let us define an event $\wtilde{E}^t$ for each $t\in\{0,\dots,T\}$ such that the following inequalities hold for all $k\in\{0,\dots,t\}$
    \begin{enumerate}
        \item $\frac{1}{G}\sum_{i\in\cG}\|\overline{v}^t - v_i^t\| \le 2\zeta^2+2\beta\tilde{z}^2$;
        \item $\|\theta^k_i\|\le b$ for all $i\in\cG$ and $\|\theta^k\|\le \frac{\wtilde{c}}{\sqrt{G}}$;
        \item $\|\overline{v}^{t-1}\| \le
    \sqrt{64L\wtilde{\Delta}} 
    + 3\frac{\wtilde{c}}{\sqrt{G}}
    + 3\sqrt{2c\deltabyz}\zeta
    + 3\sqrt{2c\deltabyz}\sqrt{\beta}\tilde{z}$,
    
    \item $\|\nabla f(x^{t-1}) - \overline{v}^{t-1}\| \le
    \sqrt{4L\wtilde{\Delta}}
    + \frac{3}{2}\frac{\wtilde{c}}{\sqrt{G}}
    + \frac{3}{2}\sqrt{2c\deltabyz}\zeta
    + \frac{3}{2}\sqrt{2c\deltabyz}\sqrt{\beta}\tilde{z}$;
        \item $\wtilde{\Phi}^k \le 2\wtilde{\Delta}$;
        \item \begin{align*}
        \frac{1}{2}\Delta & \ge 2\gamma(1-\beta)\sum_{l=0}^{k-1}\<\nabla f(x^l) - \nabla f(x^{l+1}), \theta^{l+1}>.
        \end{align*}
    \end{enumerate}
    Then, we will derive the result by induction, i.e., using the induction w.r.t.\ $t$, we will show that $\Prob(\wtilde{E}^t) \ge 1-\frac{\alpha(t+1)}{T+1}$ for all $t\in\{0,\dots, T-1\}$.
    
    Before moving on to the proof's induction part, we need to establish several useful bounds. Denote the events $\wtilde{\Theta}^t_i$ and $\wtilde{\Theta}^t$ as
    \begin{align}
        &\wtilde{\Theta}^t_i \eqdef \{\|\theta^t_i\| \ge \wtilde{b} \}, \quad \wtilde{\Theta}^t \eqdef \left\{ \|\theta^t\|\ge \frac{\wtilde{c}}{\sqrt{G}} \right\}, \quad 
        \hat{\Theta}^t_i \eqdef \left\{\left\|\sum_{k=0}^t(1-\beta)^{t-k}(\theta_i^k-\theta^k)\right\| \ge \frac{\tilde{z}}{\sqrt{\beta}}\right\}
    \end{align}
    respectively. From \Cref{asmp:stoch_grad} we have 
    \[\Prob(\wtilde{\Theta}^{t}_i) \le 2\exp\left(-\frac{\wtilde{b}^2}{2\sigma^2}\right) = \frac{\alpha}{8(T+1)G}\] 
    where the last equality is by definition of $\wtilde{b}^2$. Therefore, $\Prob(\overline{\wtilde{\Theta}}^t_i) \ge 1 - \frac{\alpha }{8G(T+1)}.$ 
    Besides, notice that the constant $\wtilde{c}$ in (\ref{eq:constants_nodp}) can be viewed as 
    \[
        \wtilde{c}^2 = (\sqrt{2}+2\wtilde{b}_3)\sigma^2 \quad\text{where} \quad \wtilde{b}_3^2 = 3\log\frac{8(T+1)}{\alpha}.
    \]
    Now, we can use \Cref{lem:concentration_lemma} to bound $\Prob(\wtilde{\Theta}^t).$ Since all $\theta^t_i$ are independent $\sigma$-sub-Gaussian random vectors, then we have
    \[
    \Prob\left(\left\|\sum_{i\in\cG}\theta^t_i\right\|\ge \wtilde{c}\sqrt{G}\right) = \Prob\left(\|\theta^t\| \ge \frac{\wtilde{c}}{\sqrt{G}}\right) \le \exp(-\wtilde{b}_3^2/3) = \frac{\alpha}{8(T+1)}.
    \]
    The bound $\Prob(\hat{\Theta}^t_i) \le \frac{\alpha}{6G(T+1)}$ was shown before in \Cref{lem:bound_vt_vt_average}.
    
    Now, we are ready to prove that $\Prob(\wtilde{E}^t) \ge 1-\frac{\alpha(t+1)}{T+1}$ for all $t\in\{0,\dots, T-1\}.$ First, we show that the base of induction holds.

    \paragraph{Base of induction.}

    \begin{enumerate}

        \item To establish the first, we use \Cref{lem:bound_vt_vt_average}, which guarantees the bound in $\cap_{i\in\cG}\overline{\hat{\Theta}^t}$, i.e., with probability $1-\frac{\alpha}{8(T+1)}$.
 
        \item $\overline{v}^0 = \frac{1}{G}\sum_{i\in\cG} ((1-\beta)v_i^{-1} + \beta\nabla f_i(x^0,\xi^0_i)) = \frac{\beta}{G}\nabla f(x^0,\xi^0).$ Therefore, we have
        \begin{align*}
            \|\overline{g}^0\| &= \beta\|\nabla f(x^0,\xi^0)\|\\
            &\le \beta\|\nabla f(x^0)\| + \beta \|\nabla f(x^0,\xi^0)-\nabla f(x^0)\|\\
            &\le \beta\sqrt{2L(f(x^0)-f^\star)} + \beta\frac{\wtilde{c}}{\sqrt{G}}\\
            &\le \beta\sqrt{2L\wtilde{\Phi}^0}+ \beta\frac{\wtilde{c}}{\sqrt{G}}\\
            &\le \beta\sqrt{4L\wtilde{\Delta}}+ \beta\frac{\wtilde{c}}{\sqrt{G}}\\
            &\le \sqrt{64L\wtilde{\Delta}} + \frac{3\wtilde{c}}{\sqrt{G}} + 3\sqrt{2c\deltabyz}\zeta + 3\sqrt{2c\deltabyz}\sqrt{\beta}\tilde{z}.
        \end{align*}

        The inequalities above again hold in $\overline{\wtilde{\Theta}^0}$, i.e., with probability at least $1-\frac{\alpha}{8(T+1)}.$ Therefore, the condition $3$ of the induction is verified. 

        \item We have 
        \begin{align*}
            \|\overline{v}^0 - \nabla f(x^0)\| &= \|\beta\nabla  f(x^0,\xi^0) - \nabla f(x^0)\|\\ 
            &\le \beta\|\nabla f(x^0, \xi^0) - \nabla f(x^0)\| + (1-\beta)\|\nabla f(x^0)\|\\
            &\le \beta b + (1-\beta)\sqrt{4L\wtilde{\Delta}}.
        \end{align*}
        The bound above holds with probability at least $1-\frac{\alpha}{8(T+1)}$ because it holds in $\cap_{i\in\cG}\overline{\wtilde{\Theta}_i^0}.$ Therefore, the bound $4$ of the assumption of the induction is verified. 
        
        \item Next, we emphasize that the condition $8$ of the induction assumption also hold, as $\wtilde{\Phi}^0 \le 2\wtilde{\Phi}^0 \le 2\wtilde{\Delta}$ by the choice of $\wtilde{\Delta}$.

        \item We finalize the induction base by noting that the condition $6$ of the induction assumption holds since the RHS equals $0$.

    \end{enumerate}
    Therefore, we conclude that the conditions $1$-$8$ hold with a probability of at least 
    \begin{align*}
        \Prob\left(\overline{\wtilde{\Theta}^0}\cap \left(\cap_{i\in\cG}\overline{\wtilde{\Theta}_i^0}\right) \cap \left(\cap_{i\in\cG}\overline{\hat{\Theta}^0_i}\right)\right) &\ge 1 - \Prob(\wtilde{\Theta}^0) - \sum_{i\in\cG} \Prob(\wtilde{\Theta}_i^0) - \sum_{i\in\cG}\Prob(\hat{\Theta}_i^0)\\
        &\ge 
        1 
        - \frac{\alpha}{8(T+1)}
        - G\cdot \frac{\alpha}{8G(T+1)} 
        - G\cdot \frac{\alpha}{8G(T+1)}\\
        &= 1-\frac{\alpha}{2(T+1)} > 1- \frac{\alpha}{T+1},
    \end{align*}
    i.e., $\Prob(E^0) \ge 1-\frac{\alpha}{T+1}$ holds. This is the base of the induction.

    \paragraph{Transition step of induction.}

    Assume that all events $\overline{\wtilde{\Theta}^{K+1}}, \overline{\wtilde{\Theta}^{K+1}_i}$ take place, i.e., $\|\theta^{K+1}_i\| \le \wtilde{b}, \|\theta^{K+1}\|\le \frac{\wtilde{c}}{\sqrt{G}}$ for all $i\in\cG$. That is, we assume that the event 
    $$\overline{\wtilde{\Theta}^{K+1}} \cap \left(\cap_{i\in\cG} \cap \overline{\wtilde{\Theta}^{K+1}_i}\right) \cap \left(\cap_{i\in\cG}\overline{\hat{\Theta}_i^{K+1}}\right) \cap \wtilde{E}^K$$ holds. Then, by the assumptions of the induction and from \Cref{lem:bound_vt_bar_nodp}, we get that 
    \[
    \|\overline{v}^{K+1}\| \le \sqrt{64L\wtilde{\Delta}} + 3\frac{\wtilde{c}}{\sqrt{G}} + 3\sqrt{2c\deltabyz}\zeta + 3\sqrt{2c\deltabyz}\sqrt{\beta}\tilde{z},
    \]
    from \Cref{lem:bound_nabla_fi_xt_vi_t_nodp} we get that
    \[
    \|\nabla f(x^{K+1}) - v^{K+1}\| \le \sqrt{4L\wtilde{\Delta}} + \frac{3}{2}\frac{\wtilde{c}}{\sqrt{G}} + \frac{3}{2}\sqrt{2c\deltabyz}\zeta + \frac{3}{2}\sqrt{2c\deltabyz}\sqrt{\beta}\tilde{z},
    \]
    from \Cref{lem:bound_vt_bar_nodp} we get 
    \[
    \frac{1}{G}\sum_{i\in\cG}\|\overline{v}^{K+1}-v^{K+1}_i\|^2 \le 2\zeta^2+2\beta\tilde{z}^2.
    \]
    This means that conditions 1-5 in the induction assumption are also verified for the step $K+1$. Since for all $t\in\{0, \dots, K+1\}$ inequalities $1$-$7$ are verified, we can write for each $t\in \{0,\ldots,K\}$ by \Cref{lem:descent_lemma_in_f_no_dp,lem:descent_Pt_tilde} the following

    \begin{align*}
        \wtilde{\Phi}^{t+1} &= \delta^{t+1} 
        + \frac{\gamma}{\beta}\wtilde{P}^{t+1}\\
        &\le \delta^t - \frac{\gamma}{2}\|\nabla f(x^t)\|^2  {\color{red} - \frac{1}{4\gamma}R^t}
       {\color{orange}+ \gamma \wtilde{P}^t}
       + 2c\deltabyz\gamma(\zeta^2+\beta\tilde{z}^2)\\
    &\;+\; \frac{\gamma}{\beta}\left({\color{orange} (1-\beta)\wtilde{P}^t }
    {\color{red} 
    + \frac{3L^2}{\beta}R^t }
    + \beta^2\frac{\wtilde{c}^2}{G}
    + 2\beta(1-\beta)\<\overline{v}^t - \nabla f(x^{t+1}), \theta^{t+1}>\right).
    \end{align*}
    
    Rearranging terms, we get
    \begin{align*}
    \wtilde{\Phi}^{t+1} &\le \delta^t 
    - \frac{\gamma}{2}\|\nabla f(x^t)\|^2
    + \frac{\gamma}{\beta}\wtilde{P}^t\left(\beta + 1-\beta\right) 
    + 2c\deltabyz\gamma(\zeta^2+\beta\tilde{z}^2)
    - \frac{1}{4\gamma}R^t\left(1 - \frac{12L^2}{\beta^2}\gamma^2\right)
    + \wtilde{c}^2\frac{\gamma\beta}{G}\\
    &\quad + 2\gamma(1-\beta)\<\nabla f(x^t) - \nabla f(x^{t+1}), \theta^{t+1}>.
    \end{align*}
    Using step-size restriction $(ii)$, we get rid of the term with $R^t$ and obtain
    \begin{align*}
    \wtilde{\Phi}^{t+1} &\le \delta^t 
    - \frac{\gamma}{2}\|\nabla f(x^t)\|^2
    + \frac{\gamma}{\beta}\wtilde{P}^t
    + 2c\deltabyz\gamma(\zeta^2+\beta\tilde{z}^2)
    + \wtilde{c}^2\frac{\gamma\beta}{G}
    + 2\gamma(1-\beta)\<\nabla f(x^t) - \nabla f(x^{t+1}), \theta^{t+1}>.
    \end{align*}
    Now we sum all the inequalities above for $t\in\{0,\dots,K\}$ and get 
    \begin{align}
    \wtilde{\Phi}^{K+1} 
    &\le \wtilde{\Phi}^0
    - \frac{\gamma}{2}\sum_{t=0}^{K}\|\nabla f(x^t)\|^2
    + K\wtilde{c}^2\frac{\gamma\beta}{G} 
    + 2c\deltabyz\gamma(\zeta^2+\beta\tilde{z}^2)K
    + 2\gamma(1-\beta)\sum_{t=0}^{K}\<\nabla f(x^t) - \nabla f(x^{t+1}), \theta^{t+1}>.\label{eq:njqnsfqknjaaaaaaa}
    \end{align}
    Rearranging terms, we get 
    \begin{align*}
    &\frac{\gamma}{2}\sum_{t=0}^{K}\|\nabla f(x^t)\|^2
    \le \wtilde{\Phi}^0 - \wtilde{\Phi}^{K+1}
    + K\wtilde{c}^2\frac{\gamma\beta}{G}
    + 2c\deltabyz\gamma(\zeta^2+\beta\tilde{z}^2)K
    + 2\gamma(1-\beta)\sum_{t=0}^{K}\<\nabla f(x^t) - \nabla f(x^{t+1}), \theta^{t+1}>.
    \end{align*}
    Taking into account that $\frac{\gamma}{2}\sum_{t=0}^{K}\|\nabla f(x^t)\|^2 \ge 0$, we get that the event $\wtilde{E}^K\cap \left(\cap_{i\in\cG}\overline{\wtilde{\Theta}^{K+1}_i}\right)\cap\left(\cap_{i\in\cG}\overline{\hat{\Theta}_i^t}\right)\cap \overline{\wtilde{\Theta}^{K+1}}$ implies 
    \begin{align*}
      \wtilde{\Phi}^{K+1} \le \wtilde{\Phi}^0 
    + K\wtilde{c}^2\frac{\gamma\beta}{G}
    + 2c\deltabyz\gamma(\zeta^2+\beta\tilde{z}^2)K
    + 2\gamma(1-\beta)\sum_{t=0}^{K}\<\nabla f(x^t) - \nabla f(x^{t+1}), \theta^{t+1}>.
    \end{align*}
    Next, we define the following random vector:
    \begin{align*}
    \zeta_{5}^t \eqdef \begin{cases} 
    \nabla f(x^t) - \nabla f(x^{t+1}), &\text{ if } \|\nabla f(x^t) - \nabla f(x^{t+1})\| \le L\gamma\left(\sqrt{64L\wtilde{\Delta}} + 3\frac{\wtilde{c}}{\sqrt{G}} + 3\sqrt{2c\deltabyz}\zeta + 3\sqrt{2c\deltabyz}\sqrt{\beta}\tilde{z}\right)\\
	0, &\text{otherwise}
    \end{cases}.
    \end{align*}
    By definition, the introduced random vector $\zeta_{5}^t$ is bounded with probability $1$. Moreover, by the definition of $\wtilde{E}^t$ we get that the event $\wtilde{E}^K\cap \overline{\wtilde{\Theta}^{K+1}}\cap \left(\cap_{i\in\cG}\overline{\wtilde{\Theta}^{K+1}_i}\right)\cap\left(\cap_{i\in\cG}\overline{\hat{\Theta}_i^t}\right)$ implies 
    \begin{align*}
    \zeta_{5}^t = \nabla f(x^t) - \nabla f(x^{t+1}).
    \end{align*}
    Therefore, the event $\wtilde{E}^K\cap \overline{\wtilde{\Theta}^{K+1}} \cap \left(\cap_{i\in\cG}\overline{\wtilde{\Theta}^{K+1}_i}\right) \cap\left(\cap_{i\in\cG}\overline{\hat{\Theta}_i^t}\right)$ implies 
    \begin{align*}
    \wtilde{\Phi}^{K+1} \le \wtilde{\Phi}^0 
    + \underbrace{K\wtilde{\wtilde{c}}^2\frac{\gamma\beta}{G}}_{I}
    + \underbrace{\frac{2\gamma(1-\beta)}{G}\sum_{t=0}^{K}\<\zeta_{5}^t, \theta^{t+1}>}_{II}
    + \underbrace{2c\deltabyz\gamma(\zeta^2+\beta\tilde{z}^2)K}_{III}.
    \end{align*}
    
    \subparagraph{Bound of the term I.} Since $4L\gamma \leq \beta$, for the term $I$ we have 
    \begin{align*}
        K\wtilde{c}^2\frac{\gamma\beta}{G}\le 
        K\wtilde{c}^2\frac{\beta^2}{4LG}.
    \end{align*}
    By choosing $\beta$ such that 
    \begin{equation}\label{eq:step-size_bound_1_nodp}
    \beta \le 
    \left(\frac{4L\wtilde{\Delta} G}{3T\wtilde{c}^2}\right)^{1/2},
    \end{equation}
    we get that 
    \[
        K\wtilde{c}^2\frac{\gamma\beta}{G} \le  \frac{\wtilde{\Delta}}{2}.
    \]
    This bound holds with probability $1.$ Note that the worst dependency in the restriction on $\beta$ w.r.t. $T$ is $\cO(\nicefrac{1}{T^{1/2}})$.

    \subparagraph{Bound of the term III.} Since $4L\gamma \leq \beta$, for the term $III$ we have 
    \begin{align*}
        2c\deltabyz\gamma(\zeta^2+\beta\tilde{z}^2)K\le 
       \frac{1}{2L}c\deltabyz\beta(\zeta^2+\beta\tilde{z}^2)K.
    \end{align*}
    By choosing $\beta$ such that 
    \begin{equation}\label{eq:step-size_bound_4_nodp}
    \beta \le \min\left\{\frac{L\wtilde{\Delta}}{3c\deltabyz\zeta^2T}, \left(\frac{L\wtilde{\Delta}}{3c\deltabyz T\tilde{z}^2}\right)^{1/2}\right\},
    \end{equation}
    we get that 
    \[
        \frac{1}{2L}c\deltabyz\beta(\zeta^2+\beta\tilde{z}^2)K \le  \frac{\wtilde{\Delta}}{3}.
    \]
    This bound holds with probability $1.$ Note that the worst dependency in the restriction on $\beta$ w.r.t. $T$ is $\cO(\nicefrac{1}{T})$.

     \subparagraph{Bound of the term II.} The bound in this case is obtained by using concentration inequality \Cref{lem:concentration_lemma}. We define
    \[
        \sigma^2_8 \eqdef 4L^2\gamma^2\cdot \left(\sqrt{64L\wtilde{\Delta}} + 3\frac{\wtilde{c}}{\sqrt{G}}+ 3\sqrt{2c\deltabyz}\zeta + 3\sqrt{2c\deltabyz}\sqrt{\beta}\tilde{z}\right)^2\cdot \frac{\sigma^2}{G}.
    \]
    Then we have 
    \begin{align*}
    &\E{\exp\left(\left|\frac{4L^2\gamma^2}{\sigma^2_8}4(1-\beta)^2\<\zeta^l_{5},\theta^{l+1}>^2\right|\right)\mid l}\\
    &\le 
    \E{\exp\left(\frac{4L^2\gamma^2}{\sigma^2_8}\|\zeta_{5}^l\|^2 \cdot \|\theta^{l+1}\|^2\right)\mid l}\\
    &\le \EE\left[\exp\left(\frac{4L^2\gamma^2}{\sigma^2_8}\left(\sqrt{64L\wtilde{\Delta}} + 3\frac{\wtilde{c}}{\sqrt{G}} + 3\sqrt{2c\deltabyz}\zeta + 3\sqrt{2c\deltabyz}\sqrt{\beta}\tilde{z}\right) \cdot \|\theta^{l+1}\|^2\right)^2\mid l\right].
    \end{align*}
    Since $\theta^{l+1}$ is sub-Gaussian with parameter $\frac{\sigma^2}{G}$, then we can continue the chain of inequalities above using the definition of $\sigma_8^2$
    \begin{align*}
        &\mathbb{E}\left[\exp\left(\left[4L^2\gamma^2\cdot \left(\sqrt{64L\wtilde{\Delta}} +3\frac{\wtilde{c}}{\sqrt{G}}+ 3\sqrt{2c\deltabyz}\zeta + 3\sqrt{2c\deltabyz}\sqrt{\beta}\tilde{z}\right)^2\cdot \frac{\sigma^2}{G}\right]^{-1}\cdot   \right.\right.\\
        &\qquad \left.\left. 4L^2\gamma^2\cdot \left(\sqrt{64L\wtilde{\Delta}} +3\frac{\wtilde{c}}{\sqrt{G}}+ 3\sqrt{2c\deltabyz}\zeta + 3\sqrt{2c\deltabyz}\sqrt{\beta}\tilde{z}\right)^2\cdot\|\theta^{l+1}\|^2\right)\mid l\right]\\
        &= \E{\exp\left(\frac{\|\theta^{l+1}\|^2}{\sigma^2/G}\right)} \le \exp(1).
    \end{align*}
    Therefore, we have by \Cref{lem:concentration_lemma} that 
    \begin{align*}
    &\Pr\left[2\gamma(1-\beta)\left\|\sum_{t=0}^K\<\zeta_{5}^t,\theta^{t+1}>\right\|\right.\\ 
    &\ge \left. (\sqrt{2}+\sqrt{2}b_1)\sqrt{\sum_{t=0}^K4L^2\gamma^2\frac{\sigma^2}{G} \left(\sqrt{64L\wtilde{\Delta}} +3\frac{\wtilde{c}}{\sqrt{G}}+ 3\sqrt{2c\deltabyz}\zeta + 3\sqrt{2c\deltabyz}\sqrt{\beta}\tilde{z}\right)^2}\right]\\ 
    &\le \exp(-\nicefrac{\wtilde{b}_1^2}{3}) = \frac{\alpha}{6(T+1)}, \quad \text{where} \quad \wtilde{b}_1^2 \eqdef 3\log\left(\frac{6(T+1)}{\alpha}\right)
    \end{align*}
    Using the restrictions $4L\gamma\le \beta$, we get
    \begin{align*}
    &(\sqrt{2}+\sqrt{2}b_1)\sqrt{(K+1)} \cdot \frac{2L\gamma}{\sqrt{G}}\sigma\left(\sqrt{64L\wtilde{\Delta}} +3\frac{\wtilde{c}}{\sqrt{G}} + 3\sqrt{2c\deltabyz}\zeta + 3\sqrt{2c\deltabyz}\sqrt{\beta}\tilde{z}\right)\\
    \le\; 
    & (\sqrt{2}+\sqrt{2}b_1)\sqrt{(K+1)} \cdot \frac{\beta\sigma}{2L\sqrt{G}}\left(\sqrt{64L\wtilde{\Delta}} +3\frac{\wtilde{c}}{\sqrt{G}} + 3\sqrt{2c\deltabyz}\zeta + 3\sqrt{2c\deltabyz}\tilde{z}\right)\\
    \le\; &\frac{\wtilde{\Delta}}{3} 
    \end{align*}
    because we choose $\beta$
    \begin{align}\label{eq:step-size_bound_8_nodp}
    \beta &\le \left(\frac{2L\wtilde{\Delta}\sqrt{G}}{3\sqrt{2}(1+\wtilde{b}_1)\sigma\sqrt{T}\left(\sqrt{64L\wtilde{\Delta}} +3\frac{\wtilde{c}}{\sqrt{G}} + 3\sqrt{2c\deltabyz}\zeta + 3\sqrt{2c\deltabyz}\tilde{z}\right)}\right),\\
    & \text{and} \quad K+1\le T.\notag
    \end{align}
    This implies 
    \begin{align*}
    &\Prob\left(2\gamma(1-\beta)\left\|\sum_{t=0}^K\<\zeta_{5,i}^t,\theta^{t+1}>\right\| \ge \frac{\wtilde{\Delta}}{3}\right) \le \frac{\alpha}{6(T+1)}.
    \end{align*}
    Note that the worst dependency in the choice of $\beta$ w.r.t $T$ is $\wtilde{\cO}(\nicefrac{1}{T^{1/2}}).$

    \paragraph{Final probability.}
    Therefore, the probability event 
    \[
    \Omega \eqdef \wtilde{E}^K 
    \cap \overline{\wtilde{\Theta}^{K+1}}
    \cap \left(\cap_{i\in\cG}\overline{\wtilde{\Theta}^{K+1}_i}\right)
    \cap \left(\cap_{i\in\cG} \overline{\hat{\Theta}_i^{K+1}}\right)
    \cap E_{I}
    \cap E_{II}
    \cap E_{III},
    \]
    where $E_I$-$E_{III}$ denotes that each of the terms $I, II,$ and $III$ is smaller than $\frac{\wtilde{\Delta}}{3}$. This implies that 
    \[
    I + II + III\le \wtilde{\Delta},
    \]
    i.e., condition $6$ in the induction assumption holds. Moreover, this also implies that 
    \[
    \wtilde{\Phi}^{K+1} \le \wtilde{\Phi}^0 + \wtilde{\Delta}\le \wtilde{\Delta} + \wtilde{\Delta} = 2\wtilde{\Delta},
    \]
    i.e., condition $5$ in the induction assumption holds. The probability $\Prob(\wtilde{E}_{K+1})$ can be lower bounded as follows 
    \begin{align*}
    \Prob(\wtilde{E}_{K+1}) &\ge \Prob(\Omega)\\
    &= \Prob\left(\wtilde{E}_K 
    \cap \overline{\wtilde{\Theta}^{K+1}}
    \cap \left(\cap_{i\in\cG}\overline{\wtilde{\Theta}^{K+1}_i}\right)
    \cap \left(\cap_{i\in\cG} \overline{\hat{\Theta}_i^{K+1}}\right)
    \cap E_{I}
    \cap E_{II}
    \cap E_{III}\right)
    \\
    &= 1 - \Prob\left(\overline{\wtilde{E}}_K\cup
    \wtilde{\Theta}^{K+1} \cup 
    \left(\cup_{i\in\cG}\wtilde{\Theta}^{K+1}_i\right) \cup
    \left(\cup_{i\in\cG}\hat{\Theta}^{K+1}_i\right)
    \cup \overline{E}_{I}
    \cup \overline{E}_{II}
    \cup \overline{E}_{III}\right)\\
    &\ge 1 - \Prob(\overline{\wtilde{E}_K}) 
    - \Prob(\wtilde{\Theta}^{K+1})
    - \sum_{i\in\cG}\Prob(\wtilde{\Theta}^{K+1}_i)
    - \sum_{i\in\cG}\Prob(\hat{\Theta}^{K+1}_i)
    - \underbrace{\Prob(\overline{E}_{I})}_{=0}
    - \Prob(\overline{E}_{II})
    - \underbrace{\Prob(\overline{E}_{III})}_{=0}\\
    &\ge 1 - \frac{\alpha(K+1)}{T+1}
    - \frac{\alpha}{8(T+1)}
    - \sum_{i\in\cG} \frac{\alpha}{8G(T+1)}
    - \sum_{i\in\cG} \frac{\alpha}{8G(T+1)}
    - \frac{\alpha}{8(T+1)}\\
    &\ge 1-\frac{\alpha(K+2)}{T+1}.
    \end{align*}
    This finalizes the transition step of induction. The result of the theorem follows by setting $K=T-1$. Indeed, from (\ref{eq:njqnsfqknj}) we obtain
    \begin{align}\label{eq:mfweodjnwejaaaa}
        \frac{\gamma}{2}\sum_{t=0}^K\|\nabla f(x^t)\|^2 
        \le \wtilde{\Phi}^0 
        - \wtilde{\Phi}^{K+1} 
        + \wtilde{\Delta} 
        \le 2\wtilde{\Delta} 
        \Rightarrow \frac{1}{T}\sum_{t=0}^{T-1}\|\nabla f(x^t)\|^2 
        \le \frac{4\wtilde{\Delta}}{\gamma T}.
    \end{align}

\begin{figure}[t]
\centering

\begin{minipage}[t]{0.48\textwidth}
\begin{algorithm}[H]
\caption{\algname{Byz-Clip-SGD}}
\label{alg:byz_clip_sgd}
\begin{algorithmic}[1]
\STATE \textbf{Input:} $x^0 \in X,$ step-size $\gamma>0$, clipping parameter $\tau>0$, DP-noise variance $\sigma_{\omega}^2\ge0$
\STATE
\FOR{$t=0, \ldots, T-1$}
  \STATE $x^{t+1} = x^t - \gamma g^t$
  \FOR{$i\in\cG$}
    \STATE $g_i^{t+1} = \clip_{\tau}(\nabla f_i(x^{t+1}, \xi^{t+1}_i))$
    \STATE $\omega_i^{t+1} \sim \cN(0, \sigma_{\omega}^2\mI)$
    \STATE $m_i^{t+1} = g_i^{t+1} + \omega_i^{t+1}$
  \ENDFOR
  \FOR{$i\in\cB$}
    \STATE $m_i^{t+1} = (*)$ \hspace{2cm} sends arbitrary vector
  \ENDFOR
  \STATE $g^{t+1} = \aragg(m_1^{t+1},\dots,m_n^{t+1})$
\ENDFOR
\end{algorithmic}
\end{algorithm}
\end{minipage}\hfill
\begin{minipage}[t]{0.48\textwidth}
\begin{algorithm}[H]
\caption{\algname{Safe-DSHB}}
\label{alg:safe_dshb}
\begin{algorithmic}[1]
\STATE \textbf{Input:} $x^0 \in X,$ momentum $\beta \in(0,1],$ step-size $\gamma>0$, $m_i^0\in\R^d,$ clipping $\tau>0$, DP-noise variance $\sigma_{\omega}^2\ge0$
\FOR{$t=0, \ldots, T-1$}
  \STATE $x^{t+1} = x^t - \gamma g^t$
  \FOR{$i\in\cG$}
    \STATE $g_i^{t+1} = \clip_{\tau}(f_i(x^{t+1}, \xi^{t+1}_i))$
    \STATE $\omega_i^{t+1} \sim \cN(0, \sigma_{\omega}^2\mI)$
    \STATE $m_i^{t+1} = (1-\beta)m_i^t + \beta (g_i^{t+1} + \omega_i^{t+1})$
  \ENDFOR
  \FOR{$i\in\cB$}
    \STATE $m_i^{t+1} = (*)$ \hspace{2cm} sends arbitrary vector
  \ENDFOR
  \STATE $g^{t+1} = \aragg(m_1^{t+1},\dots,m_n^{t+1})$
\ENDFOR
\end{algorithmic}
\end{algorithm}
\end{minipage}

\end{figure}

    \paragraph{Final rate.}

    We highlight that we are interested in the functional dependency of the rate on the problem constants. Therefore, in the rest of the proof, we omit using numerical constants. Translating momentum restriction to step-size restriction 
    \begin{align}\label{eq:ninwofejnasdadweaaaaa}
        \gamma = &\frac{1}{L}\wtilde{\cO}\left(\min\left\{1, 
         \underbrace{\left(\frac{L\wtilde{\Delta} G}{T\sigma^2}\right)^{1/2} }_{\text{from term I}~\eqref{eq:step-size_bound_1_nodp}}, 
         \underbrace{\left(\frac{L\wtilde{\Delta} \sqrt{G}}{\sigma\sqrt{T}(\sqrt{L\wtilde{\Delta}} + \sigma/\sqrt{G} + \sqrt{c\deltabyz}\zeta + \sqrt{c\deltabyz}\sigma}\right)}_{\text{from term II}~\eqref{eq:step-size_bound_8_nodp}},
         \underbrace{\left(\frac{L\wtilde{\Delta}}{c\deltabyz T\zeta^2}\right) }_{\text{from term III}~\eqref{eq:step-size_bound_4_nodp}}, \right\}\right),
    \end{align} 
    We plug in the restrictions \eqref{eq:ninwofejnasdadweaaaaa} on the step-size into \eqref{eq:mfweodjnwejaaaa} and obtain the rate 
    \begin{align}
        \frac{1}{T}\sum_{t=0}^{T-1}\|\nabla f(x^t)\|^2 &\le \frac{L\wtilde{\Delta}}{T}\wtilde{\cO}\left(1+ \left(\frac{T\sigma^2}{L\wtilde{\Delta}G}\right)^{1/2} 
        + \frac{\sigma\sqrt{T}(\sqrt{L\wtilde{\Delta}} + \sigma/\sqrt{G} + \sqrt{c\deltabyz}\zeta + \sqrt{c\deltabyz}\sigma)}{L\wtilde{\Delta}\sqrt{G}}
        + \frac{c\deltabyz T\zeta^2}{L\wtilde{\Delta}}\right)\notag\\
        &= \wtilde{\cO}\left(\frac{L\wtilde{\Delta}}{T} 
        + \left(\frac{\sigma^2L\wtilde{\Delta}}{GT}\right)^{1/2} 
        + \frac{\sigma(\sqrt{L\wtilde{\Delta}} + \sigma/\sqrt{G} + \sqrt{c\deltabyz}\zeta + \sqrt{c\deltabyz}\sigma)}{\sqrt{GT}}
        + c\deltabyz \zeta^2\right).
    \end{align}
    It remains to provide a valid upper bound $\wtilde{\Delta}$ on $\wtilde{\Phi}_0$. Since we initialize $v_i^0 = 0$, then 
    \[
    \wtilde{\Phi}^0 = F^0 + \frac{\gamma}{\beta}\|\nabla f(x^0)\|^2 \le F^0 + \frac{1}{4L}\cdot 2L(f(x^0)-f^\star) = 2F^0.
    \]
    This concludes the proof.

\end{proof}

\section{Additional Experiments and Training Details}\label{sec:additional_exp}

Our implementation builds on the codebase of \citet{horvath2020better} to simulate a distributed environment, while the attack implementations are adopted from \citet{gorbunov2022variance}. It can be found via the link \url{https://anonymous.4open.science/r/ByzClip21SGD2M}.

\subsection{Training Details}

In our experiments, we compare \algname{Byz-Clip21-SGD2M} against \algname{Byz-Clip-SGD} and \algname{Safe-DSHB}~\citep{allouah2023privacy}, outlined in \Cref{alg:byz_clip_sgd} and \Cref{alg:safe_dshb}, respectively. Note that \algname{Byz-Clip-SGD} is a variant of \algname{DP-SGD} in which a robust aggregation rule replaces server-side averaging. Since \algname{DP-SGD} is state-of-the-art for private learning, this modification (\Cref{alg:byz_clip_sgd}) serves as our first baseline. The method of \citet{allouah2023privacy} originally performs per-example clipping (line 5); to align with our analysis under $\sigma$-sub-Gaussian noise, i.e., without assuming a finite-sum structure for $f_i$, we instead clip the stochastic gradient.

In experiments of \Cref{sec:main_exp}, we use a coupling of Nearest Neighbor Mixing (\algname{NNM}) \citep{allouah2023fixing} and Coordinate Median (\algname{CM}) \citep{chen2017distributed} as an aggregation rule, while Gaussian mechanism utilizes $\delta=4\cdot10^{-4}.$  


\subsection{Label Flipping Attack}

We evaluate \algname{Byz-Clip21-SGD2M}, \algname{Byz-Clip-SGD}, and \algname{Safe-DSHB} on CNN and MLP models under a label-flipping attack \citep{paudice2018label}. MNIST is evenly partitioned across $n=25$ clients; Byzantine clients flip labels as $y\to 9-y$. We vary the number of Byzantine clients and the privacy budget $\varepsilon\in\{3,8,13,18,23\}$, fixing $\delta=4\cdot10^{-4}$. Hyperparameter tuning follows \Cref{sec:main_exp}. All methods use mini-batch gradients with a batch size $32$ and no privacy amplification by sub-sampling. We use a coupling of \algname{NNM} and \algname{CM} as a server-side aggregation rule. All algorithms are run for $60$ epochs.

\Cref{fig:label_flip} summarizes the results. Unlike under the \algname{IPM} attack, performance degrades for all methods as the number of Byzantine clients grows. Even so, \algname{Byz-Clip21-SGD2M} typically achieves the best test accuracy, aligning with our theory. In most settings, the margin is substantial, highlighting the effectiveness of the proposed algorithm for private training.

\subsection{Amplification by Sub-sampling}

\begin{algorithm}[!t]
\caption{\algname{Byz-Clip21-SGD2M+} }
\label{alg:byz_clip21_sgd2m_ampl}
\begin{algorithmic}[1]
\STATE \textbf{Input:} $x^0 \in X,$ momentum parameters $\beta,\hbeta\in(0,1],$ step-size $\gamma > 0$, $g_i^0=m_i^0\in\R^d,$ clipping parameters $\tau_i, \tau_o > 0$, DP-noise variance $\sigma_{\omega}^2 \ge 0$
    \FOR{$t=0, \ldots, T-1$}
        \STATE $x^{t+1} = x^t - \gamma g^t$
        \FOR{$i\in\cG$}
            \STATE Sample $\cS_i^{t+1} \sim {\rm Unif}[m]$ of cardinality $S$
            \STATE $\omega_i^{t+1} \sim \cN(0, \sigma_{\omega}^2\mI)$
            \STATE $v_i^{t+1} = (1-\beta)v_i^t + \frac{\beta}{S}\sum_{j\in \cS_i^{t+1}}(\clip_{\tau_i}(\nabla f_{ij}(x^{t+1})) + \omega_i^{t+1})$
            
            \STATE $c_i^{t+1} = \clip_{\tau}(v_i^{t+1} - g_i^t)$
            \STATE $g_i^{t+1} = g_i^t + \hbeta\clip_{\tau_o}(v_i^{t+1} - g_i^t)$
        \ENDFOR
        \FOR{$i\in\cB$}
            \STATE $c_i^{t+1} = (*)$ \hspace{11cm} sends arbitrary vector
        \ENDFOR
        \STATE $m_i^{t+1} = m_i^t + \hbeta c_i^{t+1}$
         \STATE $g^{t+1} =\aragg(m_1^{t+1},\dots,m_n^{t+1})$
        
    \ENDFOR
\end{algorithmic}	
\end{algorithm}

We next consider a modified variant of \algname{Byz-Clip21-SGD2M} (see \Cref{alg:byz_clip21_sgd2m_ampl}), which we call \algname{Byz-Clip21-SGD2M+}. Unlike the population setting considered in the main body, each client now holds a finite dataset of size $m$ with a local objective
\[
f(x) = \frac{1}{G}\sum_{i\in\cG}f_i(x), \quad f_i(x) = \frac{1}{m}\sum_{j=1}^m f_{ij}(x).
\]
At iteration $t$, client $i\in\cG$ samples a mini-batch $\cS^{t+1}_i$ of size $S$, averages example-wise clipped stochastic gradients, and updates its momentum buffer using a noised version of this average (line 7 in \Cref{alg:byz_clip21_sgd2m_ampl}). This modification enables privacy amplification through sub-sampling, as each per-example gradient is protected via the Gaussian mechanism, and local DP guarantees are implied by post-processing. Consequently, the required DP noise can be reduced to
\[
\sigma_{\omega} = \frac{S}{m}\cdot \frac{\tau}{\varepsilon}\sqrt{T\log\frac{1}{\delta}},
\]
following \citet{abadi2016deep}. We highlight that \algname{Byz-Clip21-SGD2M+} has two clipping operators. In the experiments, we vary the inner one (line 7) while the outer one is fixed $\tau_o=1$ (line 9). We compare this enhanced version of \algname{Byz-Clip21-SGD2M} against \algname{Byz-Clip-SGD} and original \algname{Safe-DSHB}, where example-wise clipping is also incorporated into the algorithm design.

We compare all methods under the previous setup: CNN/MLP training on MNIST with a stronger privacy budget $\varepsilon\in\{0.1,0.3,1,3\}$ and a label-flipping attack. We split the training set across $n=25$ clients equally. We tune the learning rate over $\{10, 1, 0.1, 0.01\}$ and the clipping threshold over $\{1,0.3,0.1,0.03,0.01\}$ (this corresponds to $\tau_i$ in \algname{Byz-Clip21-SGD2M+}). For \algname{Safe-DSHB} and \algname{Byz-Clip21-SGD2M+} we set $\beta=0.1$ as before; for \algname{Byz-Clip21-SGD2M+} we use $\tau_o=1$ and $\hbeta=0.01$. To defend against Byzantine clients, we employ a hybrid aggregation rule on the server that combines \algname{NNM} and \algname{CM}. We decrease the number of epochs to $30$.

Results in \Cref{fig:label_flip_ampl} (bottom line) show that with $1$ or $5$ Byzantine clients, the methods are broadly competitive when training MLP model, with \algname{Safe-DSHB} slightly leading at $\varepsilon=0.1$. With $10$ Byzantines, \algname{Safe-DSHB} and \algname{Byz-Clip-SGD} edge out \algname{Byz-Clip21-SGD2M+} at $\varepsilon=0.1$, whereas for $\varepsilon\in\{0.3,1,3\}$ \algname{Byz-Clip21-SGD2M+} clearly outperforms both baselines. 

For the CNN model (\Cref{fig:label_flip_ampl}, top row), the trend is similar: \algname{Byz-Clip-SGD} and \algname{Safe-DSHB} perform slightly better at $\varepsilon=0.1$, whereas \algname{Byz-Clip21-SGD2M+} is marginally more effective when the number of Byzantine clients is large.

Experiments on CIFAR10 \citep{krizhevsky2009learning} with the CNN model showed no significant difference between \algname{Byz-Clip21-SGD2M+} and the baselines; therefore, we do not report these results.

While we do not provide convergence guarantees for \algname{Byz-Clip21-SGD2M+}, our experiments show that it delivers competitive performance against other baselines when amplification by sub-sampling is enabled. A theoretical analysis is deferred to future work since the example-wise clipping makes the analysis substantially more involved than for \algname{Byz-Clip21-SGD2M}.

\begin{figure*}[t]
    \centering
    \begin{tabular}{c}

        \hspace{-2mm}\includegraphics[width=1\linewidth]{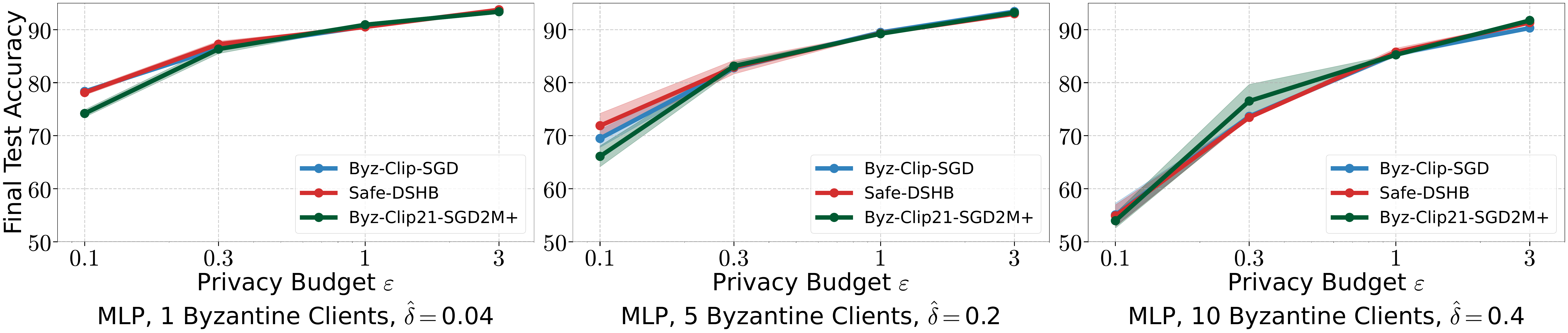} \\

        \hspace{-2mm}\includegraphics[width=1\linewidth]{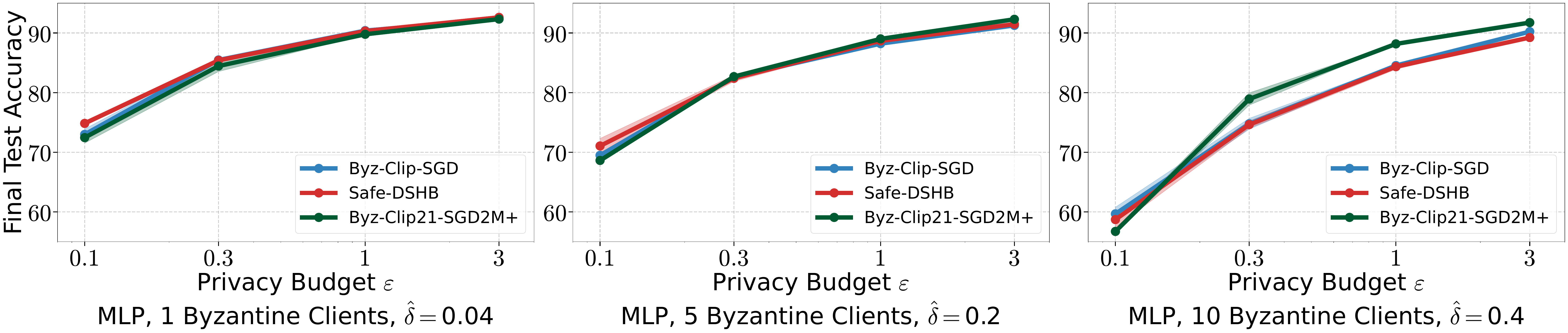}



    \end{tabular}
    \caption{Performance of \algname{Byz-Clip21-SGD2M+}, \algname{Byz-Clip-SGD}, and \algname{Safe-DSHB} when training CNN (top line) and MLP (bottom line) models on the MNIST dataset for different numbers of Byzantine clients and privacy budgets, when Byzantine clients use a label flipping attack, when amplification by sub-sampling is done.}
    \label{fig:label_flip_ampl}
\end{figure*}

\end{document}